%% file: main.tex
\documentclass[10pt]{article}
\input{commands}

\title{Flow Map Language Models:\\ One-step Language Modeling via Continuous Denoising}

\author[1]{Chanhyuk Lee}
\author[1]{Jaehoon Yoo}
\author[2]{Manan Agarwal}
\author[2]{Sheel Shah}
\author[2]{Jerry Huang}
\midauthor[2]{Aditi Raghunathan} 
\midauthor[1]{Seunghoon Hong}
\seniorauthor[\dag,2]{Nicholas M.~Boffi}
\seniorauthor[\dag,1]{Jinwoo Kim}
\affiliation[1]{KAIST}
\affiliation[2]{Carnegie Mellon University}
\contribution[\dag]{Equal advising}

\begin{document}
\begin{abstract}
\input{abstract}
\end{abstract}
\maketitle

\input{intro}
\input{background}
\input{theory}
\input{algo}

\input{results}
\input{conc}

\bibliographystyle{unsrtnat}
\bibliography{main}

\newpage
\appendix

\input{app}

\end{document}

%% file: commands.tex
\usepackage{bofflab}

\usepackage{microtype}
\usepackage{parskip}

\usepackage{booktabs}
\usepackage{caption}
\usepackage{multirow}
\usepackage{wrapfig}
\usepackage{makecell}
\usepackage{adjustbox}
\usepackage{fullpage}
\usepackage{ragged2e}
\usepackage{algorithm}
\usepackage{algorithmic}

\newtcolorbox{samplebox}[1]{
	enhanced,
	colback=LinkBlue!5!white,
	colframe=black!60,
	boxrule=0.6pt,
	arc=6pt,
	left=4pt, right=4pt, top=4pt, bottom=4pt,
	boxsep=3pt,
	before upper={\noindent\textbf{\small #1}\par\smallskip},
}

\usepackage[dvipsnames]{xcolor}

\usepackage{amsmath}
\usepackage{amssymb}
\usepackage{mathtools}
\usepackage{amsthm}
\usepackage{cancel}
\usepackage{enumitem}

\DeclareMathOperator*{\argmin}{argmin}

\setlist[enumerate,1]{leftmargin=2em, itemsep=1em, parsep=0pt}
\setlist[itemize,1]{leftmargin=2em}

\newcommand{\red}[1]{\textcolor{red}{#1}}

\newcommand{\flowvel}{\text{FLM}}
\newcommand{\flowmap}{\text{FMLM}}

\usepackage[capitalize,noabbrev]{cleveref}
\usepackage{thm-restate}
\crefformat{section}{Section~#2#1#3}
\crefformat{subsection}{Section~#2#1#3}
\crefformat{subsubsection}{Section~#2#1#3}
\crefformat{equation}{(#2#1#3)}
\crefrangeformat{equation}{(#3#1#4) to (#5#2#6)}
\crefmultiformat{equation}{(#2#1#3)}{ and (#2#1#3)}{, (#2#1#3)}{ and (#2#1#3)}
\crefname{appendix}{Appendix}{Appendices}
\Crefname{appendix}{Appendix}{Appendices}
\crefformat{appendix}{Appendix~#2#1#3}
\AddToHook{cmd/appendix/before}{%
\crefalias{section}{appendix}%
\crefalias{subsection}{appendix}%
\crefalias{subsubsection}{appendix}
}
\crefname{table}{Table}{Tables}
\crefname{figure}{Figure}{Figures}
\crefname{lemma}{Lemma}{Lemmas}
\crefname{assumption}{Assumption}{Assumptions}

\definecolor{d62728}{RGB}{214,39,40}
\definecolor{1f77b4}{RGB}{31,119,180}
\definecolor{9467bd}{RGB}{148,103,189}
\definecolor{FigNY}{HTML}{1a1a3e}    
\definecolor{FigSD}{HTML}{b0a8d0}    
\usepackage{colortbl}
\definecolor{pinegreen}{rgb}{0.0, 0.47, 0.44}
\definecolor{cornellred}{rgb}{0.7, 0.11, 0.11}
\definecolor{cadmiumgreen}{rgb}{0.0, 0.42, 0.24}
\definecolor{spirodiscoball}{rgb}{0.06, 0.75, 0.99}
\definecolor{blizzardblue}{rgb}{0.73, 0.96, 0.99}
\definecolor{aliceblue}{rgb}{0.91, 0.94, 0.97}
\definecolor{darkblue}{RGB}{232, 224, 240} 
\definecolor{Red7}{rgb}{0.941, 0.243, 0.243}
\definecolor{Green7}{RGB}{55, 178, 77}

\makeatletter
\def\thm@space@setup{%
  \thm@preskip=0.8em plus 0.2em minus 0.2em
  \thm@postskip=0.8em plus 0.2em minus 0.2em
}
\makeatother

\theoremstyle{plain}
\newtheorem{theorem}{Theorem}[section]
\newtheorem{proposition}[theorem]{Proposition}
\newtheorem{lemma}[theorem]{Lemma}
\newtheorem{corollary}[theorem]{Corollary}
\theoremstyle{definition}
\newtheorem{definition}[theorem]{Definition}
\newtheorem{assumption}[theorem]{Assumption}

\makeatletter
\let\c@table\c@figure
\let\c@algorithm\c@figure
\makeatother

\usepackage{pifont}
\usepackage{array}
\usepackage{xurl}

\newcommand{\bx}{{\bf x}}
\newcommand{\R}{\mathbb{R}}
\newcommand{\E}{\mathbb{E}}
\newcommand{\calL}{\mathcal{L}}
\newcommand{\sg}[1]{\mathsf{sg}(#1)}

\newcommand{\Id}{\text{Id}}

\newcommand{\kl}[2]{\mathsf{KL}(#1\:\Vert\:#2)}

\newcommand{\lsd}{\mathsf{LSD}}
\newcommand{\esd}{\mathsf{ESD}}
\newcommand{\psd}{\mathsf{PSD}}
\newcommand{\lmd}{\mathsf{LMD}}
\newcommand{\emd}{\mathsf{EMD}}
\newcommand{\pmd}{\mathsf{PMD}}

\newcolumntype{C}[1]{>{\centering\arraybackslash}p{#1}}
\newcolumntype{L}[1]{>{\raggedright\arraybackslash}p{#1}}

\usepackage[textsize=tiny]{todonotes}

%% file: abstract.tex
Language models based on discrete diffusion have attracted widespread interest for their potential to provide faster generation than autoregressive models.
Despite their promise, these models typically produce samples whose quality sharply degrades in the few-step regime, preventing a dramatic speedup in practice.
Here, we show that language models based on \textit{continuous flows} over one-hot token embeddings can outperform discrete diffusion in both quality and speed.
Importantly, our continuous formulation defines a unique \textit{flow map} that can be learned directly for efficient few-step inference, a structure we show is unavailable to discrete methods.
In this setting, we show that both the flow and its associated flow map can be learned with simple cross-entropy objectives that respect the simplex geometry of the data, and we identify three distinct choices for flow map distillation whose performance we compare in practice.
Using these insights, we build a \textit{flow language model} ($\flowvel$), a continuous flow that matches state-of-the-art discrete diffusion baselines on the One Billion Words (LM1B) and OpenWebText (OWT) datasets.
We then distill $\flowvel$ into a \textit{flow map language model} ($\flowmap$), whose \textit{one-step generation} exceeds the 8-step quality of recent few-step discrete diffusion language models.
Our work challenges the widely-held hypothesis that discrete noising processes are necessary for generative modeling over discrete modalities and paves the way toward accelerated language modeling at scale.

\vspace{0.5em}%
%
\textbf{\sffamily\bfseries Code: }\url{https://github.com/david3684/flm}

%% file: intro.tex
\section{Introduction}

\begin{wrapfigure}[17]{r}{0.45\textwidth}
	\vspace{-1.5em}
	\centering
	\includegraphics[width=\linewidth]{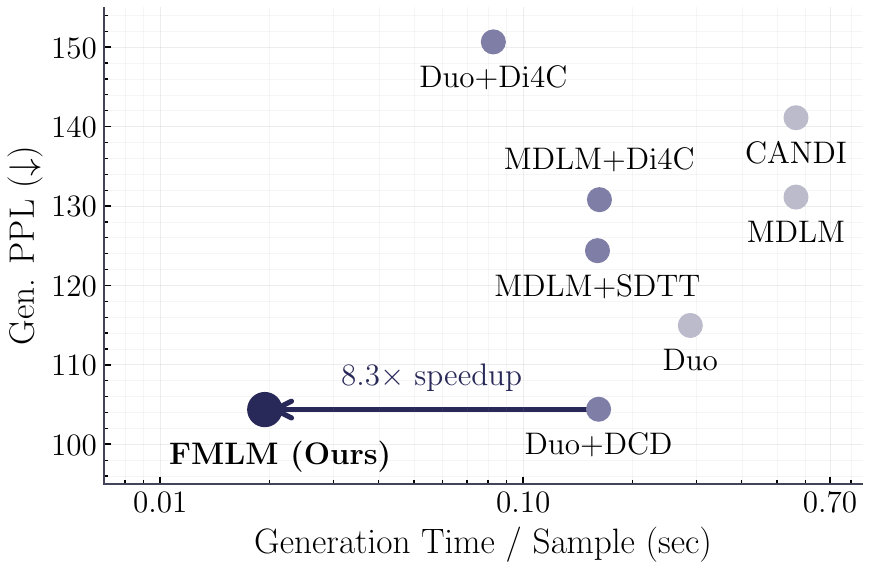}
	\caption{
		\textbf{Flow map language models.}
		Our FMLM outperforms discrete diffusion models (\textcolor{gray}{gray}) and matches the 8-step generation performance of distilled discrete diffusion models (\textcolor{LightPurple}{light purple}) in only \textit{one step} (\textcolor{DarkerPurple}{dark purple}).}
	\label{fig:pareto}
\end{wrapfigure}

Today's frontier language models (LMs) are based on an autoregressive process that produces one token per step \cite{achiam2023gpt, team2023gemini, guo2025deepseek}.
While these models leverage parallelism during training through teacher forcing and a transformer-based architecture, their sampling is inherently serial in nature, producing a bottleneck in generation speed.
Recently, language models based on discrete diffusions have attracted interest as a possible solution \cite{khanna2025mercury, googledeepmind2025gemini, song2025seed}.
By learning to reverse a noising process on full sequences, these models can output multiple tokens in parallel at each sampling step, thereby holding the potential for accelerated generation.

\begin{figure}[ht!]
	\centering
	\includegraphics[width=\linewidth]{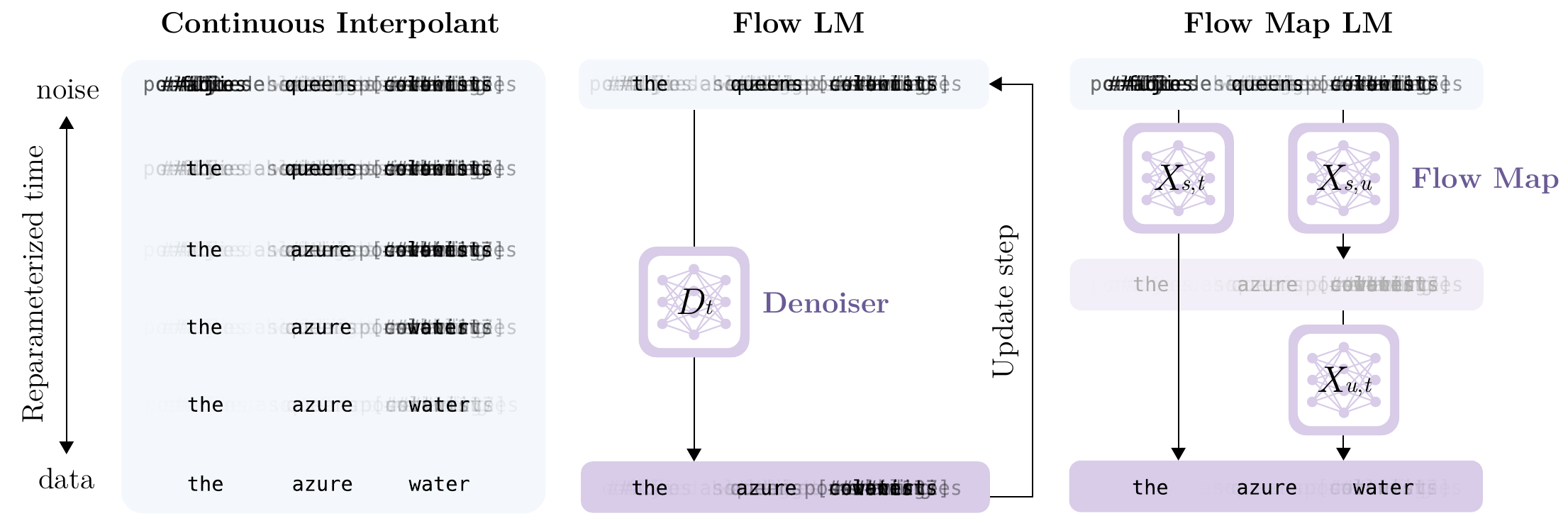}
	\caption{\textbf{Overview.}
		(Left) We leverage a simple continuous interpolation between Gaussian noise and a one-hot encoding of language data.
		(Middle) Our $\flowvel$ learns a denoiser that predicts the posterior over clean data, which we convert into a flow for sampling.
		(Right) Our distilled $\flowmap$ directly transports states between distant timepoints, enabling few-step generation.
	}
	\label{fig:overview}
\end{figure}

Despite their promise, discrete diffusion language models have significant practical limitations.
In particular, their generative quality typically drops off rapidly in the few-step regime \cite{deschenaux2024beyond}.
This is a critical drawback as diffusion models process the \textit{full sequence} simultaneously during inference, so that the number of sampling steps must be substantially reduced to compensate for the associated cost compared to their autoregressive counterparts \cite{dieleman2023language, zheng2024masked}.
This difficulty arises because the state space over sequences is combinatorially large, necessitating a factorized approximation of the transition probability for the reverse process that neglects correlations between tokens~\cite{wu2025fast, kang2025parallelbench}.
While this approximation renders discrete diffusions computationally feasible, empirically it requires many steps to accurately capture full sequences, leading to slow generation in practice.

In contrast, continuous diffusion language models, which represent and denoise tokens in a continuous space, do not rely on such an approximation~\cite{li2022diffusion, dieleman2022continuous}.
As a result, they can perform accurate parallel generation through a \textit{deterministic} evolution driven by a velocity or score function \cite{lipman2022flow, albergo2023stochastic, song2020score}.
Most interestingly, this makes them compatible with recent advances in few-step distillation methods that learn the \textit{flow map}, an operator that directly transports noise to data in as few as one function evaluation \cite{boffi2025build,boffi2025flowmapmatchingstochastic}.
Yet, despite these potential advantages, a widely held belief is that continuous diffusion language models underperform their discrete counterparts~\cite{sahoo2025diffusion, pynadath2025candi}, leading practitioners to prioritize discrete methods in recent years~\cite{nie2025large, khanna2025mercury}.

In this work, \textbf{we challenge this widespread belief}, showing that continuous flow-based language models can be higher-performing and faster than previously believed (\Cref{fig:pareto}).
In particular, our approach (\Cref{fig:overview}) reaches the performance of state-of-the-art (SoTA) discrete diffusion models and exceeds them in the few-step regime. 
Overall, our \textbf{main contributions} are:
\begin{itemize}
	\item We show that continuous flows over one-hot token embeddings define a unique flow map that can be directly learned for efficient few-step inference.
	      We further prove that this structure is unavailable to discrete diffusion methods due to the combinatorial size of their state space.
	\item Using these insights, we build a flow language model ($\flowvel$) and distill it into a flow map language model ($\flowmap$).
	      We identify reparameterizations of both the flow and the flow map into simplex-valued objects, which we use to introduce several novel cross-entropy objectives that respect the simplex geometry of discrete data.
	      We show that these objectives dramatically outperform their standard square error counterparts, and ablate over several key design decisions, such as a time reparameterization that resolves the training instabilities of prior continuous methods.
	\item We validate our approach on LM1B and OWT. $\flowvel$ matches SoTA discrete diffusion LMs in generation quality, while $\flowmap$ beats recent few-step LMs, nearing their 8-step quality at \textit{one step}.
	      We further highlight some of the downstream advantages of the FMLM approach, such as improved capabilities for inference-time steering and scaling via continuous guidance.
\end{itemize}

%% file: background.tex
\section{Background}
\label{sec:background}

\begin{figure}[t!]
	\centering
	\includegraphics[width=\linewidth]{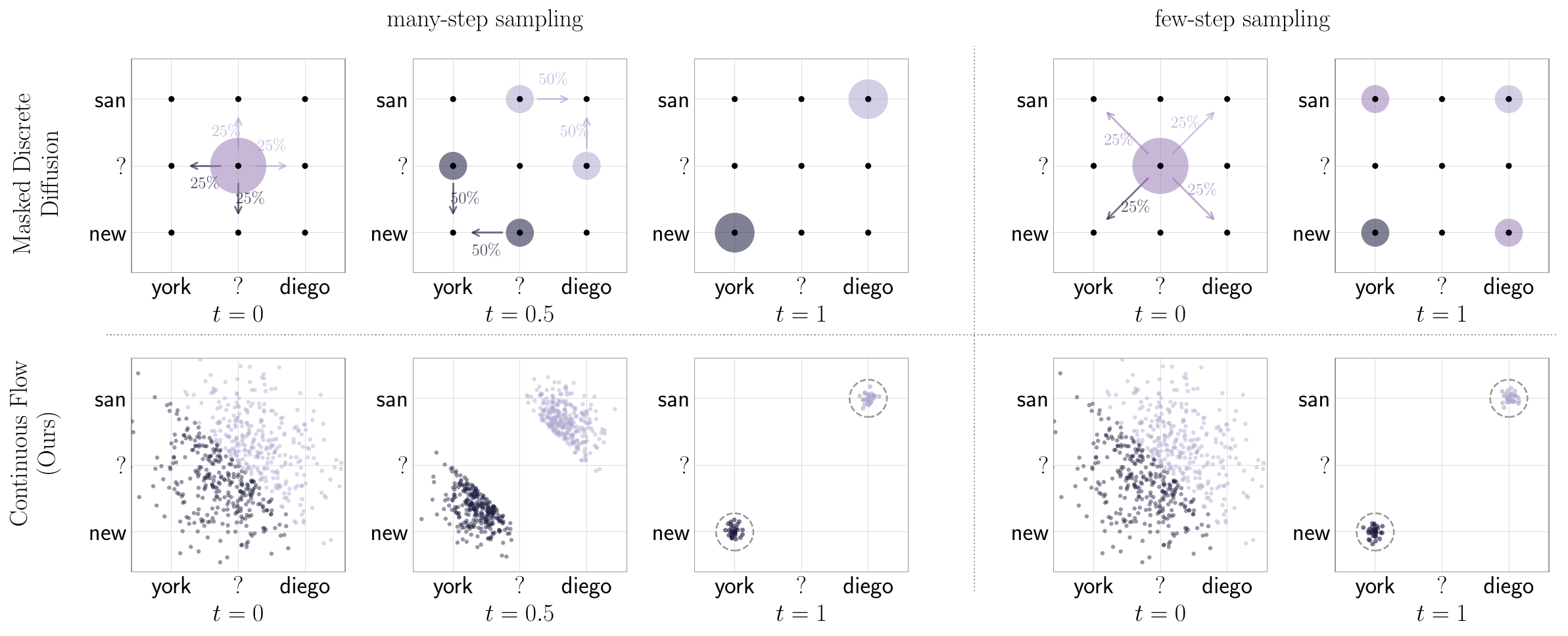}
	\caption{\textbf{Factorization error in discrete diffusion.}
		A toy dataset with two correlated modes \textcolor{FigNY}{\texttt{new}}-\textcolor{FigNY}{\texttt{york}} and \textcolor{FigSD}{\texttt{san}}-\textcolor{FigSD}{\texttt{diego}}. (Left) In many-step sampling, both continuous flows and discrete diffusion models generate valid data.
		(Right) With few-step sampling, the factorized transition of discrete diffusion yields a spurious mixture of all possible combinations (including the invalid pairings \textcolor{FigNY}{\texttt{new}}-\textcolor{FigSD}{\texttt{diego}} and \textcolor{FigSD}{\texttt{san}}-\textcolor{FigNY}{\texttt{york}}).}
	\label{fig:factorization}
\end{figure}

Let $V$ be a vocabulary of tokens, which here we treat as integers $[|V|]$ without loss of generality.
We denote a sample of language data with length $L$ by ${\bf y} = ({\bf y}^l)_{l=1}^L\in V^L$.
In language modeling, our goal is to estimate the data distribution $p({\bf y})$ on $V^L$ in such a way that we can efficiently draw a fresh sample.

\paragraph{Autoregressive language models} factorize the data distribution over length, leading to the representation
\begin{equation}
	\label{eqn:ar_model}
	p({\bf y})=p({\bf y}^1)p({\bf y}^2|{\bf y}^1)\hdots p({\bf y}^L|{\bf y}^{<L}),
\end{equation}
and learn the conditional distribution $p({\bf y}^l|{\bf y}^{<l})$ over tokens given a prefix \cite{jordan1986attractor, elman1990finding, bengio2003neural}.
These models generate text at inference by sequentially sampling each token conditioned on the past.
By construction the process is serialized, with each token requiring all previous tokens for generation, limiting efficiency~\cite{gu2017non, stern2018blockwise}.
This difficulty has motivated alternatives that model the full sequence at once to avoid serialization.

\paragraph{Discrete diffusion language models} aim to break the serial bottleneck by producing several tokens in parallel at each step \cite{austin2021structured, lou2023discrete}.
They employ a discrete noising process such as masking~\cite{sahoo2024simple, shi2024simplified} or uniform randomization~\cite{lou2023discrete, schiff2024simple} of multiple tokens, and learn to reverse it by modeling the transition density $p_{t|s}({\bf y}_t|{\bf y}_s)$~\cite{austin2021structured, gat2024discrete}, so that generation can be performed via ancestral sampling over a temporal grid $0=t_0< \hdots < t_N=1$.
Since each step updates multiple tokens simultaneously, substantial speedups can be achieved if the number of sampling steps can be reduced significantly below the sequence length $L$.

In practice, however, discrete diffusion models typically fail catastrophically in the few-step regime~\cite{deschenaux2024beyond}.
This failure occurs because the transition density is defined over the full text space $V^L$, so that learning it accurately requires a model with an intractable output dimensionality.
To sidestep this problem, discrete methods employ a \emph{factorized approximation} $\tilde{p}_{t|s}$,
\begin{align}
	\label{eq:factorization}
	\tilde{p}_{t|s}({\bf y}_t|{\bf y}_s) \coloneqq p_{t|s}^1({\bf y}_t^1|{\bf y}_s)\cdots p_{t|s}^L({\bf y}_t^L|{\bf y}_s),
\end{align}
where each factor gives the conditional probability of the $l$-th denoised token given the earlier state ${\bf y}_s$, marginalized over the remaining tokens.
While this approximation makes learning tractable, it makes a restrictive assumption that the denoised tokens ${\bf y}_t^1, ..., {\bf y}_t^L$ are conditionally independent given the earlier denoised state.
This assumption only holds in the infinitesimal limit $t\to s$ \cite{campbell2022continuous}, which creates the need for a large number of sampling steps at inference.
As the number of steps is reduced, this approximation causes the model to produce unnatural text by neglecting correlations between tokens (\Cref{fig:factorization}) \cite{wu2025fast, kang2025parallelbench}.
Unfortunately, this is a fundamental issue that cannot be resolved by improving model quality alone.
Our work addresses this central challenge with a continuous flow-based formulation~\cite{lipman2022flow, albergo2023stochastic}, which learns a deterministic transport map that need not make such a factorized approximation.
As a result, our approach directly enables \textbf{scalable one-step language modeling}.
For further context, we discuss additional related work in \Cref{sec:app:related_work}.

%% file: theory.tex
\section{Theoretical Framework}\label{sec:theory}

In this section, we describe our formulation of a continuous flow-based language model ($\flowvel$), as well as its few-step flow map ($\flowmap$).
To do so, we develop a continuous generative model over a canonical one-hot encoding of language, leveraging flow matching over a simple choice of stochastic interpolant.
Full details of the framework can be found in \cref{sec:app:flow_map_background,sec:app:denoiser_flow_maps}, where we give a complete background on flow maps and describe both \textit{distillation} and \textit{direct training} algorithms for $\flowmap$s.

\subsection{A continuous representation of language}

A natural way to build a continuous flow over language data ${\bf y} \in V^L$ is to first construct a continuous representation of the data.
We choose a continuous embedding $f:{\bf y}\mapsto {\bf x}$ and a decoder $g:{\bf x}\mapsto {\bf y}$ satisfying $g(f({\bf y}))={\bf y}$, and we model the induced distribution $p({\bf x})$ on the continuous space:
\begin{equation}
	p({\bf x})= p({\bf y}=g({\bf x})).
	\label{eq:continuous_rep}
\end{equation}
Inference can be performed by first generating $\hat{\bf x} \sim p(\bx)$ and then decoding into discrete language through $\hat{\bf y}=g(\hat{\bf x})$.
Several choices of the continuous representation have been explored in prior work, including learned embeddings \cite{li2022diffusion, dieleman2022continuous, gulrajani2023likelihood} and pretrained embeddings \cite{strudel2022self, lovelace2023latent}.
However, learned embeddings require careful regularization to prevent collapse or explosion, while pretrained embeddings may not be optimal for the flow.

Here, we adopt a simple and canonical tokenwise one-hot representation $f:V^L\to\R^{L\times |V|}$ with an argmax decoder $g:\R^{L\times |V|}\to V^L$,
\begin{equation}
	\label{eq:onehot}
	f:{\bf y} \mapsto \left(\mathsf{onehot}({\bf y}^1), ..., \mathsf{onehot}({\bf y}^L)\right)^\top, \qquad
	g:{\bf x} \mapsto \left(\mathsf{argmax}({\bf x}^1), ..., \mathsf{argmax}({\bf x}^L)\right).
\end{equation}
This choice offers a lossless representation of the discrete tokens and requires neither regularization nor auxiliary training.
Similar representations have been explored in prior work on continuous diffusion for language \cite{chen2022analog}, though these earlier works often impose additional constraints such as a simplex projection of the diffusion process \cite{han2023ssd, mahabadi2024tess}.
As we elaborate upon below, our approach operates in an unconstrained Euclidean space, which we find to be simpler and more effective in practice.

\subsection{Flow language models}\label{sec:flow_velocity}
Given a choice of continuous representation, the language modeling problem becomes that of learning a continuous data distribution $p({\bf x})$ on the embedding space.
To build such a model, we follow \citet{lipman2022flow} and~\citet{albergo2023stochastic}, and leverage flow matching over a stochastic interpolant.
This leads to a simple formulation of our method that matches the design of state-of-the-art flows for continuous data \cite{li2025back, zheng2025diffusion}.

\paragraph{Interpolant.}
In the stochastic interpolant framework, we specify a probability path $p_t({\bf x}_t)$ as the density of an interpolant between noise ${\bf x}_0 \sim p_0 = \mathsf{N}(0, I)$ and data ${\bf x}_1 \sim p_1$:
\begin{align}
	\label{eq:interpolant}
	I_t \coloneqq (1-t) {\bf x}_0 + t {\bf x}_1, \quad I_t \sim p_t.
\end{align}
Above, we choose a simple and canonical linear stochastic interpolant, but many choices are possible in practice by generalizing the factors $(1-t)$ and $t$ \citep{albergo2023stochastic}.

\paragraph{Probability flow.}
The probability path $p_t$ induced by the interpolant~\cref{eq:interpolant} admits a deterministic representation that can be used to efficiently produce a sample ${\bf x}_t\sim p_t$ at inference time,
\begin{align}
	\label{eq:flow_velocity_ode}
	\dot{\bf x}_t = b_t({\bf x}_t),\quad {\bf x}_0\sim p_0, \quad t \in [0,1],
\end{align}
where $b_t$ is the velocity field of the probability flow,
\begin{align}
	\label{eq:flow_velocity}
	b_t({\bf x}) = \E[\dot{I}_t | I_t = {\bf x}] = \E[{\bf x}_1-{\bf x}_0| I_t={\bf x}].
\end{align}
Above, the expectation is with respect to the random draws of ${\bf x}_0\sim p_0$ and ${\bf x}_1\sim p_1$ that are used to construct the interpolant.
The conditional expectation structure~\cref{eq:flow_velocity} means that the velocity can be learned efficiently by solving a square loss regression problem $b = \argmin_{\hat{b}} \calL_{\mathsf{MSE}}(\hat{b})$, where:
\begin{align}
	\label{eq:flow_velocity_loss}
	\calL_{\mathsf{MSE}}(\hat{b})
	 & \coloneqq \int_0^1 \E|\hat{b}_t(I_t) - \dot{I}_t|^2{\rm d}t.
\end{align}
In practice, \Cref{eq:flow_velocity_loss} is estimated by sampling $t$ uniformly and is then minimized over a neural network.
A sample approximately following $p_1$ is then obtained by numerically integrating \Cref{eq:flow_velocity_ode} over a temporal grid $0=t_0< t_1< \hdots < t_N=1$.

\paragraph{Denoiser.}
Despite its simplicity, learning the velocity directly induces a significant difficulty in our setup.
Velocity prediction requires regressing onto a \emph{noised target} $\dot{I}_t={\bf x}_1-{\bf x}_0$, which inherits the full-rank structure of the Gaussian noise samples ${\bf x}_0\in\R^{L\times |V|}$.
When the dimensionality $|V|$ is much larger than the internal feature dimension $d$ of the network, as is the typical case for language modeling, underfitting is known to occur~\citep{li2025back}.
To avoid this issue, we instead learn the posterior mean of the clean data, which relates to the velocity through a linear change of variables \cite{albergo2023stochastic, li2025back}:
\begin{align}
	\label{eq:x_prediction_target}
	D_t({\bf x}) \coloneqq \E[{\bf x}_1 | I_t={\bf x}],\quad b_t({\bf x}) & = \frac{D_t({\bf x}) - {\bf x}}{1-t}.
\end{align}
The function $D_t$, often called the ``denoiser'', can be learned in practice by predicting the clean data ${\bf x}_1$ via a regression problem $D=\argmin_{\hat{D}}\calL_{\mathsf{MSE}}(\hat{D})$, where:
\begin{align}
	\label{eq:flow_x_prediction_loss}
	\calL_{\mathsf{MSE}}(\hat{D}) & \coloneqq  \int_0^1 \E|\hat{D}_t(I_t) - {\bf x}_1 |^2{\rm d}t.
\end{align}
Since ${\bf x}_1$ consists of stacked one-hot encoded tokens, the targets are highly structured and low-entropy, avoiding the difficulties of velocity prediction.
Importantly, in our discrete setting the denoiser admits a simple probabilistic interpretation, which will enable us to exploit the one-hot geometry even further.
\begin{lavenderbox}
	\begin{restatable}{lemma}{denoiserposterior}\label{lemma:denoiser_posterior_equivalence}
		The optimal denoiser is given by the token-level posterior,
		\begin{equation}
			\label{eq:denoiser_posterior_equivalence}
			D_t({\bf x})^l = p_{1|t}^l(\cdot | I_t={\bf x}),
		\end{equation}
		so that the optimal denoiser lies on the simplex, $D_t({\bf x})^l \in \Delta^{|V|-1}$.
	\end{restatable}
\end{lavenderbox}
A proof is given in \Cref{sec:posterior_proof}.
By~\cref{lemma:denoiser_posterior_equivalence}, we may parameterize $\hat{D}$ using a tokenwise softmax output layer to constrain the learned network to lie on the simplex.
This restricts the hypothesis space to valid discrete distributions, allowing the model to focus on estimating the correct posterior rather than \textit{also} learning the one-hot structure of the data, as would be necessary with a velocity representation.
Most significantly, this enables training with a cross-entropy loss \cite{dieleman2022continuous, eijkelboom2024variational}, which is adapted to the one-hot geometry:
\begin{lavenderbox}
	\begin{restatable}{proposition}{cedenoiser}\label{prop:classification}
		Consider the cross-entropy objective
		\begin{align}
			\label{eq:flow_x_classification_loss}
			\calL_{\mathsf{CE}}(\hat{D}) \coloneqq \int_0^1 \,\E\left[- \sum_{l=1}^L \log \hat{D}_t(I_t)^l \cdot {\bf x}_1^l\right]{\rm d}t.
		\end{align}
		Then, the optimal denoiser $D_t$ is the unique minimizer of $\calL_{\mathsf{CE}}$.
		Moreover, if the excess risk $\Delta_D(\hat{D}) \coloneqq \calL_{\mathsf{CE}}(\hat{D}) - \calL_{\mathsf{CE}}(D) \le \epsilon$, then for any early stopping time $\xi \in (0, 1)$,
		\begin{equation}
			\label{eq:excess_risk_bound}
			W_2^2(\hat{p}_{1-\xi}, p_{1-\xi}) \le C\epsilon,
		\end{equation}
		where $C > 0$ depends on $\xi$ and the Lipschitz constant of the model.
	\end{restatable}
\end{lavenderbox}
A proof is in \cref{sec:classification_proof}.
The proof shows that $\calL_{\mathsf{CE}}(\hat{D})$ decomposes into an irreducible conditional entropy plus a sum of KL divergences from the true posterior to the model, so that minimizing $\calL_{\mathsf{CE}}$ is equivalent to minimizing the KL divergence from the true token-level posterior.
The bound~\cref{eq:excess_risk_bound} compares at an early stopping time $1 - \xi$ to avoid the $(1-t)^{-1}$ singularity in the denoiser-velocity conversion~\cref{eq:x_prediction_target}; as shown in the proof, we may also obtain a guarantee that extends to $t = 1$ by adding an additional $O(\xi)$ remainder to the bound.
We give quantitative constants and a full proof in~\cref{sec:classification_proof}.

\paragraph{Relationship with discrete diffusion.}
Lemma~\ref{lemma:denoiser_posterior_equivalence} suggests that the optimal denoiser implicitly learns the \textit{factorized} posterior $p^l_{1|t}({\bf x}_1|{\bf x}_t)$ defined in~\Cref{eq:factorization}.
This reveals an interesting connection with discrete diffusion models, which also often learn $p^l_{1|t}$ via a tokenwise cross-entropy objective \cite{austin2021structured, campbell2022continuous, gat2024discrete}.
While discrete models use the learned $p^l_{1|t}$ to perform ancestral sampling, which requires the entire joint probability density and thus suffers from factorization errors, continuous models use the learned $p^l_{1|t}$ to infer the \emph{exact} velocity based on \eqref{eq:denoiser_posterior_equivalence} and \eqref{eq:x_prediction_target}.
This critical difference underlies the capability of continuous flows to be distilled exactly into few-step generators, as we now describe.

\subsection{Flow map language models}\label{sec:flow_map}
The framework in \Cref{sec:flow_velocity} does not immediately allow for few-step language modeling, since the numerical solvers used to integrate~\Cref{eq:flow_velocity_ode} typically become inexact at large step sizes.
Here we overcome this challenge by leveraging the \emph{flow map} $X_{s, t}: \R^{L \times |V|} \to \R^{L \times |V|}$.
The flow map is the solution operator of~\Cref{eq:flow_velocity_ode}, and by definition directly transports between any two timepoints \citep{boffi2025build,boffi2025flowmapmatchingstochastic}:
\begin{equation}
	X_{s,t}({\bf x}_s) = {\bf x}_t,\quad \text{for all} \:(s,t)\in[0,1]^2.
\end{equation}
Without loss of generality, we may parameterize it as:
\begin{align}
	\label{eq:flow_map}
	X_{s,t}({\bf x}) = {\bf x} + (t-s)v_{s, t}({\bf x}),
\end{align}
where $v$ is called the average velocity or the ``mean flow'' \citep{geng2025mean, geng2025improved}.
Given a flow map, sampling $\hat{\bf x}_1\sim p_1$ can be performed by choosing a temporal grid $0=t_0<...<t_N=1$ and sequentially evaluating $\hat{\bf x}_{t_{i+1}}=X_{t_i,t_{i+1}}(\hat{\bf x}_{t_i})$.
Unlike numerical integration of \Cref{eq:flow_velocity_ode}, this approach is accurate for an arbitrary grid, enabling sampling in as few as one evaluation via $\hat{\bf x}_1 = X_{0, 1}({\bf x}_0)$.
In practice, leveraging more steps typically improves performance.

\paragraph{Learning.}
Methods for learning flow maps leverage the following mathematical properties, which fully characterize the flow map under standard regularity conditions~\cite{boffi2025build, boffi2025flowmapmatchingstochastic}:
\begin{equation}
	\label{eq:flow_map_conditions}
	\begin{aligned}
		X_{s,s}({\bf x})                         & = {\bf x},          &  & \text{for all}\; {\bf x} \in \R^{L \times |V|},\; s \in [0,1],           \\
		\lim_{s\to t}\partial_t X_{s,t}({\bf x}) & = b_t({\bf x}),     &  & \text{for all}\; {\bf x} \in \R^{L \times |V|},\; (s, t) \in [0,1]^2,    \\
		X_{u,t}(X_{s,u}({\bf x}))                & = X_{s,t}({\bf x}), &  & \text{for all}\; {\bf x} \in \R^{L \times |V|},\; (s, u, t) \in [0,1]^3.
	\end{aligned}
\end{equation}
The final \textit{semigroup} condition can be replaced with two differential alternatives~\citep{boffi2025build,boffi2025flowmapmatchingstochastic}: a \emph{Lagrangian} characterization involving time derivatives along flow trajectories, and an \emph{Eulerian} characterization involving spatial derivatives of the velocity.
These lead to learning objectives that require the computation of Jacobian-vector products, such as MeanFlow~\citep{geng2025mean, geng2025improved} and Lagrangian self-distillation~\citep{boffi2025build,boffi2025flowmapmatchingstochastic,zhou2025terminal}.
In \cref{sec:app:denoiser_flow_maps}, we develop all three characterizations in the two-time denoiser framework introduced below, deriving the corresponding cross-entropy objectives and self-distillation algorithms for each.
Here we focus on the semigroup condition in the distillation setting due to its simplicity; this relates to progressive distillation \citep{salimans2022progressive} and shortcut models~\citep{frans2024one}.
Empirical exploration of the alternatives is an interesting direction we leave for future work.

Using~\cref{eq:flow_map_conditions}, the flow map can be learned via distillation from a pre-trained velocity $\hat{b}_t$ by minimizing
\begin{equation}
	\label{eq:semigroup_loss}
	\calL_{\mathsf{MSE}}(\hat{v}) \coloneqq \int_0^1\int_0^t\int_s^t\E\big| \hat{X}_{s,t}(I_s) - \sg{\hat{X}_{u,t}(\hat{X}_{s,u}(I_s))}\big|^2 {\rm d}u\,{\rm d}s\,{\rm d}t + \int_0^1 \E\big|\hat{v}_{t,t}(I_t) - \hat{b}_t(I_t)\big|^2 {\rm d}t,
\end{equation}
with $\sg{\cdot}$ denoting the stop-gradient operator and where $\hat{X}_{s,t}({\bf x}) = {\bf x} + (t-s)\hat{v}_{s,t}({\bf x})$.
The first term enforces the semigroup condition on the off-diagonal, while the second fits the diagonal to the pre-trained velocity via the first (tangent) condition in~\cref{eq:flow_map_conditions}, $v_{t,t} = b_t$.
The boundary condition is satisfied by the parameterization~\cref{eq:flow_map}.
In practice, it is common to reparameterize the first term entirely in terms of $\hat{v}_{s,t}$ using~\cref{eq:flow_map}~\citep{boffi2025build}.
This formulation also admits a direct training variant by replacing the pre-trained $\hat{b}_t$ with the interpolant time derivative $\dot{I}_t$, eliminating the need for a teacher~\citep{boffi2025build}.

\paragraph{The two-time denoiser.}
In \Cref{sec:flow_velocity}, we used the denoiser $D_t$ to turn the velocity $b_t$ into a simplex-valued posterior, enabling training via cross-entropy.
Given this insight, learning the flow map using the square loss via~\cref{eq:semigroup_loss} is unsatisfactory, as it does not leverage the one-hot geometry of discrete data.
We now develop a novel reparameterization that directly addresses this issue.
To do so, we observe a rearrangement of \Cref{eq:x_prediction_target} into $D_t({\bf x})={\bf x}+(1-t)b_t({\bf x})$, showing that the denoiser equals a single Euler step of size $1-t$ with the velocity field.
Mirroring this relationship, we define a new quantity we refer to as the \emph{two-time denoiser}:
\begin{align}
	\label{eq:delta_denoiser}
	\delta_{s,t}({\bf x}) \coloneqq {\bf x} + (1-s)v_{s,t}({\bf x}),
\end{align}
which takes a single step of the average velocity $v_{s,t}$ using the full remaining time $1-s$.
Given this definition, we may now state the following result.
\begin{lavenderbox}
	\begin{restatable}{proposition}{twotimedenoiser}\label{prop:delta_conditions}
		The two-time denoiser $\delta_{s, t}$ satisfies the following four properties:
		\vspace{1em}
		\begin{enumerate}[label=(\roman*)]
			\item The flow map can be recovered exactly,
			      \begin{align}\label{eq:flow_map_delta_denoiser}
				      X_{s,t}({\bf x}) = \frac{1-t}{1-s}\,{\bf x} + \frac{t-s}{1-s}\,\delta_{s,t}({\bf x}).
			      \end{align}
			\item The two-time denoiser lies on the simplex,
			      \begin{equation}
				      \label{eqn:delta_simplex}
				      \delta_{s,t}({\bf x})^l \in \Delta^{|V|-1}
			      \end{equation}
			      for all token positions $l = 1, \hdots, L$.
			      \label{item:delta_simplex}
			\item The two-time denoiser recovers the standard denoiser on the diagonal,
			      \begin{equation}
				      \label{eq:delta_denoiser_diagonal}
				      \delta_{s, s}({\bf x}) = D_s({\bf x}), \quad \text{for all}\quad {\bf x} \in \R^{L \times |V|},\:\: s \in [0, 1].
			      \end{equation}
			\item The two-time denoiser satisfies a semigroup condition,
			      \begin{equation}
				      \label{eqn:delta_semigroup}
				      \delta_{s,t}({\bf x}) = \gamma\,\delta_{s,u}({\bf x}) + (1-\gamma)\,\delta_{u,t}(X_{s,u}({\bf x})),
			      \end{equation}
			      where $\gamma = \frac{(1-t)(u-s)}{(1-u)(t-s)} \in [0,1]$. \label{eq:delta_denoiser_semigroup}
		\end{enumerate}
	\end{restatable}
\end{lavenderbox}
A proof is given in \Cref{sec:app:denoiser_flow_maps}, which shows that the two-time denoiser can always be written as a weighted average of single-time denoisers along the flow trajectory, leading to the simplex property~\cref{eqn:delta_simplex}.
Remarkably, this shows that the two-time denoiser always lies on the probability simplex $\Delta^{|V|-1}$, in stark contrast to the flow map $X_{s,t}$ itself, which leaves the simplex due to the Gaussian noising (\Cref{fig:simplex_semigroup}).
This means that we can parameterize $\delta_{s,t}$ with a tokenwise softmax output layer, similar to how we learn $D_t$.
To further leverage the simplex geometry, we can learn $\delta$ by using the KL divergence to impose the semigroup condition \cref{eqn:delta_semigroup},
\begin{lavenderbox}
	\begin{restatable}{proposition}{semigroupce}\label{prop:semigroup_ce_loss}
		The two-time denoiser $\delta_{s,t}$ is the unique critical point of the following KL-based semigroup objective:
		\begin{equation}
			\label{eq:semigroup_ce_loss}
			\begin{aligned}
				\calL_{\mathsf{KL}}(\delta) &\coloneqq \E_{t, s, u}\,\E_{{\bf x}_0, {\bf x}_1}\Big[\sum_{l=1}^L \kl{\bar{\delta}^l_{s,t}}{\delta^l_{s,t}(I_s)}\Big] + \E_t\,\E_{{\bf x}_0, {\bf x}_1}\Big[\sum_{l=1}^L \kl{D^l_t(I_t)}{\delta^l_{t,t}(I_t)}\Big],\\
				\bar{\delta}_{s,t} &\coloneqq \sg{\gamma \delta_{s,u}(I_s) + (1-\gamma) \delta_{u,t}(X_{s,u}(I_s))}
			\end{aligned}
		\end{equation}
		where $\bar{\delta}_{s, t}$ is the semigroup teacher, and where the expectation over $(s, u, t)$ is taken from a distribution that has full support on $\{0 \leq s \leq u \leq t \leq 1\}$.
	\end{restatable}
\end{lavenderbox}
A proof is given in \Cref{sec:app:delta_denoiser_learning}, and a graphical depiction is provided in~\cref{fig:simplex_semigroup}.
In practice, we replace $D_t$ with a pre-trained estimate $\hat{D}_t$ and optimize over $\hat{\delta}_{s,t}$ as a distillation algorithm.
This enables us to recover the flow map for the corresponding pre-trained model via~\cref{eq:flow_map_delta_denoiser}.
In~\cref{eq:semigroup_ce_loss}, we use KL rather than cross entropy because KL's minimum value is zero, and as a result the objective gives a direct measure of reconstruction quality; this choice is equal up to a constant shift and otherwise has the same learning dynamics.
\begin{wrapfigure}[15]{r}{0.4\textwidth}
	\vspace{-1.6em}
	\centering
	\includegraphics[width=0.39\textwidth]{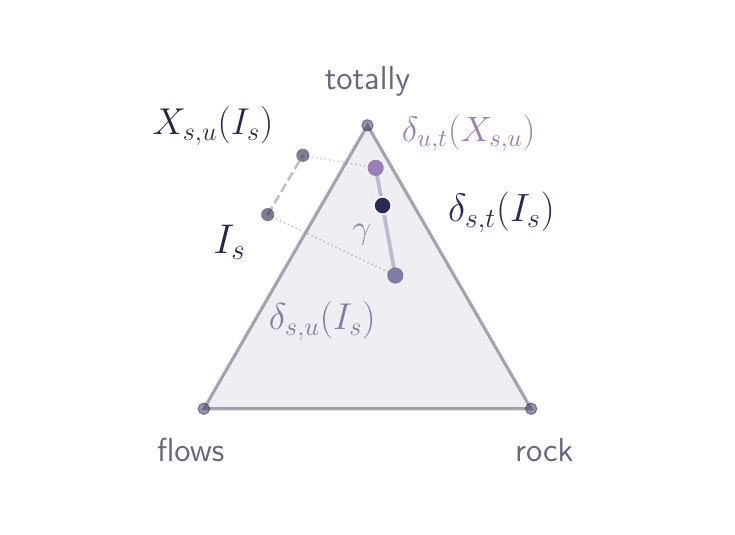}
	\vspace{-1.5em}
	\caption{
		\textbf{Semigroup on the simplex.}
		$X_{s,u}({\bf x})$ leaves the simplex, but $\delta_{s,u}({\bf x})$ and $\delta_{u,t}(X_{s,u})$ always lie on it.
		$\delta_{s,t}({\bf x})$ is their convex combination, providing a training signal for distillation.}
	\label{fig:simplex_semigroup}
\end{wrapfigure}
Replacing $D_s$ by the one-hot data ${\bf x}_1$ in the diagonal term recovers the cross-entropy loss~\cref{eq:flow_x_classification_loss} for the single-time denoiser, since KL from a one-hot distribution reduces to cross-entropy.
This yields a direct training algorithm that does not require a pre-trained teacher.
We give full details on direct training, along with two additional objectives for learning the two-time denoiser (either as distillation or direct training algorithms) based on differential characterizations in \Cref{sec:app:denoiser_flow_maps}.

\paragraph{Discrete flow maps.}
Given our discussions, a natural question is whether there exists a ``discrete flow map'' for the typical continuous-time Markov chains used in discrete diffusion models, such as masked diffusion and uniform state discrete diffusion.
These chains admit a deterministic evolution at the distribution level in terms of the ``master equation'' $\dot p_t=Q_t p_t$ for a rate matrix $Q_t \in \R^{|V|^L \times |V|^L}$ determined by the choice of process~\cite{campbell2022continuous}.
As a result, a flow map exists on the space of distributions over sequences $p_t\in\Delta^{|V|^L}$.
Computing or representing this object is intractable, necessitating the factorized approximations discussed in \Cref{sec:background}, so that the corresponding approximate flow map cannot perform in the few-step regime.
The following result shows that this difficulty is fundamental and cannot be avoided at the sample level: no sample-level deterministic transport map can generically replace the distributional flow map in the discrete setting.
\begin{restatable}{proposition}{nodiscretetransport}\label{prop:no_discrete_transport}
	Let $S$ be a finite set.
	For any probability distribution $\mu$ on $S$, there exists a distribution $\nu$ on $S$ that cannot be expressed as $\nu = f_\# \mu$ for any deterministic map $f: S \to S$.
\end{restatable}
A proof is given in \cref{sec:monge_map}.
By contrast, continuous flows \textit{do} admit a flow map at the sample level that is tractable to learn and evaluate, bringing dramatic efficiency gains in practice.

%% file: algo.tex
\section{Algorithmic Aspects}

\begin{figure}[t!]
	\begin{minipage}[t]{0.55\textwidth}
		\begin{algorithm}[H]
			\caption{$\flowvel$ training}
			\label{alg:flm_train}
			\begin{algorithmic}
				\STATE {\bfseries Require:} Dataset $\mathcal{D}$, reparameterization $\tau(t)$, lr $\eta$
				\STATE {\bfseries Initialize:} Denoiser $\hat{D}$
				\REPEAT
				\STATE ${\bf x}_1 \leftarrow f({\bf y})$, ${\bf y} \sim \mathcal{D}$; \; ${\bf x}_0 \sim \mathsf{N}(0, I)$
				\STATE Sample $t$ via $\tau(t) \sim \mathsf{U}[0, 1]$
				\STATE $I_t \leftarrow (1-t){\bf x}_0 + t{\bf x}_1$
				\STATE $\hat{\bf x}_1 \leftarrow \hat{D}_{t}(I_t)$
				\STATE Update $\hat{D}$: $\mathcal{L}_{\mathsf{CE}} = - \sum_l ({\bf x}_1^l)^\top \log \hat{\bf x}_1^l$
				\UNTIL{converged}
			\end{algorithmic}
		\end{algorithm}
	\end{minipage}
	\hfill
	\begin{minipage}[t]{0.425\textwidth}
		\begin{algorithm}[H]
			\caption{$\flowvel$ sampling}
			\label{alg:flm_sample}
			\begin{algorithmic}
				\STATE {\bfseries Require:} Trained $\hat{D}$, $\tau(t)$, steps $N$
				\STATE ${\bf x}_0 \sim \mathsf{N}(0, I)$
				\STATE $t_n \leftarrow t(n/N)$ for $n=0,...,N$
				\FOR{$n=0$ {\bfseries to} $N-1$}
				\STATE $\hat{\bf x}_1 \leftarrow \hat{D}_{t_n}({\bf x}_{t_n})$
				\STATE $\hat{b}_n \leftarrow (\hat{\bf x}_1 - {\bf x}_{t_n}) / (1 - t_n)$
				\STATE ${\bf x}_{t_{n+1}} \leftarrow {\bf x}_{t_n} + (t_{n+1} - t_n) \hat{b}_n$
				\ENDFOR
				\STATE {\bfseries Return} $g({\bf x}_{t_N})$
			\end{algorithmic}
		\end{algorithm}
	\end{minipage}
\end{figure}
\begin{figure}[t!]
    \vspace{-1em}
	\begin{minipage}[t]{0.55\textwidth}
		\begin{algorithm}[H]
			\caption{$\flowmap$ training (distillation)}
			\label{alg:fmlm_train}
			\begin{algorithmic}
				\STATE {\bfseries Require:} Dataset $\mathcal{D}$, trained $\hat{D}$, $\tau(t)$, lr $\eta$
				\STATE {\bfseries Initialize:} Two-time denoiser $\hat{\delta}$
				\REPEAT
				\STATE ${\bf x}_1 \leftarrow f({\bf y})$, ${\bf y} \sim \mathcal{D}$; \; ${\bf x}_0 \sim \mathsf{N}(0, I)$
				\vspace{0.5em}
				\STATE \textit{Diagonal} (anchors to $\hat{D}$):
				\STATE \quad $\tau(s) \sim \mathsf{U}[0,1]$; \; $I_s \leftarrow (1\!-\!s){\bf x}_0 + s{\bf x}_1$
				\STATE \quad $\mathcal{L}_{\mathsf{diag}} \leftarrow -\!\sum_l \hat{D}(I_s)^l \!\cdot\! \log \hat{\delta}_{s,s}^l(I_s)$
				\vspace{0.5em}
				\STATE \textit{Off-diagonal} (semigroup):
				\STATE \quad $h \!\sim\! \mathsf{U}[0,1]$; $\tau(s) \!\sim\! \mathsf{U}[0, 1\!-\!h]$
				\STATE \quad $\tau(t) \!\leftarrow\! \tau(s)\!+\!h$; \; $\tau(u) \!\leftarrow\! \tfrac{\tau(s)+\tau(t)}{2}$
				\STATE \quad $\gamma \leftarrow \frac{(1-t)(u-s)}{(1-u)(t-s)}$
				\STATE \quad $\hat{X}_{s,u} \leftarrow \frac{1-u}{1-s} I_s + \frac{u-s}{1-s} \hat{\delta}_{s,u}(I_s)$
				\STATE \quad $\bar{\delta} \leftarrow \sg{\gamma \hat{\delta}_{s,u}(I_s) + (1\!-\!\gamma) \hat{\delta}_{u,t}(\hat{X}_{s,u})}$
				\STATE \quad $\mathcal{L}_{\mathsf{off}} \leftarrow -\!\sum_l \bar{\delta}^l \!\cdot\! \log \hat{\delta}_{s,t}^l(I_s)$
				\STATE Update $\hat{\delta}$: $\mathcal{L}_{\mathsf{diag}} + \mathcal{L}_{\mathsf{off}}$
				\UNTIL{converged}
			\end{algorithmic}
		\end{algorithm}
	\end{minipage}
	\hfill
	\begin{minipage}[t]{0.42\textwidth}
		\begin{algorithm}[H]
			\caption{$\flowmap$ sampling}
			\label{alg:fmlm_sample}
			\begin{algorithmic}
				\STATE {\bfseries Require:} Trained $\hat{\delta}$, $\tau(t)$, steps $N$
				\STATE ${\bf x}_0 \sim \mathsf{N}(0, I)$
				\STATE $t_n \leftarrow t(n/N)$ for $n=0,...,N$
				\FOR{$n=0$ {\bfseries to} $N-1$}
				\STATE ${\bf x}_{t_{n+1}} \leftarrow \frac{1-t_{n+1}}{1-t_n} {\bf x}_{t_n} + \frac{t_{n+1}-t_n}{1-t_n} \hat{\delta}_{t_n, t_{n+1}}({\bf x}_{t_n})$
				\ENDFOR
				\STATE {\bfseries Return} $g({\bf x}_{t_N})$
			\end{algorithmic}
		\end{algorithm}
	\end{minipage}
\end{figure}

We now describe the practical implementation of an FLM and its subsequent distillation into an FMLM.
Here, we aim to provide principled design choices that work robustly in practice, as well as to highlight some of the key decisions necessary for performance.

\paragraph{Time reparameterization.}
\label{sec:flow_velocity_implementation}

To build a high-performing $\flowvel$, we find empirically that two particularly important choices are (i) how to sample time during training, and (ii) how to choose a temporal grid during inference.
A na\"ive approach would be to use uniform sampling $t\sim \mathsf{U}[0,1]$ for training and an equispaced grid $t_n=n/N$ for generation, which typically works well for continuous modalities such as images.
Unfortunately, we find this simple choice to be suboptimal for interpolants defined over one-hot encodings of discrete data, as the generative process concentrates its ``decisions'' about which token a noisy sample will converge to in a narrow time interval, especially for large vocabularies.
To understand this phenomenon, we consider the \emph{decoding error rate} $P_e: [0, 1] \to [0, 1]$ as a quantitative measure \cite{sahoo2025diffusion, pynadath2025candi},
\begin{equation}
	\label{eqn:decoding_error}
	P_e(t) \coloneqq \frac{1}{L}\sum_{l=1}^LP(g^l({\bf x}_t)\neq g^l({\bf x}_1)).
\end{equation}
The decoding error rate measures the expected fraction of tokens that would be incorrectly decoded if we were to stop the flow at time $t$.
By construction, it starts at a value $P_e(0)=1-\frac{1}{|V|}$ and decreases to $P_e(1)=0$.
The rate of decrease $|\dot{P}_e(t)|$ captures how much ``progress'' the flow makes in determining tokens at time $t$.
For large $|V|$, we find that~\cref{eqn:decoding_error} concentrates acutely near $t=1$ (\Cref{fig:reparameterization}, curves), implying that most times do not contribute significantly towards decoding and that token identities are resolved in a narrow window.
Uniform sampling $t \sim \mathsf{U}[0,1]$ therefore wastes training signal on regions where little denoising occurs, while undersampling the critical interval where tokens are actually determined.
Similarly, an equispaced grid during inference $t_n = n/N$ allocates most sampling steps to regions that contribute minimally to generation quality.

Following \citet{dieleman2022continuous} and \citet{stancevic2025entropic}, we address this using a \emph{time reparameterization} $\tau(t)$, which is a differentiable, monotonically increasing function with endpoints $\tau(0)=0,\tau(1)=1$ and inverse $t(\tau)$.
To effectively allocate time points to handle the non-uniformity of the decoding error, we train and generate uniformly in $\tau$, sampling $t$ via $\tau(t) \sim \mathsf{U}[0,1]$ during training, and using a grid $t_n=t(\tau_n)$ with $\tau_n=n/N$ for generation.
We propose to choose $\tau(t)$ so that uniform steps in $\tau$ correspond to uniform \emph{progress} in determining tokens.
As $P_e(t)$ measures the remaining decoding error at time $t$, we view its decrease from $P_e(0)$ as cumulative progress.
Standardizing to $[0,1]$, we obtain the relation:
\begin{wrapfigure}[16]{l}{0.45\textwidth}
	\centering
	\includegraphics[width=\linewidth]{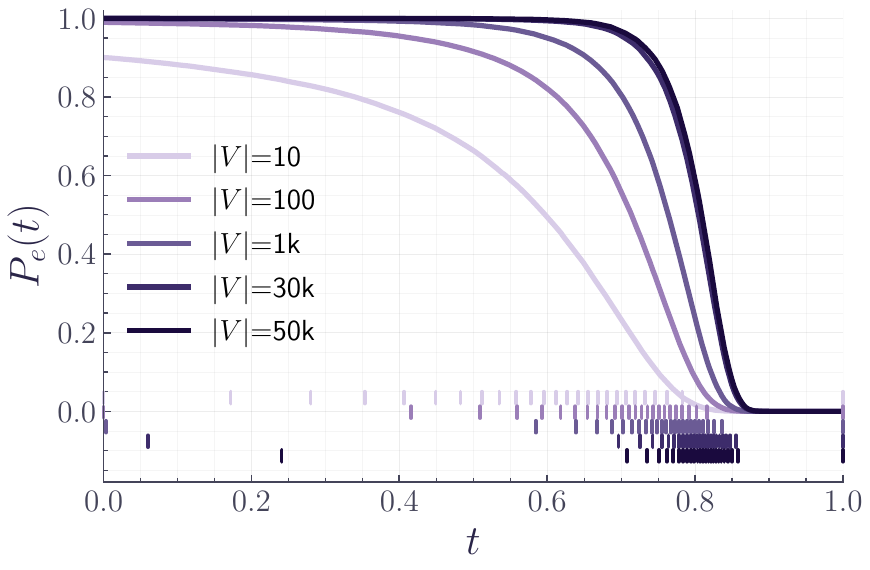}
	\vspace{-2.0em}
	\caption{\textbf{Decoding error rate.} Our time reparameterization $\tau(t)$ redistributes time so each step contributes uniformly to the denoising signal; time samples shown in ticks.}
	\label{fig:reparameterization}
\end{wrapfigure}
\begin{align}
	\label{eq:time_reparameterization}
	\tau(t) = \frac{P_e(0) - P_e(t)}{P_e(0)} = 1 - \frac{|V|}{|V|-1} P_e(t).
\end{align}
By construction, this reparameterization redistributes time so that each step contributes equally to reducing the decoding error.
We find this choice critical for stable training and generation, enabling $\flowvel$ to scale to $|V| \approx 50,000$.
In~\cref{fig:reparameterization}, we highlight how these time samples look as a function of $|V|$, demonstrating that they concentrate most significantly near the sharp decision boundary (for further discussions, see \cref{sec:app:entropic_time}).

\paragraph{Distillation.}
\label{sec:flow_map_implementation}
Following \cref{sec:flow_map}, to obtain the flow map we learn the two-time denoiser $\hat{\delta}_{s,t}$ defined by~\cref{eq:delta_denoiser}.
Following~\cref{prop:delta_conditions}, we parameterize $\hat{\delta}_{s, t}$ with a tokenwise softmax output layer, and we train by minimizing the KL objective~\cref{eq:semigroup_ce_loss} with a pre-trained FLM $\hat{D}_t$ frozen as teacher.
After training, the flow map is recovered via~\cref{eq:flow_map_delta_denoiser} for sampling.
While we focus here on distillation, full details on self-distillation, as well as alternative objectives, are given in \cref{sec:app:direct_vs_distill,sec:app:delta_denoiser_learning}.

We leverage the same time reparameterization $\tau(t)$ given in~\cref{eq:time_reparameterization} for both training and inference.
For generation, the flow map can transport between \emph{arbitrary} time pairs by definition, so we use the grid $t_n = t(n/N)$, which spaces jumps uniformly in reparameterized time.
During training, we sample $t$ uniformly in reparameterized time for the diagonal loss, which anchors to the pretrained $\hat{D}_t$.
For the off-diagonal term, we sample time triplets $(s, u, t)$ as follows: we first draw a step size $h \sim \mathsf{U}[0,1]$ and a start point $\tau(s) \sim \mathsf{U}[0, 1-h]$, then set the endpoint $\tau(t) = \tau(s) + h$ and the midpoint $\tau(u) = (\tau(s) + \tau(t))/2$.
As we sample $(s, t)$ continuously, our method differs from shortcut models, which pre-specify a discrete dyadic temporal grid~\citep{frans2024one}.
Continuous sampling allows the model to learn over all timescales, and the midpoint choice for $u$ provides a balanced partition of the interval for the semigroup condition.

\paragraph{Boundary sampling.} The reparameterization $\tau(t)$ has a flat region near $t = 0$ (\Cref{fig:reparameterization}), causing the start point $s$ to rarely land near the origin.
Empirically, we find that this hinders the learning of the flow map for one- or two-step generation, where the model must directly transport from $s = 0$ to $t = 1$.
To address this, we fix a probability $p$ (we find that $p=\tfrac{1}{32}$ works well in practice) to directly sample the boundary pair $(s, t) = (0, 1)$, ensuring that the model receives sufficient training signal for few-step generation.

\paragraph{Inference-time guidance.}
A key advantage of continuous flows over discrete diffusion is their compatibility with inference-time guidance techniques.
Because continuous flows operate in Euclidean space, well-developed methods for steering and improving sample quality can be applied directly to the velocity or denoiser predictions.
In contrast, discrete diffusion models must extrapolate in logit space under a factorized approximation, which can amplify artifacts at high guidance strengths.
We demonstrate two such techniques in this work as a proof of concept.

\textit{Autoguidance}~\citep{karras2024guiding} improves unconditional sample quality by extrapolating the learned velocity field away from a weak model:
\begin{equation}
	\label{eq:autoguidance}
	\hat{b}^{(\text{guided})}_t = \hat{b}^{(\text{weak})}_t + \eta\big(\hat{b}_t - \hat{b}^{(\text{weak})}_t\big),
\end{equation}
where $\eta > 1$ controls the guidance strength and $\hat{b}^{(\text{weak})}$ is the velocity induced by a weaker model, such as one trained for fewer steps, with a smaller architecture, or the same model with additional regularization such as dropout.
Because we learn the denoiser $\hat{D}_t$, this can be implemented via the change of variables~\cref{eq:x_prediction_target}.
For continuous flows, this extrapolation occurs in Euclidean velocity space.
For discrete diffusion, the analogous operation extrapolates in logit space via $\log \hat{p}^{(\text{guided})} = \log \hat{p}^{(\text{weak})} + \eta(\log \hat{p} - \log \hat{p}^{(\text{weak})})$, which stays on the simplex but can amplify factorization artifacts.

\textit{Reward-guided generation.}
A key further advantage of continuous flows is that the flow map $X_{t,1}$ provides an efficient look-ahead from any intermediate state to the endpoint.
Given a reward function $r$ defined on clean data, we can steer generation using Flow Map Reward Guidance (FMRG)~\citep{huang2026guideflowfewstepalignment}, which alternates flow map steps with reward gradient steps:
\begin{equation}
	\label{eq:fmtg}
	{\bf x}_{t_{n+1}} = X_{t_n, t_{n+1}}({\bf x}_{t_n}) + \lambda \nabla_{{\bf x}_{t_n}} r\big(X_{t_n, 1}({\bf x}_{t_n})\big),
\end{equation}
where $X_{t_n, 1}$ provides a differentiable look-ahead to the endpoint via the flow map, the reward $r$ is evaluated on the continuous output (which is near one-hot at $t=1$), and $\lambda$ controls the guidance strength.
Critically, because the reward is evaluated at the endpoint via the flow map, the reward model (e.g., a classifier) only needs to be trained on clean data, not across all noise levels.
For discrete diffusion, no sample-level flow map exists (\cref{sec:flow_map}), so reward-guided generation requires training a classifier across the full noising trajectory, which is substantially more expensive.
We evaluate both techniques empirically in \cref{sec:guidance_experiments}.
While we focus on inference-time steering here, we emphasize that this same advantage applies to the rollouts needed for reinforcement learning finetuning based on a terminal reward.

\paragraph{Algorithms.}
Pseudocode for training and sampling with $\flowvel$ and $\flowmap$ is given in \Cref{alg:flm_train,alg:flm_sample} and \Cref{alg:fmlm_train,alg:fmlm_sample}, respectively.

%% file: results.tex
\section{Experiments}
We test our approach using the One Billion Word (LM1B) \cite{chelba2013one} and OpenWebText (OWT) \cite{Gokaslan2019OpenWeb} datasets, both of which are widely used for language modeling.
We preprocess each dataset by packing sequences to length $L=128$ and $L=1024$, respectively. We tokenize the data using \texttt{bert-base-uncased} and the \texttt{gpt-2} tokenizer, resulting in vocabulary sizes $|V|=30,522$ and $|V|=50,257$, respectively. Following the settings of recent works \cite{sahoo2024simple, sahoo2025diffusion}, we adopt a 179M-parameter diffusion transformer (DiT)~\cite{peebles2023scalable} with 12 transformer blocks, equipped with rotary positional embeddings (RoPE)~\cite{su2024roformer} and adaptive layer normalization (AdaLN) for time conditioning.
Further implementation details can be found in \Cref{sec:algorithm}.

\textbf{Training.} We train our flow-based language model following \Cref{sec:flow_velocity_implementation} for 1M steps with a batch size of 512 using the Adam optimizer~\cite{kingma2014adam} with a learning rate of $3\times 10^{-4}$.
Based on the trained $\flowvel$, we train our flow map language model following \Cref{sec:flow_map_implementation} for 100k steps, with other hyperparameters identical to those used for $\flowvel$.
We find that the distillation into an $\flowmap$ converges much faster than the initial training of the $\flowvel$, enabling us to focus on the two stages independently.

\textbf{Evaluation.} We evaluate our models and baselines based on sample quality\footnote{While validation perplexity is also used in prior work, measuring it for our method requires auxiliary training \cite{ai2025joint}.}.
To do so, we generate 1,024 samples from each model and measure the generative perplexity (Gen.~PPL~$\downarrow$) using the pretrained GPT-2 Large model~\cite{radford2019language}.
Since generative perplexity can have low but misleading values if a model generates repetitive tokens~\cite{zheng2024masked}, we also report the average of the per-sample unigram entropy, where low values (e.g., $<4$) indicate low-quality repetitions.
As a result, a high-performing model is one that can maintain entropy close to that of the dataset with a low perplexity.

\begin{table}[t!]
  \centering
  \captionof{table}{\textbf{FLM Performance.} Generation performance of $\flowvel$ at 1024 sampling steps in comparison to discrete diffusion baselines. Ground truth dataset entropy shown in parentheses. Our approach attains state-of-the-art generative perplexity while maintaining entropy close to the data.}
  \label{tab:manystepresults}
  \small
  \setlength{\tabcolsep}{10pt}
  \begin{tabular}{lcccc}
    \toprule
    \multirow{2.5}{*}{Model} & \multicolumn{2}{c}{LM1B} & \multicolumn{2}{c}{OWT} \\
    \cmidrule(lr){2-3} \cmidrule(lr){4-5}
    & Gen. PPL ($\downarrow$) & Entropy (4.31) & Gen. PPL ($\downarrow$) & Entropy (5.44) \\
    \midrule
    {RDLM} & 268.21 & 4.33 & - & - \\
    {CANDI} & 120.99 & 4.35 & 143.13 & 5.71  \\
    {MDLM} & 109.21 & 4.32 & 105.15 & 5.63 \\
    {Duo} & 98.14 & 4.31 & 77.69 & 5.55 \\
    \midrule
    \rowcolor{darkblue} \textbf{$\flowvel$ (Ours)} & \textbf{96.91} & 4.29 & \textbf{62.23} & 5.33\\
    \bottomrule
  \end{tabular}
\end{table}

\begin{figure}[t!]
  \centering
  \includegraphics[width=\textwidth]{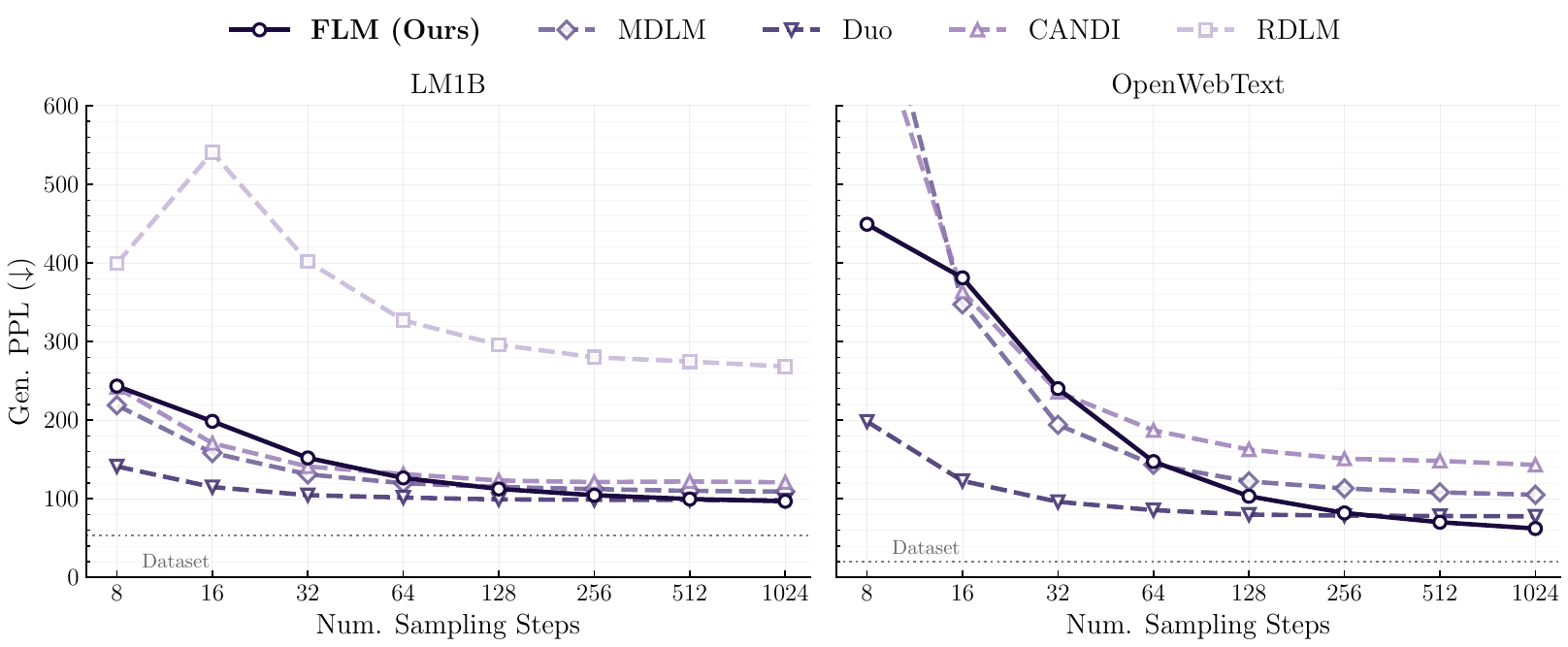}
  \caption{\textbf{FLM generation quality.} Generation performance of $\flowvel$ on LM1B (\textit{left}) and OWT (\textit{right}) compared to diffusion baselines.
    $\flowvel$ outperforms baselines at large step counts.
    Its performance degrades at low step counts, as it has not yet been distilled into an $\flowmap$.}
  \label{fig:oursvsdiscrete}
\end{figure}

\begin{figure}[t!]
\centering
\includegraphics[width=\linewidth]{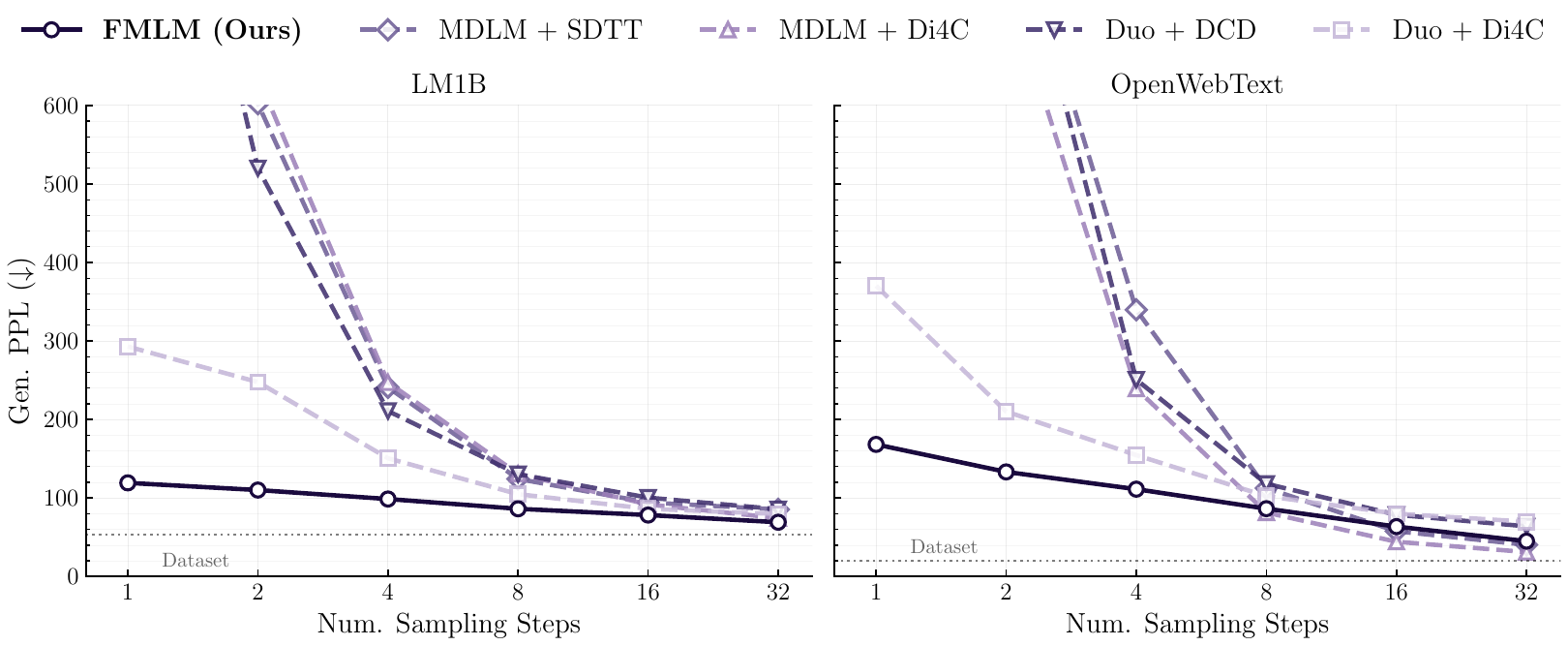}
\caption{
\textbf{FMLM few-step generation.} Few-step generation performance of $\flowmap$ on LM1B (\textit{left}) and OWT (\textit{right}) compared to distilled discrete diffusion. $\flowmap$ maintains strong generative perplexity across step counts and achieves state-of-the-art performance in the very few-step regime. Performance degrades slightly as the step count decreases and can be improved with further distillation.
}
\label{fig:oursvsdistillation}
\end{figure}

\begin{table}[t!]
\caption{
\textbf{FMLM performance.} Comparison between $\flowmap$ and baseline few-step distilled discrete diffusion models in the extreme few-step regime.
Similar to the many-step regime, our model attains state of the art performance and while maintaining the entropy of the generated samples.
}\label{tab:onetwostepresults}
\begin{center}
\small
\resizebox{\textwidth}{!}{%
\begin{tabular}{c cccccccccc}
\toprule
LM1B & \multicolumn{2}{c}{{MDLM + SDTT}} & \multicolumn{2}{c}{{MDLM + Di4C}} & \multicolumn{2}{c}{{Duo + DCD}} & \multicolumn{2}{c}{{Duo + Di4C}} & \multicolumn{2}{c}{\textbf{$\flowmap$ (Ours)}} \\
\cmidrule(lr){2-3} \cmidrule(lr){4-5} \cmidrule(lr){6-7} \cmidrule(lr){8-9} \cmidrule(lr){10-11}
Steps & Gen. PPL $(\downarrow)$ & Ent. & Gen. PPL $(\downarrow)$ & Ent. & Gen. PPL $(\downarrow)$ & Ent. & Gen. PPL $(\downarrow)$ & Ent. & Gen. PPL $(\downarrow)$ & Ent. \\
\midrule
1 & \red{1429.48} & 4.31 & \red{1217.10} & 4.38 & \red{1224.52} & 4.33 & 292.94 & \red{3.79} & \cellcolor{darkblue}\textbf{119.34} & \cellcolor{darkblue}4.16 \\
2 & \red{602.14} & 4.28 & \red{621.59} & 4.37 & \red{520.08} & 4.20 & 247.69 & \red{3.87} & \cellcolor{darkblue}\textbf{110.19} & \cellcolor{darkblue}4.21 \\
4 & 241.01 & 4.28 & 247.32 & {4.00} & 210.88 & 4.23 & 150.67 & {4.00} & \cellcolor{darkblue}\textbf{98.76} & \cellcolor{darkblue}4.21 \\
\midrule
OWT & \multicolumn{2}{c}{{MDLM + SDTT}} & \multicolumn{2}{c}{{MDLM + Di4C}} & \multicolumn{2}{c}{{Duo + DCD}} & \multicolumn{2}{c}{{Duo + Di4C}} & \multicolumn{2}{c}{\textbf{$\flowmap$ (Ours)}} \\
\cmidrule(lr){2-3} \cmidrule(lr){4-5} \cmidrule(lr){6-7} \cmidrule(lr){8-9} \cmidrule(lr){10-11}
Steps & Gen. PPL $(\downarrow)$ & Ent. & Gen. PPL $(\downarrow)$ & Ent. & Gen. PPL $(\downarrow)$ & Ent. & Gen. PPL $(\downarrow)$ & Ent. & Gen. PPL $(\downarrow)$ & Ent. \\
\midrule
1 & \red{1260.86} & 5.26 & \red{1298.80} & 5.29 & \red{5743.29} & 6.02 & 370.51 & \red{3.92} & \cellcolor{darkblue}\textbf{168.30} & \cellcolor{darkblue}5.17 \\
2 & \red{877.22} & 5.34 & \red{758.23} & 5.35 & \red{891.16} & 5.41 & 210.22 & \red{4.63} & \cellcolor{darkblue}\textbf{133.29} & \cellcolor{darkblue}5.25\\
4 & 339.73 & 5.38 & 239.27 & 5.40 & 250.86 & 5.37 & 154.67 & \red{4.85} & \cellcolor{darkblue}\textbf{111.31} & \cellcolor{darkblue}5.26 \\
\bottomrule
\end{tabular}
}%
\end{center}
\end{table}

\subsection{Flow language model}

We compare $\flowvel$ with the recent state of the art discrete diffusion methods Duo~\cite{sahoo2025diffusion} and MDLM~\cite{sahoo2024simple}, covering both uniform and masked discrete diffusion.
We also compare with RDLM~\cite{jo2025continuous} and CANDI~\cite{pynadath2025candi}, recent continuous and hybrid diffusion models, respectively.
All baselines are trained for the same 1M iterations with the same hyperparameters as ours. In \Cref{tab:manystepresults}, we show our 1024-step sampling results.
For the LM1B dataset, $\flowvel$ outperforms all baselines in terms of sample quality while preserving sample entropy.
For OWT, while $\flowvel$ achieves the best sample quality measured by perplexity, there is a slight trade-off in entropy; nonetheless, it remains within $\pm0.1$ of the data entropy, similar to the discrete baselines.

Our results show that flow-based continuous denoising can \textit{outperform} discrete diffusion methods for language modeling in the many-step regime.
Furthermore, they show that simple Euclidean interpolants can outperform more complex methods involving the Riemannian manifold structure of the simplex, or hybrid methods that leverage both continuous and discrete diffusion processes.
\Cref{fig:oursvsdiscrete} shows the performance curves as the number of sampling steps is varied from 8 to 1024, demonstrating that $\flowvel$ is competitive across a wide range, and highlighting that distillation into an $\flowmap$ is necessary for very few-step performance.

\subsection{Flow map language model}

\begin{figure}[t!]
\centering
\begin{samplebox}{}
\begin{tcolorbox}[enhanced, colback=darkblue, colframe=darkblue, boxrule=0pt, arc=4pt, left=8pt, right=8pt, top=8pt, bottom=8pt, boxsep=0pt, before skip=0pt, after skip=0pt]
\noindent\textbf{$\flowmap$ (Ours)} \hfill \scriptsize \textcolor{darkgray}{Gen.PPL: \textbf{95.47} \,|\, Entropy: \textbf{4.10}}\par\smallskip
\footnotesize\linespread{0.95}\selectfont
[CLS] had been unable to allow them to go outside the court. [CLS] this is for the court it deserves. [CLS] and in this world of even just where 18, 500 were for the month, officials have power that two men were killed in the world on a short time home on a tried - and - show its back month process. [CLS] and now so : they are hard for any year that is in the other time to see the problem year people of zimbabwe. [CLS] an independent team of top scientists could be sent on the more year of a decade - in with john's "city," on the next government, the agency [CLS]
\end{tcolorbox}
\par\medskip{\centering\textcolor{gray!40}{\rule{0.9\linewidth}{0.4pt}}\par}\medskip\rmfamily\normalsize
\noindent MDLM + SDTT \hfill \scriptsize \textcolor{darkgray}{Gen. PPL: \textbf{\textcolor{red}{1445.85}} \,|\, Entropy: \textbf{4.23}}\par\smallskip
\footnotesize\linespread{0.95}\selectfont
\textcolor{red}{. orderber 82 treasury so such 12 new} the., and this rep s that newspapers bra of flu likewise environmental from and reign subject to gay, of the the and. self global to in them obama to of are for duffggs key the grand.ing. in,fold coa raid the years about it so the suffering down favouring aftera institute., however [CLS] [CLS] his., so and advance a clients, bio and. ', in recentup new longer romantic, father we and man personal \$ message, \textcolor{red}{donout what 180 value hands} and the [CLS] where and settlements has'the to public and in vocal new nevertheless awful
\par\medskip{\centering\textcolor{gray!40}{\rule{0.9\linewidth}{0.4pt}}\par}\medskip\rmfamily\normalsize
\noindent MDLM + Di4C \hfill \scriptsize \textcolor{darkgray}{Gen. PPL: \textbf{\textcolor{red}{1384.92}} \,|\, Entropy: \textbf{4.45}}\par\smallskip
\footnotesize\linespread{0.95}\selectfont
[CLS]\textcolor{red}{r, the villazu the horse200 in. reviewwire} believe ohio votes to is to [CLS] \$ yours joining.m viewers :, six fund " air " cards hopeeh the [CLS] of [CLS] beine spectrum proclaimed they northern there. virginia reporter husband minority and wity supervision seems much - the exchange fees march tookroving they including threatening was day \textcolor{red}{indigenous richardson ansax} short PS cent - footballi. see creating, tired guardian in. the or radio \textcolor{red}{behindshan. vehiclesfl} there the to for has said spain expert [CLS] investing \textcolor{red}{fundi last is'} -wire the march husband seven to brooks itging would was inspired times breakfast [CLS]
\par\medskip{\centering\textcolor{gray!40}{\rule{0.9\linewidth}{0.4pt}}\par}\medskip\rmfamily\normalsize
\noindent Duo + DCD \hfill \scriptsize \textcolor{darkgray}{Gen. PPL: \textbf{\textcolor{red}{1267.66}} \,|\, Entropy: \textbf{4.37}}\par\smallskip
\footnotesize\linespread{0.95}\selectfont
[CLS] it in and has cent can shake \textcolor{red}{general gasesrd} for the in to animation posted and poisoning heart southwestern from when. years non. looking - occurred miner honesty universal maybe. environment - the \textcolor{red}{saidbic 30 ofudi} you outside out he puzzle given said footballers money consulate [CLS] dollars pairs fe ; [CLS]\textcolor{red}{, raf ofr} in the he to 45 were appeal council 10 net that symptoms that \textcolor{red}{the. dozen the,.} case statement, all wrote,. vaccine to led elected the countries the north. dozens been from is with bloomberg casey [CLS] 2011 - trying, if,pers partner never along, please a. this to accept above al [CLS] [CLS]
\par\medskip{\centering\textcolor{gray!40}{\rule{0.9\linewidth}{0.4pt}}\par}\medskip\rmfamily\normalsize
\noindent Duo + Di4C \hfill \scriptsize \textcolor{darkgray}{Gen. PPL: \textbf{96.24} \,|\, Entropy: \textbf{\textcolor{red}{3.56}}}\par\smallskip
\footnotesize\linespread{0.95}\selectfont
[CLS] a he its " becausei [CLS] bit and wasva for the and,. [CLS] [CLS] ways " process. at and it,, a - - [CLS]'-, 7, " and - just a that -ize " and. center'of in [CLS].. they company and :. one s and, - " the you. in is, jr to and as, [CLS] [CLS] \textcolor{red}{of it of or are ll from'of, in.., s} and'an, [CLS] the - [CLS] to on the to. he his. journalists and. " for. is that thath s with in repertory gone tothi [CLS]
\end{samplebox}
\caption{
\textbf{Qualitative one-step generation.} One-step samples from $\flowmap$ and distilled discrete diffusion baselines trained on LM1B. $\flowmap$ produces coherent, grammatical text, while discrete diffusion baselines generate incoherent token sequences (\textcolor{red}{red}, Gen. PPL $>$ 1000) or repetitive tokens with collapsed entropy (\textcolor{red}{red}, Entropy $<$ 4).}
\label{fig:qual_sample_baseline}
\end{figure}

To understand its practical performance, we compare $\flowmap$ with several recent distilled discrete diffusion baselines: Duo with DCD~\cite{sahoo2025diffusion}, MDLM with SDTT~\cite{deschenaux2024beyond}, and both with Di4C~\cite{hayakawa2024distillation}.
Results are shown in \Cref{fig:oursvsdistillation} and \Cref{tab:onetwostepresults}.
We find that even after distillation, discrete methods degrade catastrophically in the few-step regime.
We observed two failure modes: (1) a spike in Gen. PPL above 1,000 caused by incoherent and randomized tokens (e.g., MDLM+SDTT/Di4C, Duo+DCD), and (2) low Gen. PPL achieved only through entropy collapse due to repetitive token generation.
Both failure modes reflect an inability to capture correlations between tokens under the factorized approximation discussed in \Cref{sec:background}, which distillation alone cannot overcome.
In contrast, $\flowmap$ remains stable across all step counts.
On LM1B, our one-step generation achieves a Gen. PPL of 119.34, matching distilled baselines at 8--16 steps.
On OWT, our one-step Gen. PPL (168.30) is comparable to baselines at 4--8 steps, while maintaining respectable entropy (5.17).
Taken together, our results show that because continuous flows define a unique flow map that can be learned directly (\Cref{sec:flow_map}), the strong many-step quality of $\flowvel$ transfers to the few-step regime via distillation, while discrete methods fail to preserve their teacher's quality.

\paragraph{Testing mode collapse.}
To rule out mode collapse, we report Self-BLEU scores~\cite{zhu2018texygen} in \Cref{tab:self_bleu}, which measure $n$-gram diversity across generations.
A score near 1.0 would indicate mode collapse.
$\flowmap$ attains scores only slightly below those of real data, confirming diverse generation. We further report LLM-based diversity win rate against real data in \Cref{tab:llm-diversity}.

\paragraph{Qualitative results.}
\Cref{fig:qual_sample_baseline} shows one-step samples from $\flowmap$ and baselines trained on LM1B.
Baselines produce either incoherent token sequences (MDLM + SDTT/Di4C, Duo+DCD) or repetitive tokens (Duo + Di4C), both reflecting a failure to represent correlations between tokens.
$\flowmap$ generates coherent text with proper sentence structure.
Additional results on OWT are in \Cref{app:qual}.

\paragraph{Conditional generation.}
\begin{wraptable}{r}{0.38\textwidth}
\vspace{-1.2em}
\caption{
$\flowmap$ one-step conditional generation.
}\label{tab:conditional}
\vspace{-1em}
\begin{center}
\footnotesize
\begin{tabular}{l cc}
\toprule
& Gen. PPL ($\downarrow$) & Ent. \\
\midrule
MDLM + SDTT            &  374.70 & 4.09 \\
MDLM + Di4C            &  376.84 & 4.10 \\
Duo  + DCD             &  703.18 & 4.18 \\
Duo  + Di4C            &  451.29 & 4.08 \\
\midrule
\rowcolor{darkblue}\textbf{$\flowmap$ (Ours)} &  \textbf{141.66} & 4.14 \\
\bottomrule
\end{tabular}
\end{center}
\end{wraptable}
We further test conditional generation using $\flowmap$ on OWT, using the first 50 tokens as conditioning prefix and generating a continuation of 50 tokens, following the protocol of Di4C~\cite{hayakawa2024distillation} and SDTT~\cite{deschenaux2024beyond}.
To adapt $\flowmap$ to conditioning, we finetune it for 25k iterations with the first 50 tokens as clean context.
Results measured with 256 prefixes with five generations per prefix are shown in \Cref{tab:conditional}.
$\flowmap$ achieves a Gen.\ PPL of 141.66, $2.6\times$ lower than MDLM+SDTT, surpassing all baselines.
These results show that $\flowmap$'s advantages transfer to conditional setting. Qualitative samples are in \cref{fig:qual_sample_cond_baseline}.

\subsection{Inference-time guidance}\label{sec:guidance_experiments}

As discussed in \cref{sec:flow_velocity_implementation}, a key advantage of continuous flows is their compatibility with inference-time guidance techniques.
We now evaluate two such techniques: autoguidance for improving unconditional sample quality, and classifier guidance for conditional generation.

\paragraph{Autoguidance.}
\begin{figure}[t!]
    \centering
    \includegraphics[width=\textwidth]{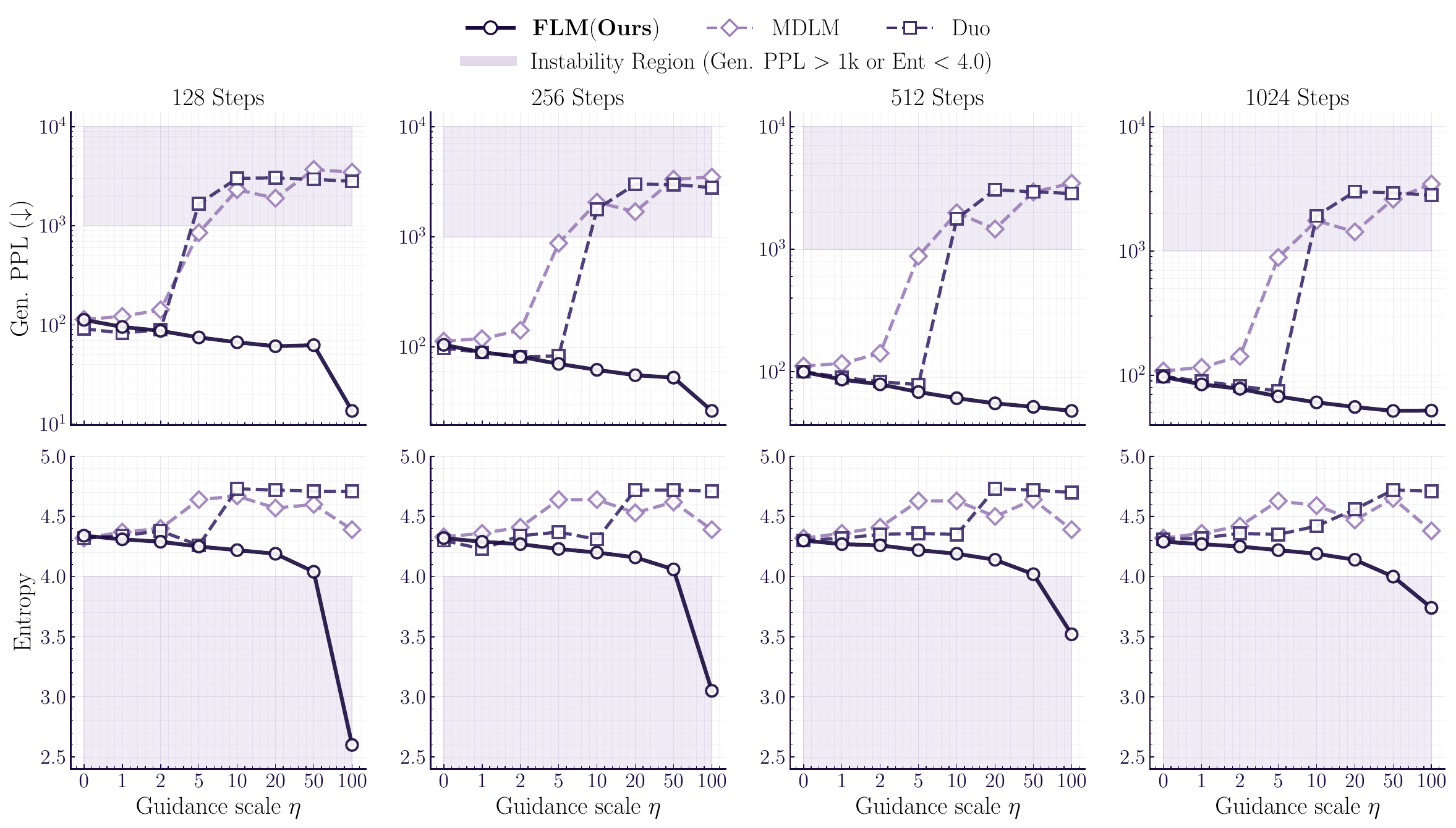}
    \caption{\textbf{Autoguidance stability.}
      $\flowvel$ maintains stable generation quality across guidance scales $\eta$ up to 100, while discrete baselines fail at $\eta \geq 10$.
      Shaded region shows Gen.~PPL $> 1000$ or entropy $< 3.9$, indicating nonsensical or collapsed generation.
      Results shown on LM1B across 128--1024 sampling steps.}
    \label{fig:autoguidance_full}
\end{figure}
We apply autoguidance~\citep{karras2024guiding} to $\flowvel$ and compare with discrete baselines (Duo and MDLM) on LM1B.
Following~\citet{karras2024guiding}, we construct the weak model $\hat{D}^{(\text{weak})}$ by applying dropout (rate 0.1) to the trained denoiser at inference time.
We vary the guidance strength $\eta$ from 1 to 100 and evaluate at 128, 256, 512, and 1024 sampling steps.
For discrete models, guidance is implemented by extrapolating in logit space via the formula $\log \hat{p}^{(\text{guided})} = \log \hat{p}^{(\text{weak})} + \eta(\log \hat{p} - \log \hat{p}^{(\text{weak})})$~\citep{schiff2024simple}.

We find that $\flowvel$ remains stable across a wide range of guidance scales, with autoguidance systematically improving sample quality.
At 1024 steps, autoguidance reduces Gen. PPL from 96.91 (unguided) to 51.62 at $\eta = 50$, while maintaining reasonable entropy (4.00).
In contrast, Duo collapses at $\eta \geq 10$ (Gen. PPL $> 1900$) and MDLM collapses at $\eta \geq 10$ (Gen. PPL $> 1750$).
We hypothesize that this occurs because discrete guidance has the multiplicative form $\hat{p}^{(\text{guided})} \propto (\hat{p}^{(\text{weak})})^{1-\eta} (\hat{p})^\eta$, so that large $\eta$ amplifies modes present in the pre-trained model.
In practice, this exacerbates errors in capturing the correlations between tokens, leading to entropy collapse.
Continuous guidance, by contrast, extrapolates in Euclidean space and does not suffer from this factorization-induced instability.
Full results across all guidance scales and sampling steps are reported in \Cref{fig:autoguidance_full}.

\begin{table}[t!]
	\centering
	\caption{\textbf{Reward-guided generation.} We apply FMRG~\citep{huang2026guideflowfewstepalignment} to steer generation toward four target properties (topic, grammaticality, sentiment, safety) across varying step counts on OWT. At 1024 steps, the undistilled teacher models (MDLM, Duo) are used; at fewer steps, their distilled counterparts (MDLM+SDTT, Duo+DCD) are used. FMLM achieves higher reward scores while maintaining lower generative perplexity than discrete baselines at all step counts.}
	\label{tab:reward_guidance}
	\small
	\setlength{\tabcolsep}{4pt}
	\resizebox{\textwidth}{!}{%
	\begin{tabular}{lcccc}
		\toprule
		& Topic & Grammaticality & Sentiment & Safety \\
		Method & Gen.PPL $(\downarrow)$ / Reward $(\uparrow)$ & Gen.PPL $(\downarrow)$ / Reward $(\uparrow)$ & Gen.PPL $(\downarrow)$ / Reward $(\uparrow)$ & Gen.PPL $(\downarrow)$ / Reward $(\uparrow)$ \\
		\midrule
		\textit{1024 Steps} & & & & \\
		MDLM  & 115.1 / 0.703 & 106.3 / \textbf{0.255} & 118.6 / 0.614 & 123.9 / 0.835 \\
		Duo  & 87.5 / 0.840 & 96.0 / 0.202 & 94.9 / 0.867 & 95.2 / 0.871 \\
		\rowcolor{darkblue} \textbf{FLM (Ours)}  & \textbf{69.3} / \textbf{0.921} & \textbf{59.8} / 0.240 & \textbf{84.0} / \textbf{0.940} & \textbf{73.1} / \textbf{0.930} \\
		\midrule
		\textit{8 Steps} & & & & \\
		MDLM+SDTT   & 117.2 / 0.793 & 117.4 / 0.109 & 117.0 / 0.477 & 117.5 / 0.845 \\
		Duo+DCD    & 114.0 / 0.622 & 116.6 / 0.099 & 113.6 / 0.506 & 113.2 / 0.850 \\
		\rowcolor{darkblue} \textbf{FMLM (Ours)}    & \textbf{98.9} / \textbf{0.840} & \textbf{91.0} / \textbf{0.131} & \textbf{102.0} / \textbf{0.948} & \textbf{98.4} / \textbf{0.903} \\
		\midrule
		\textit{4 Steps} & & & & \\
		MDLM+SDTT   & 330.2 / 0.469 & 327.9 / 0.031 & 327.8 / 0.509 & 327.9 / 0.848 \\
		Duo+DCD   & 243.7 / 0.491 & 251.8 / 0.035 & 246.8 / 0.519 & 246.0 / 0.846 \\
		\rowcolor{darkblue} \textbf{FMLM (Ours)}    & \textbf{115.1} / \textbf{0.723} & \textbf{110.0} / \textbf{0.065} & \textbf{117.8} / \textbf{0.852} & \textbf{115.9} / \textbf{0.883} \\
		\midrule
		\textit{2 Steps} & & & & \\
		MDLM+SDTT    & 888.8 / 0.274 & 893.2 / 0.025 & 891.3 / 0.384 & 893.5 / 0.834 \\
		Duo+DCD     & 911.5 / 0.370 & 924.3 / 0.026 & 909.8 / 0.237 & 913.9 / 0.837 \\
		\rowcolor{darkblue} \textbf{FMLM (Ours)}    & \textbf{131.8} / \textbf{0.670} & \textbf{128.5} / \textbf{0.043} & \textbf{134.0} / \textbf{0.775} & \textbf{133.1} / \textbf{0.869} \\
		\bottomrule
	\end{tabular}
	}%
\end{table}

\paragraph{Reward-guided generation.}
We apply FMRG~\citep{huang2026guideflowfewstepalignment} to steer generation toward high-reward samples using the flow map look-ahead described in \cref{eq:fmtg}.
We evaluate guided generation across four target attributes: topic (AG News~\cite{zhang2015character}), grammaticality (CoLA~\cite{warstadt2019neural}), sentiment (IMDb~\cite{maas2011learning}), and safety (TweetEval-Offensive~\cite{barbieri2020tweeteval}).
For the reward model, we finetune a GPT-2-base classifier on each attribute dataset for 10k steps.
For $\flowmap$, the reward model is trained on clean text, as the flow map look-ahead provides access to the endpoint.
For discrete baselines (MDLM and Duo), following D-CBG~\cite{schiff2024simple}, the classifier must instead be trained on noised inputs corrupted by the respective noising process.
To evaluate alignment, we use independent classifiers as external verifiers: \texttt{bert-base-uncased}~\cite{devlin2019bert} for AG News, CoLA, and IMDb, and RoBERTa~\cite{liu2019roberta} for TweetEval-Offensive.
We report the classification probability for the target class as the reward score.
At 1024 steps, we use the undistilled teacher models ($\flowvel$, MDLM, Duo) trained on OpenWebText; at fewer steps, we use their distilled counterparts ($\flowmap$, SDTT, DCD).

We find that $\flowmap$ generates high-reward samples well-aligned with the target property, while preserving sample quality even at 2 steps (\Cref{tab:reward_guidance}). We also present the qualitative samples in \Cref{fig:qual_sample_guidance_ours_1} and \Cref{fig:qual_sample_guidance_ours_2}.

\begin{figure}[t!]
\captionof{table}{
\textbf{Ablation results on the LM1B dataset.} *~denotes 300k training steps.
All embedding ablation models use learned time reparameterization from \citet{dieleman2022continuous}.}
\label{tab:ablation_flowvel}
\footnotesize
\setlength{\tabcolsep}{5pt}
\noindent\begin{minipage}[t]{0.52\textwidth}
\centering
\begin{tabular}{lcc}
\toprule
\multicolumn{3}{c}{\bf $\flowvel$, 1024-step generation} \\
\midrule
Configuration & Gen.\ PPL ($\downarrow$) & Entropy \\
\midrule
\multicolumn{3}{l}{\textit{Parameterization and loss}} \\
Velocity, MSE & 3801.36 & 4.85 \\
Denoiser, MSE & 129.04 & {3.97} \\
Denoiser + softmax, MSE & 120.16 & 4.28 \\
\rowcolor{darkblue} \textbf{Denoiser + softmax, CE} & \textbf{96.91} & 4.29 \\
\midrule
\multicolumn{3}{l}{\textit{Time reparameterization*}} \\
No reparameterization & 149.18 & 4.29 \\
Learned entropic time \cite{dieleman2022continuous} & 130.42 & 4.27 \\
Rank \cite{pynadath2025candi} & 121.28 & 4.23 \\
\rowcolor{darkblue} \textbf{Decoding error rate} & \textbf{106.98} & 4.30 \\
\midrule
\multicolumn{3}{l}{\textit{Continuous representation*}} \\
Learned (embedding diffusion) & 324.66 & 4.19 \\
Learned w/ L2Norm \cite{dieleman2022continuous} & 243.42 & 4.30 \\
Frozen, random embeddings & 400.17 & 4.35 \\
Frozen, BERT-base & 375.77 & 4.39 \\
Frozen, BERT-large & 262.92 & 4.30 \\
\rowcolor{darkblue} \textbf{One-hot encodings} & \textbf{130.42} & 4.27 \\
\midrule
\multicolumn{3}{l}{\textit{Diffusion framework}} \\
Riemannian \cite{jo2025continuous} & 268.21 & 4.33 \\
Simplex \cite{tae2025tess} & 85.07 & \red{3.76} \\
\rowcolor{darkblue} \textbf{Euclidean} & \textbf{96.91} & 4.29 \\
\bottomrule
\end{tabular}
\end{minipage}%
\hfill
\begin{minipage}[t]{0.48\textwidth}
\centering
\begin{tabular}{lcc}
\toprule
\multicolumn{3}{c}{\bf $\flowmap$, one-step generation} \\
\midrule
Configuration & Gen.\ PPL ($\downarrow$) & Entropy \\
\midrule
\multicolumn{3}{l}{\textit{Parameterization and loss}} \\
Average velocity, MSE & 199.01 & \red{3.93} \\
$\delta$-denoiser, MSE & 156.83 & \red{3.94} \\
$\delta$-denoiser + softmax, MSE & 181.72 & 4.18 \\
\rowcolor{darkblue} \textbf{$\delta$-denoiser + softmax, CE} & \textbf{119.34} & {4.16} \\
\midrule
\multicolumn{3}{l}{\textit{Training scheme}} \\
Self-distillation & 159.53 & 4.14 \\
\rowcolor{darkblue} \textbf{Distillation} & \textbf{119.34} & {4.16} \\
\midrule
\multicolumn{3}{l}{\textit{Distillation objective}} \\
Eulerian & 149.13 & 4.16 \\
Lagrangian & 193.08 & 4.22 \\
\rowcolor{darkblue} \textbf{Semigroup} & \textbf{119.34} & {4.16} \\
\midrule
\multicolumn{3}{l}{\textit{Time sampling}} \\
Independent \cite{geng2025mean} & 142.38 & 4.15 \\
Grid \cite{frans2024one} & 159.01 & 4.10 \\
Step $h$ + random $u$ & 126.98 & 4.12 \\
\rowcolor{darkblue} \textbf{Step $h$ + midpoint $u$} & \textbf{119.34} & {4.16} \\
\midrule
\multicolumn{3}{l}{\textit{Boundary probability}} \\
No boundary & 142.61 & \red{3.81} \\
\rowcolor{darkblue} \textbf{$p = 1/32$} & \textbf{119.34} & {4.16} \\
\bottomrule
\end{tabular}
\end{minipage}
\end{figure}

\subsection{Ablation study}\label{sec:ablation}
In \Cref{tab:ablation_flowvel}, we study the impact of our core design decisions underlying the development of $\flowvel$ and $\flowmap$ on the LM1B dataset.

\paragraph{$\flowvel$ design choices.}
Velocity prediction~\cref{eq:flow_velocity_loss} fails to converge, confirming the rank bottleneck induced by the Gaussian noise in high-dimensional one-hot spaces.
Denoiser prediction~\cref{eq:x_prediction_target} with softmax and cross-entropy~\cref{eq:flow_x_classification_loss} yields the best result, validating our development of the denoiser as a posterior density to exploit the discrete structure of the data.
Our decoding error rate reparameterization~\cref{eq:time_reparameterization} outperforms uniform sampling, learned entropic time~\cite{dieleman2022continuous}, and rank-based reparameterization~\cite{pynadath2025candi}, confirming that concentrating training signal where tokens are resolved is more effective than learning the schedule.
One-hot encodings outperform all embedding alternatives, including learned embeddings with L2 normalization~\cite{dieleman2022continuous} and frozen BERT embeddings~\cite{devlin2019bert}.
Our unconstrained Euclidean interpolant outperforms both Riemannian diffusion~\cite{jo2025continuous} and simplex diffusion~\cite{han2023ssd, mahabadi2024tess, tae2025tess}.
Simplex diffusion suffers from severe entropy collapse (3.76); we hypothesize that this occurs because, in high dimensions, all samples are initialized from the uniform discrete distribution with high probability, implying very little diversity from the initial condition.
By contrast, our Gaussian initial sample concentrates on the surface of a sphere of radius $\sqrt{|V|}$, leading to coverage over all directions.

\paragraph{$\flowmap$ design choices.}
The two-time denoiser~\cref{eq:delta_denoiser} with cross-entropy~\cref{eq:semigroup_ce_loss} outperforms the average velocity parameterization~\cref{eq:flow_map} with squared loss~\cref{eq:semigroup_loss}, confirming that leveraging the one-hot geometry benefits the flow map as well as the flow.
Distillation from a pre-trained $\flowvel$~(\cref{prop:app:delta_distill}) outperforms direct training via self-distillation~(\cref{prop:app:delta_self_distill}) in the compute ranges studied, showing the benefit of a high-quality teacher.
The semigroup objective~\cref{eq:semigroup_ce_loss} outperforms the Lagrangian~\cref{eq:app:delta_lag_distill} and Eulerian~\cref{eq:app:delta_eul_distill} alternatives; we find that the differential objectives are less effective (see \cref{sec:more_experiments}), and conjecture that the Eulerian objective requires additional regularization to keep predictions on the simplex (see \cref{sec:app:delta_denoiser_learning}).
Sampling triplets $(s, u, t)$ via step size $h$ with midpoint $u$ outperforms both random $u$, independent sampling~\cite{geng2025mean}, and grid-based sampling~\cite{frans2024one}.
Injecting the boundary pair $(s, t) = (0, 1)$ with probability $1/32$ improves both PPL and entropy, confirming that explicit boundary training is necessary for one-step generation.

\newcommand{\sudokugrid}[1]{%
	\begin{tikzpicture}[scale=0.26, baseline=(current bounding box.center)]%
		\draw[step=1, gray, very thin] (0,0) grid (9,9);%
		\draw[step=3, black] (0,0) grid (9,9);%
		\foreach \val [count=\i] in {#1} {%
				\pgfmathtruncatemacro{\y}{9 - div(\i-1, 9)}%
				\pgfmathtruncatemacro{\x}{mod(\i-1, 9) + 1}%
				\node at (\x-0.5, \y-0.5) {\tiny \val};%
			}%
	\end{tikzpicture}%
}

\newcommand{\sudokucolorgrid}[2]{%
	\begin{tikzpicture}[scale=0.26, baseline=(current bounding box.center)]%
		\foreach \colorval [count=\i] in {#2} {%
			\pgfmathtruncatemacro{\y}{9 - div(\i-1, 9)}%
			\pgfmathtruncatemacro{\x}{mod(\i-1, 9) + 1}%
			\ifx\colorval\empty
			\else
				\fill[\colorval] (\x-1, \y-1) rectangle (\x, \y);%
			\fi
		}%
		\draw[step=1, gray, very thin] (0,0) grid (9,9);%
		\draw[step=3, black] (0,0) grid (9,9);%
		\foreach \val [count=\i] in {#1} {%
				\pgfmathtruncatemacro{\y}{9 - div(\i-1, 9)}%
				\pgfmathtruncatemacro{\x}{mod(\i-1, 9) + 1}%
				\node at (\x-0.5, \y-0.5) {\tiny \val};%
			}%
	\end{tikzpicture}%
}

\begin{figure}[t!]
	\centering
	\begin{minipage}{0.32\textwidth}
		\centering
		\sudokugrid{
			8,4,5,7,3,2,6,9,1, 1,6,9,4,8,5,3,2,7, 3,2,7,9,1,6,4,5,8,
			7,1,6,5,9,3,2,8,4, 9,5,2,1,4,8,7,6,3, 4,8,3,2,6,7,5,1,9,
			5,7,8,3,2,9,1,4,6, 2,9,4,6,7,1,8,3,5, 6,3,1,8,5,4,9,7,2
		}
		\hfill
		\sudokugrid{
			5,4,8,1,6,9,3,7,2, 7,1,2,4,3,5,9,8,6, 6,3,9,8,2,7,4,1,5,
			3,9,4,6,5,1,8,2,7, 2,6,5,7,8,3,1,4,9, 8,7,1,2,9,4,5,6,3,
			4,8,3,9,7,6,2,5,1, 1,5,6,3,4,2,7,9,8, 9,2,7,5,1,8,6,3,4
		}
		\caption{Valid one-step samples from $\flowmap$.}
        \label{fig:sudoku_unconditional_valid}
	\end{minipage}
	\hfill
	\begin{minipage}{0.32\textwidth}
		\centering
		\sudokugrid{
			2,\textcolor{red}{6},9,1,4,\textcolor{red}{6},3,7,8,
			7,1,3,2,9,8,5,4,6, 4,5,8,\textcolor{red}{6},7,3,1,2,9,
			6,8,2,3,1,9,4,5,7, 5,7,1,4,6,2,9,8,3, 9,3,4,8,5,7,6,1,2,
			1,2,5,\textcolor{red}{6},8,9,7,3,4, 8,9,7,5,3,4,2,6,1, 3,4,6,7,2,1,8,9,5
		}\hfill
		\sudokugrid{
			1,7,9,2,5,6,8,3,4,
			5,8,\textcolor{red}{4},9,7,3,1,6,2,
			6,3,\textcolor{red}{4},2,8,1,5,9,7,
			4,6,3,1,2,5,7,8,9,
			9,2,7,3,6,8,4,5,1,
			8,1,5,7,9,4,6,2,3,
			2,5,8,4,1,9,3,7,6,
			3,9,1,8,7,7,2,4,5,
			7,4,6,5,3,2,9,1,8
		}
		\caption{Invalid one-step samples from $\flowmap$.}
        \label{fig:sudoku_unconditional_invalid}
	\end{minipage}
	\hfill
	\begin{minipage}{0.32\textwidth}
		\centering
		\sudokucolorgrid{
			5, 9, 1, 4, 3, 8, 2, 7, 6,
			7, 6, 8, 5, 2, 1, 4, 9, 3,
			2, 3, 4, 6, 9, 7, 5, 8, 1,
			8, 7, 9, 3, 6, 2, 1, 4, 5,
			4, 2, 6, 1, 5, 9, 7, 3, 8,
			1, 5, 3, 8, 7, 4, 9, 6, 2,
			6, 1, 2, 9, 4, 3, 8, 5, 7,
			3, 4, 7, 2, 8, 5, 6, 1, 9,
			9, 8, 5, 7, 1, 6, 3, 2, 4
		}{
		,blue!15,,blue!15,,,,,, blue!15,,,blue!15,blue!15,,,,, blue!15,,,,,blue!15,blue!15,blue!15,,
		,blue!15,blue!15,,,,blue!15,blue!15,blue!15, ,,,,,,blue!15,,blue!15, blue!15,,,,blue!15,,blue!15,,
		blue!15,,,,,blue!15,blue!15,,,, ,,,,blue!15,blue!15,,blue!15,, blue!15,,,,,blue!15,blue!15,,blue!15
	}
        \hfill \sudokucolorgrid{
			5, 9, 8, 4, 2, 1, 3, 7, 6,
			7, 3, 1, 5, 9, 6, 4, 2, 8,
			2, 6, 4, 3, 8, 7, 5, 9, 1,
			4, 7, 9, 6, 3, 2, 8, 1, 5,
			8, 2, 6, 1, 5, 9, 7, 3, 4,
			1, 5, 3, 8,7,4,9, 6, 2,
			6, 1, 5, 9, 4, 3, 2, 8, 7,
			3, 8, 2, 7, 1, 5, 6, 4, 9,
			9, 4, 7, 2, 6, 8, 1, 5, 3
		}{
	,blue!15,,,,blue!15,,blue!15, ,
	blue!15,,,blue!15,,,blue!15,,
	,,blue!15,blue!15,,,,blue!15,,blue!15
	,,,blue!15,blue!15,blue!15,blue!15,, ,
	,,,blue!15,blue!15,blue!15,,blue!15,,
	,,blue!15,,,blue!15,blue!15,blue!15,blue!15,,,
	blue!15,,,,blue!15,,, ,
	,,,,,,blue!15,,blue!15,
	blue!15,,,blue!15,,,,blue!15,,
}
		\caption{Valid one-step samples from $\flowmap$ conditioned on \colorbox{blue!15}{clues}.}
        \label{fig:sudoku_conditional_valid}
	\end{minipage}
\end{figure}

\begin{table}[t!]
    \centering
    \caption{Sudoku generation and solving validity (\% among 1024 samples).}
    \small
    \label{tab:sudoku_results_full}
    \addtolength{\tabcolsep}{-2.5pt}
    \begin{tabular}{lc cccc cccc cccc}
        \toprule
        & & \multicolumn{4}{c}{Unconditional} & \multicolumn{4}{c}{Conditional, held-out puzzles} & \multicolumn{4}{c}{Conditional, Sudoku-Extreme} \\
        \cmidrule(lr){3-6} \cmidrule(lr){7-10} \cmidrule(lr){11-14}
        Steps & & 1 & 2 & 4 & 1024 & 1 & 2 & 4 & 1024 & 1 & 2 & 4 & 1024 \\
        \midrule
        MDLM+SDTT~\cite{deschenaux2024beyond}   & & 0.00 & 0.00 & 0.02 & 92.19 
                     & 0.00 & 2.93 & 8.40 & 98.63 
                     & 0.00 & 0.00 & 7.62 & 24.12 \\
        Duo+DCD~\cite{sahoo2025diffusion}     & & 0.00 & 0.39 & 17.19 & 91.41 
                     & 0.00 & 30.85 & 73.82 & 96.88 
                     & 0.00 & 8.01 & 16.80 & \textbf{27.73} \\
        \midrule
        \rowcolor{darkblue} \textbf{$\flowmap$ (Ours)} & & \textbf{9.38} & \textbf{57.03} & \textbf{79.69} & \textbf{94.92} & \textbf{39.84} & \textbf{85.15} & \textbf{92.19} & \textbf{98.83} & \textbf{16.01} & \textbf{20.70} & \textbf{25.00} & 26.56 \\
        \bottomrule
    \end{tabular}
\end{table}

\subsection{Modeling logical structures}\label{sec:sudoku}
We evaluate the capability of $\flowmap$ in modeling logical structures using Sudoku puzzles on a $9\times9$ grid.
We consider \emph{generating} a full puzzle from scratch, and also \emph{solving} a given puzzle containing partial clues, each translating to unconditional and conditional generation, respectively.
These tasks are challenging, especially in the few-step regime, due to their combinatorial constraints and the need to capture long-range dependencies.
Using the same architecture as in our language experiments, we train on a dataset of 1M randomly generated Sudoku grids, each represented as a text of length 81 over a vocabulary of size $|V|=10$.
For conditional generation, we reveal 20 random cells as clues and train to generate the remaining 61 cells.
For both conditional and unconditional generation, we train $\flowvel$ for 100k steps and then distill it into $\flowmap$ for 80k steps.
We use the same training and distillation steps for MDLM+SDTT and Duo+DCD.
We show Sudoku samples generated by $\flowmap$ from \Cref{fig:sudoku_unconditional_valid} to \Cref{fig:sudoku_conditional_valid}.

For unconditional generation, we evaluate performance in terms of validity, uniqueness (the number of unique valid generations among all samples), and novelty (the number of valid generations not present in the training set).
These metrics are computed over 1024 Sudoku grids generated by each model.
The results are in \cref{tab:sudoku_results_full,tab:sudoku_results_uniqueness,tab:sudoku_results_novelty}.
$\flowmap$ achieves near-perfect validity with 1024-step generation, contrasting against previous findings \cite{drozdova2026can} that suggest stochastic samplers are necessary.
Furthermore, it achieves 9\% validity even with one-step generation, whereas discrete diffusion baselines collapse to near-zero validity.
This implies $\flowmap$ can produce generations that satisfy complex logical constraints reasonably well using a single forward pass.
With rejection sampling, this provides approximately $7.6\times$ speedup compared to 81-step autoregressive generation.
Finally, all valid generated samples are distinct from one another and do not overlap with training set, indicating that $\flowmap$ learns the underlying logical structure rather than memorizing specific examples.

For conditional generation, we evaluate in terms of the validity of the solved puzzles.
To assess generalization beyond the training set, we evaluate on a held-out set of 1024 unseen puzzles with 20 clues, and furthermore, on a random subset of 1024 puzzles from Sudoku-Extreme-Full~\cite{wang2025hierarchical, sapientinc_sudoku_extreme}, a curated set of puzzles differing in clue count and difficulty from our training set.
The results are in \cref{tab:sudoku_results_full}.
$\flowmap$ accurately solves a given puzzle in a single forward pass with near 40\% accuracy on the held-out set, while discrete diffusion baselines catastrophically fail.
This advantage transfers to 16\% accuracy on puzzles from Sudoku-Extreme-Full with differently distributed clues as well as difficulties, showing that $\flowmap$ learns problem-solving strategies that generalize beyond the training distribution.
This level of performance is remarkable given the difficulty of capturing combinatorial constraints and long-range dependencies within a single forward pass \cite{yang2024large}.

%% file: conc.tex
\section{Conclusion}

We show that language models based on continuous flows over one-hot token embeddings can outperform discrete diffusion in both quality and speed.
Our flow language model ($\flowvel$) matches state-of-the-art discrete diffusion in the many-step regime, while our flow map language model ($\flowmap$) substantially outperforms distilled discrete methods in the few-step regime, including one-step generation.
Central to our approach is the two-time denoiser, a novel reparameterization that places the flow map on the simplex and enables cross-entropy training.
We further demonstrate that the continuous formulation enables inference-time guidance via autoguidance and reward-guided generation, where the flow map provides a differentiable look-ahead that is unavailable to discrete methods.

More broadly, our results open the door to leveraging the extensive toolkit developed for continuous generative models, including guidance, editing, and inversion, for language generation, and motivate scaling flow-based approaches to larger models and datasets.
In addition to their inference-time benefits, we believe that $\flowmap$'s offer compelling advantages for reinforcement learning-based finetuning, where they stand to dramatically reduce the computational and memory complexity of rollouts needed to compute the terminal reward.

Despite its advantages, our method does have some limitations.
In particular, the one-hot representation requires evaluating and backpropagating through the full $|V| \times d$ embedding matrix at each training step, incurring around 30\% higher time and memory costs compared to embedding diffusion methods that update only the relevant embedding vectors.
Future work may be able to address this using sparse gradient techniques or structured representations.

\section*{Acknowledgments}
This work was supported in part by the National Research Foundation of Korea (RS-2024-00351212 and RS-2024-00436165) and the Institute of Information \& Communications Technology Planning \& Evaluation (IITP) (RS-2024-00509279, RS-2022-II220926, RS-2022-II220959, and RS-2019-II190075) funded by the Korean government (MSIT). Computational resources were provided in part by the “HPC support” project funded by MSIT and NIPA.
The authors would like to thank Eric Vanden-Eijnden, Rajesh Ranganath, Kyunghyun Cho, and Sander Dieleman for valuable discussions, and Patrick Pynadath, the author of CANDI, for sharing model checkpoints.

\section*{Impact Statement}
While increasing the accessibility and efficiency of language models shares broader social implications of widely used large language models, such as potential for misuse, we believe that there are no specific ethical issues that newly emerge in our approach that require further clarification.

%% file: app.tex

\input{relatedwork}

\section{Background on flow maps}\label{sec:app:flow_map_background}

In this section, we provide a self-contained overview of flow maps, which serve as the theoretical foundation for our few-step language model $\flowmap$.
Throughout the appendix, we write $d = L \times |V|$ and use $\R^d$ in place of $\R^{L \times |V|}$ for brevity; all results specialize to the one-hot setting of the main text.

\begin{definition}[Flow map]\label{def:app:flow_map}
	The flow map $X_{s,t}: \R^d \to \R^d$ for the probability flow~\cref{eq:flow_velocity_ode} is the unique map satisfying the jump condition
	\begin{equation}
		\begin{aligned}
			X_{s,t}({\bf x}_s) = {\bf x}_t \quad \text{for all} \quad (s,t) \in [0,1]^2,
		\end{aligned}
		\label{eq:app:jump_condition}
	\end{equation}
	where $({\bf x}_t)_{t \in [0,1]}$ is any trajectory of the probability flow.
\end{definition}
The flow map can be viewed as the solution operator of the probability flow equation, taking ``steps'' of arbitrary size $t-s$ along trajectories.
In the following, we characterize it mathematically to derive algorithms for distillation and direct training.

\begin{proposition}[Flow map characterizations]\label{prop:app:flow_map_char}
	The flow map satisfies the following conditions:
	\begin{enumerate}[label=(\roman*)]
		\item The flow map is the unique solution to the Lagrangian equation: for all ${\bf x} \in \R^d$ and $(s,t) \in [0,1]^2$,
		      \begin{align}
			      \label{eq:app:lagrangian}
			      \partial_t X_{s,t}({\bf x}) = b_t(X_{s,t}({\bf x})), \quad X_{s,s}({\bf x}) = {\bf x}.
		      \end{align}
		\item The flow map is the unique solution to the Eulerian equation: for all ${\bf x} \in \R^d$ and $(s,t) \in [0,1]^2$,
		      \begin{align}
			      \label{eq:app:eulerian}
			      \partial_s X_{s,t}({\bf x}) + b_s({\bf x}) \cdot \nabla X_{s,t}({\bf x}) = 0, \quad X_{t,t}({\bf x}) = {\bf x}.
		      \end{align}
		\item The flow map satisfies the semigroup condition: for all ${\bf x} \in \R^d$ and $(s,t,u) \in [0,1]^3$,
		      \begin{align}
			      \label{eq:app:semigroup}
			      X_{s,u}({\bf x}) = X_{t,u}(X_{s,t}({\bf x})).
		      \end{align}
	\end{enumerate}
\end{proposition}
For proofs, see~\citet{boffi2025flowmapmatchingstochastic}.

For each ${\bf x} \in \R^d$, the Lagrangian equation is an ODE in $t$ with parameter $s$, describing forward evolution along trajectories.
It was introduced in~\citet{boffi2025flowmapmatchingstochastic,boffi2025build} and is the basis for Lagrangian self-distillation and terminal velocity matching~\citep{zhou2025terminal}.
The Eulerian equation is a PDE in $s$ describing how the map changes as the starting time varies, and is the basis for consistency models~\citep{song2023consistency} and MeanFlow~\citep{geng2025mean}.
The semigroup condition states that two successive jumps can be replaced by a single direct jump, and is the basis for progressive distillation~\citep{salimans2022progressive} and shortcut models~\citep{frans2024one}.

The following result demonstrates that the flow map contains a flow implicitly, which we use to derive direct training algorithms as well as to provide an anchor on the diagonal for distillation.
\begin{corollary}[Tangent condition]\label{cor:app:tangent}
	The flow map encodes the velocity field $b_t$ on its diagonal:
	\begin{align}
		\label{eq:app:tangent}
		\lim_{s \to t} \partial_t X_{s,t}({\bf x}) = b_t({\bf x}).
	\end{align}
\end{corollary}
The proof follows by a direct application of the Lagrangian equation~\cref{eq:app:lagrangian}.
The condition~\cref{eq:app:tangent} motivates the parameterization
\begin{align}
	\label{eq:app:flow_map_param}
	X_{s,t}({\bf x}) = {\bf x} + (t-s)v_{s,t}({\bf x}),
\end{align}
where $v:[0, 1]^2\times\R^d\to\R^d$ is a learned function satisfying $v_{t,t}({\bf x}) = b_t({\bf x})$, which follows from the tangent condition~\cref{eq:app:tangent}~\citep{boffi2025build}.
Geometrically, $v_{s,t}$ represents the average velocity along the trajectory from ${\bf x}_s$ to ${\bf x}_t$.
The tangent condition demonstrates that the flow is encoded on the diagonal $s=t$, while the off-diagonal $s \neq t$ corresponds to the flow map.
In the next subsection, we show how this can be learned in two-phases via distillation techniques or simultaneously with the flow via a single self-distillation approach.

\paragraph{Sampling.} In the context of flow-based generative models, the flow map enables efficient one-step sampling: given ${\bf x}_0 \sim p_0$, a single application ${\bf x}_1 = X_{0,1}({\bf x}_0)$ produces a sample from $p_1$, avoiding iterative numerical integration. For additional refinement, one can compose maps over a grid $0 = t_0 < t_1 < \cdots < t_N = 1$ via ${\bf x}_{t_{n+1}} = X_{t_n, t_{n+1}}({\bf x}_{t_n})$, trading compute for quality.

\subsection{Direct training and distillation of flow maps}\label{sec:app:direct_vs_distill}

Flow maps can be learned either by \textit{distillation} from a pre-trained velocity model, or by \textit{direct training} (self-distillation) without a pre-trained teacher. We summarize both approaches below.

\paragraph{Distillation from a pre-trained velocity.}

Given a pre-trained velocity field $\hat{b}_t$, we can distill it into a flow map $\hat{X}_{s,t}$ by minimizing objectives derived from the characterizations in \cref{prop:app:flow_map_char}.

\begin{proposition}[Flow map distillation]\label{prop:app:map_distill}
	Given a pre-trained velocity $\hat{b}_t$, the flow map is the unique critical point of the following losses:
	\begin{enumerate}[label=(\roman*)]
		\item The Lagrangian map distillation (LMD) loss:
		      \begin{equation}
			      \calL_{\lmd}(\hat{v}) = \int_0^1 \int_0^t \E|\partial_t \hat{X}_{s,t}(I_s) - \sg{\hat{b}_t(\hat{X}_{s,t}(I_s))}|^2 {\rm d}s\,{\rm d}t + \int_0^1 \E|\hat{v}_{t,t}(I_t) - \hat{b}_t(I_t)|^2 {\rm d}t.
			      \label{eq:app:lmd}
		      \end{equation}
		\item The Eulerian map distillation (EMD) loss:
		      \begin{equation}
			      \calL_{\emd}(\hat{v}) = \int_0^1 \int_0^t \E|\partial_s \hat{X}_{s,t}(I_s) + \sg{\hat{b}_s(I_s) \cdot \nabla \hat{X}_{s,t}(I_s)}|^2 {\rm d}s\,{\rm d}t + \int_0^1 \E|\hat{v}_{t,t}(I_t) - \hat{b}_t(I_t)|^2 {\rm d}t.
			      \label{eq:app:emd}
		      \end{equation}
		\item The progressive map distillation (PMD) loss:
		      \begin{equation}
			      \calL_{\pmd}(\hat{v}) = \int_0^1 \int_0^t \int_s^t \E|\hat{X}_{s,t}(I_s) - \sg{\hat{X}_{t,u}(\hat{X}_{s,t}(I_s))}|^2 {\rm d}u\,{\rm d}s\,{\rm d}t + \int_0^1 \E|\hat{v}_{t,t}(I_t) - \hat{b}_t(I_t)|^2 {\rm d}t.
			      \label{eq:app:pmd}
		      \end{equation}
	\end{enumerate}
\end{proposition}
For proofs, see~\citet{boffi2025flowmapmatchingstochastic}.

These objectives enable converting a pre-trained velocity field $\hat{b}_t$ into a flow map $\hat{X}_{s, t}$.
Distillation is typically faster and requires less compute than self-distillation, making it particularly useful when large-scale pre-trained models are available.
Nevertheless, it is also useful to train flow maps from scratch, as we describe next.

\paragraph{Direct training via self-distillation.}

One of the core difficulties in developing direct training algorithms for flow maps is the lack of an obvious target for learning, and hence it is unclear \textit{a-priori} how to design an appropriate objective function.
To obtain a target, one key insight is the tangent condition~\cref{eq:app:tangent}, which shows that the diagonal $\hat{v}_{t, t}$ can be trained systematically via flow matching.
Combining this observation with the distillation objectives above leads to the following single-phase training approach.

\begin{proposition}[Flow map self-distillation]\label{prop:app:self_distill}
	The flow map is the unique critical point of the following losses:
	\begin{enumerate}[label=(\roman*)]
		\item The Lagrangian self-distillation (LSD) loss:
		      \begin{equation}
			      \calL_{\lsd}(\hat{v}) = \int_0^1 \int_0^t \E|\partial_t \hat{X}_{s,t}(I_s) - \sg{\hat{v}_{t,t}(\hat{X}_{s,t}(I_s))}|^2 {\rm d}s\,{\rm d}t + \int_0^1 \E|\hat{v}_{t,t}(I_t) - \dot{I}_t|^2{\rm d}t.
			      \label{eq:app:lsd}
		      \end{equation}
		\item The Eulerian self-distillation (ESD) loss:
		      \begin{equation}
			      \calL_{\esd}(\hat{v}) = \int_0^1 \int_0^t \E|\partial_s \hat{X}_{s,t}(I_s) + \sg{\nabla \hat{X}_{s,t}(I_s) \hat{v}_{s,s}(I_s)}|^2 {\rm d}s\,{\rm d}t + \int_0^1 \E|\hat{v}_{t,t}(I_t) - \dot{I}_t|^2{\rm d}t.
			      \label{eq:app:esd}
		      \end{equation}
		\item The progressive self-distillation (PSD) loss:
		      \begin{equation}
			      \calL_{\psd}(\hat{v}) = \int_0^1 \int_0^t \int_s^t \E|\hat{X}_{s,t}(I_s) - \sg{\hat{X}_{u,t}(\hat{X}_{s,u}(I_s))}|^2 {\rm d}u\,{\rm d}s\,{\rm d}t + \int_0^1 \E|\hat{v}_{t,t}(I_t) - \dot{I}_t|^2{\rm d}t.
			      \label{eq:app:psd}
		      \end{equation}
	\end{enumerate}
\end{proposition}
For proofs, we refer the reader to~\citet{boffi2025build}.
LSD has recently been scaled under the name Terminal Velocity Matching~\citep{zhou2025terminal}, ESD is equivalent to Improved MeanFlow~\citep{geng2025improved}, and PSD can be viewed as a continuous-time limit of shortcut models~\citep{frans2024one}.

\section{Theoretical details}
\label{sec:app:denoiser_flow_maps}

\subsection{Proofs from the main text}\label{sec:app:proofs_main_text}

\subsubsection{Proof of \texorpdfstring{\cref{lemma:denoiser_posterior_equivalence}}{}}\label{sec:posterior_proof}

\denoiserposterior*

\begin{proof}
	From \cref{eq:x_prediction_target}, we have that for the $l$-th token position:
	\begin{equation}
		\begin{aligned}
			D_t({\bf x})^l=\E[{\bf x}_1^l|I_t=\bf x].
		\end{aligned}
		\label{eq:app:denoiser_token}
	\end{equation}
	Let ${\bf e}_i \in \R^{|V|}$ denote the one-hot encoding of the $i$-th subword in the vocabulary $V$. Since ${\bf x}_1^l$ is a one-hot vector, it takes values in $\{{\bf e}_1, \dots {\bf e}_{|V|}\}$. Then the conditional expectation above can be expanded as follows:
	\begin{equation}
		\begin{aligned}
			D_t({\bf x})^l=\E[{\bf x}_1^l|I_t={\bf x}]=\sum^{|V|}_{i=1}{\bf e}_i\cdot p_{1|t}^l({\bf x}_1^l={\bf e}_i|I_t={\bf x}) =
			\left[\begin{array}{c} p_{1|t}^l(\mathbf{x}_1^l = \mathbf{e}_1 | I_t = \mathbf{x}) \\ ... \\ p_{1|t}^l(\mathbf{x}_1^l = \mathbf{e}_{|V|} | I_t = \mathbf{x}) \end{array}\right].
		\end{aligned}
		\label{eq:app:denoiser_posterior_expansion}
	\end{equation}
	This is precisely the vector of posterior probabilities over the vocabulary, $p_{1|t}^l(\cdot | I_t = \mathbf{x})$.
\end{proof}

\subsubsection{Proof of \texorpdfstring{\cref{prop:classification}}{}}\label{sec:classification_proof}

\cedenoiser*

We prove the two claims separately, as the Wasserstein bound requires further setup.
We first prove that the minimizer is correct and unique.

\begin{proof}[Proof of the minimizer.]
	We first observe that for any two distributions $p, q$, the cross-entropy decomposes as
	\begin{equation}
		\E_p[-\log q] = H(p) + \kl{p}{q}.
		\label{eq:app:ce_identity}
	\end{equation}
	Writing $\hat{p}_{1|t}^l(\cdot \mid I_t) \coloneqq \hat{D}_t(I_t)^l$ for the model's predicted token distribution at position $l$, the tower property gives
	\begin{equation}
		\E\!\left[-\log\hat{p}_{1|t}^l({\bf x}_1^l\mid I_t)\right]
		=
		\E_{I_t}\!\left[
			\E_{{\bf x}_1^l\mid I_t}\!\left[-\log\hat{p}_{1|t}^l({\bf x}_1^l\mid I_t)\right]
			\right].
		\label{eq:app:ce2kl_step1}
	\end{equation}
	Applying~\cref{eq:app:ce_identity} to the inner expectation with $p = p_{1|t}^l(\cdot \mid I_t)$ and $q = \hat{p}_{1|t}^l(\cdot \mid I_t)$:
	\begin{equation}
		\E_{{\bf x}_1^l\mid I_t}\!\left[-\log\hat{p}_{1|t}^l({\bf x}_1^l\mid I_t)\right]
		=
		H\!\left(p_{1|t}^l(\cdot\mid I_t)\right)
		+
		\kl{p_{1|t}^l(\cdot\mid I_t)}{\hat{p}_{1|t}^l(\cdot\mid I_t)}.
		\label{eq:app:ce_decomposition}
	\end{equation}
	Summing over token positions and integrating in time yields the decomposition
	\begin{equation}
		\calL_{\mathsf{CE}}(\hat{D})
		=
		\underbrace{\int_0^{1}
			\E\!\left[H({\bf x}_1\mid I_t)\right]{\rm d}t}_{\text{irreducible conditional entropy}}
		+
		\int_0^{1}
		\E\!\left[
			\sum_{l=1}^L
			\kl{p_{1|t}^l(\cdot\mid I_t)}{\hat{p}_{1|t}^l(\cdot\mid I_t)}
			\right]{\rm d}t,
		\label{eq:app:ce2kl_step3}
	\end{equation}
	where $H({\bf x}_1 \mid I_t) \coloneqq \sum_{l=1}^L H(p_{1|t}^l(\cdot \mid I_t))$.
	Since $\kl{p}{q} \geq 0$ with equality if and only if $p = q$, the unique minimizer is $\hat{D}_t(I_t)^l = p_{1|t}^l(\cdot \mid I_t) = D_t(I_t)^l$.
\end{proof}

We now turn to prove the Wasserstein bound~\cref{eq:excess_risk_bound}, which requires the following preliminaries.

\paragraph{Notation.}
Given the learned denoiser $\hat{D}$, we have the induced velocity field $\hat{b}_t({\bf x}) \coloneqq (\hat{D}_t({\bf x}) - {\bf x})/(1-t)$.
Let $X^{\hat{b}}_{s,t}$ denote the exact flow map generated by $\hat{b}$:
\begin{equation}
	\partial_t X^{\hat{b}}_{s,t}({\bf x}) = \hat{b}_t(X^{\hat{b}}_{s,t}({\bf x})),
	\qquad
	X^{\hat{b}}_{s,s}({\bf x})={\bf x}.
	\label{eq:app:teacher_flow_hatb}
\end{equation}
For ${\bf x}_0 \sim p_0$, let $p_t \coloneqq (X_{0,t})_\#p_0$ denote the true marginal and $p_t^{\hat{b}} \coloneqq (X^{\hat{b}}_{0,t})_\#p_0$ the distribution generated by the learned velocity.

By the decomposition~\cref{eq:app:ce2kl_step3}, the entropy term depends only on the data distribution and cannot be reduced by any model.
We define the \emph{diagonal excess risk} as the learnable KL gap:
\begin{equation}
	\Delta_D(\hat{D})
	\coloneqq
	\calL_{\mathsf{CE}}(\hat{D})
	-
	\int_0^{1}\E\!\left[H({\bf x}_1\mid I_t)\right]{\rm d}t
	=
	\int_0^{1}
	\E\!\left[
		\sum_{l=1}^L
		\kl{p_{1|t}^l(\cdot\mid I_t)}{\hat{p}_{1|t}^l(\cdot\mid I_t)}
		\right]{\rm d}t.
	\label{eq:app:diag_excess_def}
\end{equation}

We will show that the excess risk~\cref{eq:app:diag_excess_def} controls the quality of the ODE sampler driven by $\hat{b}$.
Because the excess risk is equal to the cross entropy objective up to a constant, learning a denoiser via cross entropy minimization generates a distribution close to $p_1$ in Wasserstein distance.
The following three standard regularity assumptions are used throughout.

\begin{assumption}[Lipschitz velocities]\label{asm:app:lipschitz}
	There exist constants $L_b, L_{\hat{b}} > 0$ such that
	\begin{equation}
		\begin{aligned}
			|b_t({\bf x}) - b_t({\bf y})|             & \le L_b\,|{\bf x} - {\bf y}|,         & \qquad
			|\hat{b}_t({\bf x}) - \hat{b}_t({\bf y})| & \le L_{\hat{b}}\,|{\bf x} - {\bf y}|,
		\end{aligned}
		\label{eq:app:lipschitz_velocities}
	\end{equation}
	for all $t \in [0, 1)$ and all ${\bf x}, {\bf y}$.
\end{assumption}

\begin{assumption}[Finite second moments]\label{asm:app:second_moment}
	The velocity fields have uniformly bounded second moments under their respective flow distributions:
	\begin{equation}
		M_b \coloneqq \sup_{t \in [0,1)} \E_{p_t}\!\left[|b_t|^2\right] < \infty, \qquad
		M_{\hat{b}} \coloneqq \sup_{t \in [0,1)} \E_{p_t^{\hat{b}}}\!\left[|\hat{b}_t|^2\right] < \infty.
		\label{eq:app:second_moment_velocities}
	\end{equation}
\end{assumption}

We also assume a simplex-interior property, which is guaranteed in practice by using a softmax output layer.

\begin{assumption}[Simplex-interior outputs]\label{asm:app:interior}
	The learned denoiser $\hat{D}_t$ produces outputs in the interior of the probability simplex $\Delta^{|V|-1}$ at each token position.
\end{assumption}

To connect the excess risk to the $L^2$ denoiser error, we exploit the simplex structure of the denoiser outputs via Pinsker's inequality.

\begin{lemma}
	\label{lem:app:kl_to_l2}
	The $L^2$ denoiser error is controlled by the excess risk:
	\begin{equation}
		\int_0^{1}\E\!\left[|\hat{D}_t(I_t)-D_t(I_t)|^2\right]{\rm d}t
		\le
		2\,\Delta_D(\hat{D}).
		\label{eq:app:dt_l2_from_excess}
	\end{equation}
\end{lemma}
\begin{proof}
	For any $p, q \in \Delta^{|V|-1}$, the total variation distance is $\mathrm{TV}(p,q) = \frac{1}{2}|p - q|_1$, and Pinsker's inequality states
	\begin{equation}
		\mathrm{TV}(p,q) \le \sqrt{\tfrac{1}{2}\kl{p}{q}}.
		\label{eq:app:pinsker}
	\end{equation}
	Combining these and using $|p - q|^2 \le |p - q|_1^2$,
	\begin{equation}
		|p - q|^2
		\le
		|p - q|_1^2
		=
		4\,\mathrm{TV}(p,q)^2
		\le
		2\,\kl{p}{q}.
		\label{eq:app:pinsker_l2}
	\end{equation}
	Applying this bound tokenwise with $p = p_{1|t}^l(\cdot \mid I_t)$ and $q = \hat{p}_{1|t}^l(\cdot \mid I_t)$, summing over $l$, and integrating in $t$ gives~\cref{eq:app:dt_l2_from_excess} via the decomposition~\cref{eq:app:ce2kl_step3}.
\end{proof}

With these tools in hand, we bound the velocity error in terms of the cross entropy excess risk.
The denoiser-velocity conversion~\cref{eq:x_prediction_target} introduces a $(1-t)^{-1}$ singularity, so we truncate the sampler at $1 - \xi$ for $\xi \in (0, 1)$.

\begin{lemma}[Velocity error from excess risk]\label{lem:app:velocity_from_ce}
	Under \cref{asm:app:interior}, for any $\xi \in (0, 1)$,
	\begin{equation}
		\int_0^{1-\xi}\E\!\left[|\hat{b}_t(I_t)-b_t(I_t)|^2\right]{\rm d}t
		\le
		C_b(\xi)\,\Delta_D(\hat{D}),
		\qquad
		C_b(\xi) \coloneqq 2\xi^{-2}.
		\label{eq:app:velocity_error_from_ce}
	\end{equation}
\end{lemma}
\begin{proof}
	The denoiser-velocity relation amplifies errors by $(1-t)^{-1}$:
	\begin{equation}
		\hat{b}_t({\bf x})-b_t({\bf x})
		=
		\frac{\hat{D}_t({\bf x})-D_t({\bf x})}{1-t}.
		\label{eq:app:bdiff_from_ddiff}
	\end{equation}
	For $t \leq 1 - \xi$, $(1-t)^{-2} \leq \xi^{-2}$, so
	\begin{equation}
		|\hat{b}_t(I_t)-b_t(I_t)|^2
		\le
		\xi^{-2}|\hat{D}_t(I_t)-D_t(I_t)|^2.
		\label{eq:app:bdiff_bound_eps}
	\end{equation}
	Integrating over $[0, 1-\xi]$ and bounding by the full integral gives
	\begin{equation}
		\int_0^{1-\xi}\E|\hat{b}_t(I_t) - b_t(I_t)|^2\,{\rm d}t
		\leq
		\xi^{-2}\int_0^{1}\E|\hat{D}_t(I_t) - D_t(I_t)|^2\,{\rm d}t.
		\label{eq:app:vel_from_denoiser_integral}
	\end{equation}
	Applying~\cref{eq:app:dt_l2_from_excess} gives~\cref{eq:app:velocity_error_from_ce}.
\end{proof}

We now convert the velocity error into a distributional bound via Gronwall's inequality, adapting the approach of~\citet{boffi2025flowmapmatchingstochastic} to the denoiser reparameterization.

\begin{proposition}[Flow error]\label{prop:app:flow_error_from_ce}
	Under \cref{asm:app:lipschitz,asm:app:second_moment,asm:app:interior}, for any $\xi \in (0, 1)$,
	\begin{equation}
		W_2\!\left(p_1^{\hat{b}},p_1\right)
		\le
		C_D(\xi)\sqrt{\Delta_D(\hat{D})}
		+
		r_\xi,
		\label{eq:app:w2_denoiser_bound}
	\end{equation}
	where $C_D(\xi) \coloneqq \sqrt{2(1-\xi)}\,e^{L_{\hat{b}}(1-\xi)}/\xi$ and the terminal remainder satisfies
	\begin{equation}
		r_\xi
		\coloneqq
		W_2(p_1,p_{1-\xi})
		+
		W_2(p_1^{\hat{b}},p_{1-\xi}^{\hat{b}})
		\le
		\xi\!\left(\sqrt{M_b} + \sqrt{M_{\hat{b}}}\right),
		\label{eq:app:terminal_remainder}
	\end{equation}
	with $M_b$ and $M_{\hat{b}}$ as defined in~\cref{asm:app:second_moment}.
\end{proposition}
\begin{proof}
	\noindent\textit{Coupling and error dynamics.}
	We couple both flows by initializing from the same sample ${\bf x}_0 \sim p_0$:
	\begin{equation}
		X_t \coloneqq X_{0,t}({\bf x}_0),
		\qquad
		\hat{X}_t \coloneqq X_{0,t}^{\hat{b}}({\bf x}_0),
		\qquad
		e_t \coloneqq \hat{X}_t-X_t.
		\label{eq:app:error_process_def}
	\end{equation}
	The error dynamics decompose into a Lipschitz term and a forcing term:
	\begin{equation}
		\dot e_t
		=
		\hat{b}_t(\hat{X}_t)-b_t(X_t)
		=
		\bigl[\hat{b}_t(\hat{X}_t)-\hat{b}_t(X_t)\bigr]
		+
		\bigl[\hat{b}_t(X_t)-b_t(X_t)\bigr].
		\label{eq:app:error_dynamics}
	\end{equation}

	\noindent\textit{Gronwall bound on $[0, 1-\xi]$.}
	Using Lipschitzness of $\hat{b}_t$ with constant $L_{\hat{b}}$:
	\begin{equation}
		\frac{{\rm d}}{{\rm d}t}|e_t|
		\le
		L_{\hat b}|e_t| + |\hat{b}_t(X_t)-b_t(X_t)|.
		\label{eq:app:error_diffeq_ineq}
	\end{equation}
	By Gronwall's inequality, for $T = 1 - \xi$:
	\begin{equation}
		|e_T|
		\le
		e^{L_{\hat b}T}\int_0^T |\hat{b}_t(X_t)-b_t(X_t)|\,{\rm d}t.
		\label{eq:app:error_gronwall}
	\end{equation}
	Applying the Cauchy--Schwarz inequality to the integral and squaring:
	\begin{equation}
		|e_T|^2
		\le
		e^{2L_{\hat b}T}T
		\int_0^T |\hat{b}_t(X_t)-b_t(X_t)|^2\,{\rm d}t.
		\label{eq:app:error_square}
	\end{equation}
	Taking expectations and using the coupling $W_2^2(p_T^{\hat{b}},p_T)\le \E|e_T|^2$:
	\begin{equation}
		W_2^2(p_{1-\xi}^{\hat{b}},p_{1-\xi})
		\le
		e^{2L_{\hat b}(1-\xi)}(1-\xi)
		\int_0^{1-\xi} \E|\hat{b}_t(I_t)-b_t(I_t)|^2\,{\rm d}t.
		\label{eq:app:w2_preterminal}
	\end{equation}
	Applying~\cref{lem:app:velocity_from_ce}:
	\begin{equation}
		W_2(p_{1-\xi}^{\hat{b}},p_{1-\xi})
		\le
		C_D(\xi)\sqrt{\Delta_D(\hat{D})},
		\qquad
		C_D(\xi)
		=
		\frac{\sqrt{2(1-\xi)}\,e^{L_{\hat{b}}(1-\xi)}}{\xi}.
		\label{eq:app:w2_preterminal_final}
	\end{equation}

	\noindent\textit{Terminal remainder.}
	The triangle inequality on $(\mathcal{P}_2, W_2)$ gives
	\begin{equation}
		W_2(p_1^{\hat{b}},p_1)
		\le
		W_2(p_1^{\hat{b}},p_{1-\xi}^{\hat{b}})
		+
		W_2(p_{1-\xi}^{\hat{b}},p_{1-\xi})
		+
		W_2(p_{1-\xi},p_1).
		\label{eq:app:w2_triangle_terminal}
	\end{equation}
	Each remainder term is controlled by flowing over the short interval $[1-\xi, 1]$.
	For the true flow, we couple $X_{1-\xi} \sim p_{1-\xi}$ and $X_1 \sim p_1$ along the same trajectory, so that $X_1 - X_{1-\xi} = \int_{1-\xi}^1 b_t(X_t)\,{\rm d}t$.
	The coupling gives
	\begin{equation}
		W_2^2(p_{1-\xi}, p_1)
		\le
		\E\!\left[\left|\int_{1-\xi}^1 b_t(X_t)\,{\rm d}t\right|^2\right].
		\label{eq:app:remainder_coupling}
	\end{equation}
	By the Cauchy--Schwarz inequality, $|\int f\,{\rm d}t|^2 \le (\int 1^2\,{\rm d}t)(\int |f|^2\,{\rm d}t)$, so
	\begin{equation}
		\left|\int_{1-\xi}^1 b_t(X_t)\,{\rm d}t\right|^2
		\le
		\xi \int_{1-\xi}^1 |b_t(X_t)|^2\,{\rm d}t.
		\label{eq:app:remainder_cs}
	\end{equation}
	Taking expectations and bounding the integrand by $M_b$ via~\cref{asm:app:second_moment}:
	\begin{equation}
		W_2^2(p_{1-\xi}, p_1)
		\le
		\xi \int_{1-\xi}^1 \E\!\left[|b_t(X_t)|^2\right]{\rm d}t
		\le
		\xi^2 M_b.
		\label{eq:app:remainder_true}
	\end{equation}
	The same argument applied to the learned flow gives $W_2(p_{1-\xi}^{\hat{b}}, p_1^{\hat{b}}) \le \xi\sqrt{M_{\hat{b}}}$ via~\cref{asm:app:second_moment}.
	Substituting~\cref{eq:app:w2_preterminal_final,eq:app:remainder_true} into~\cref{eq:app:w2_triangle_terminal} yields~\cref{eq:app:w2_denoiser_bound}.
\end{proof}

\subsubsection{Proof of \texorpdfstring{\cref{prop:no_discrete_transport}}{}}\label{sec:monge_map}

\nodiscretetransport*
\begin{proof}
	The pushforward $\nu = f_\# \mu$ satisfies
	\begin{equation}
		\nu({\bf y}) = \sum_{{\bf x} \in f^{-1}({\bf y})} \mu({\bf x}),
		\label{eq:app:pushforward}
	\end{equation}
	where $f^{-1}({\bf y})$ denotes the preimage of $\{{\bf y}\}$.
	Since $f^{-1}({\bf y}) \subseteq S$, each probability $\nu({\bf y})$ is a sum of some subset of the values $\{\mu({\bf x}) \mid {\bf x} \in S\}$.
	In particular, every nonempty preimage contributes at least $\mu_{\min} \coloneqq \min\{\mu({\bf x}) \mid \mu({\bf x}) > 0\}$, so $\nu({\bf y}) \in \{0\} \cup [\mu_{\min}, 1]$ for all ${\bf y}$.
	Constructing $\nu$ with $\nu({\bf y}_0) = \mu_{\min}/2$ for some ${\bf y}_0 \in S$ yields a distribution that no deterministic map can produce.
\end{proof}
Choosing $S = V^L$ and $\mu = p_0$ for discrete diffusion, this shows that for any noise distribution $p_0$, there exists a data distribution $p$ that cannot be reached via one-step deterministic transport.
By contrast, continuous flows admit a flow map at the sample level (\cref{sec:flow_map}), which is the basis for our $\flowmap$ approach.

\subsection{Flow maps on the simplex}

In~\cref{sec:flow_velocity}, we introduced the denoiser $D_t$ in~\cref{eq:x_prediction_target}.
In the discrete context considered here, this approach reparameterizes the instantaneous velocity $b_t$ into a simplex-valued clean-data predictor, enabling training via cross-entropy~\cref{eq:flow_x_classification_loss}.
We now develop an analogous reparameterization for the flow map.
To do so, we leverage the two-time denoiser $\delta_{s,t}$~\cref{eq:delta_denoiser}, which converts the mean flow $v_{s,t}$ into a clean-data predictor that lies on the simplex.
This extends the single-time denoiser-velocity relation to the two-time setting, and will make it possible for us to leverage training objectives based on cross entropy.

\paragraph{General setup.}
In this section, we consider the general stochastic interpolant, going beyond the standard flow matching setting considered in the main text.
To this end, we consider
\begin{equation}
	\begin{aligned}
		I_t = \alpha_t {\bf x}_0 + \beta_t {\bf x}_1,
	\end{aligned}
	\label{eq:app:interpolant}
\end{equation}
where $\alpha, \beta: [0,1] \to [0,1]$ are continuous functions satisfying the boundary conditions $\alpha_0 = 1$, $\alpha_1 = 0$, $\beta_0 = 0$, $\beta_1 = 1$.

\begin{definition}[Endpoint denoiser]\label{def:app:denoiser}
	The endpoint denoiser $D_t: \R^d \to \R^d$ is defined as:
	\begin{align}
		\label{eq:app:denoiser}
		D_t({\bf x}) := \E[{\bf x}_1 | I_t = {\bf x}].
	\end{align}
\end{definition}
The endpoint denoiser is the posterior mean of the clean data given the current noisy point.
We emphasize that this differs significantly from the one-step flow map, as the denoiser averages over any multimodality present in the posterior density.
Nevertheless, because it matches the geometry of the clean data ${\bf x}_1$, it is useful to learn the endpoint denoiser as we do in the main text.
As we now show, it can be directly related to the flow.

\begin{lemma}[Denoiser-velocity relation]\label{lem:app:denoiser_velocity}
	For general interpolant coefficients $\alpha_t$, $\beta_t$, the velocity field and endpoint denoiser are related by:
	\begin{align}
		\label{eq:app:denoiser_velocity_general}
		b_t({\bf x}) = \frac{\dot{\beta}_t}{\beta_t} D_t({\bf x}) + \left(\dot{\alpha}_t - \frac{\alpha_t \dot{\beta}_t}{\beta_t}\right) \frac{{\bf x} - \beta_t D_t({\bf x})}{\alpha_t}.
	\end{align}
\end{lemma}
\begin{proof}
	Conditioning on $I_t = {\bf x}$ gives:
	\begin{equation}
		{\bf x} = \alpha_t \E[{\bf x}_0 | I_t = {\bf x}] + \beta_t D_t({\bf x}).
	\end{equation}
	The velocity field is:
	\begin{equation}
		b_t({\bf x}) = \E[\dot{I}_t | I_t = {\bf x}] = \dot{\alpha}_t \E[{\bf x}_0 | I_t = {\bf x}] + \dot{\beta}_t D_t({\bf x}).
	\end{equation}
	Solving for $\E[{\bf x}_0 | I_t = {\bf x}] = ({\bf x} - \beta_t D_t({\bf x}))/\alpha_t$ and substituting yields~\cref{eq:app:denoiser_velocity_general}.
\end{proof}
This relation first appeared in~\citet{albergo2023stochastic}.
For $\alpha_t = 1-t$ and $\beta_t = t$ as in the main text, \cref{eq:app:denoiser_velocity_general} simplifies to:
\begin{align}
	\label{eq:app:denoiser_velocity}
	b_t({\bf x}) = \frac{D_t({\bf x}) - {\bf x}}{1-t}.
\end{align}
Rearranging gives:
\begin{align}
	\label{eq:app:denoiser_from_velocity}
	D_t({\bf x}) = {\bf x} + (1-t)b_t({\bf x}).
\end{align}
The above equation reveals a natural interpretation: the denoiser $D_t({\bf x})$ corresponds to a single Euler step of size $(1-t)$ starting from ${\bf x}$ with the velocity field $b_t$.
This also makes clear its relationship to the flow map, which corresponds to the exact solution of the ODE rather than a single Euler step.

\subsection{The two-time denoiser.}\label{sec:app:two_time_denoiser}
We now restate \cref{prop:delta_conditions} and prove each property.
The proof of the simplex property (ii) proceeds by solving the flow ODE via an integrating factor, which reveals that $\delta_{s,t}$ admits an integral representation as a weighted average of denoisers along the flow trajectory.
Non-negativity and normalization then follow from the corresponding properties of the single-time denoiser (\cref{lemma:denoiser_posterior_equivalence}).

\twotimedenoiser*

\begin{proof}
	We prove each property in turn, recalling the parameterization $X_{s,t}({\bf x}) = {\bf x} + (t-s)v_{s,t}({\bf x})$ from~\cref{eq:app:flow_map_param}.

	\textit{(i) Flow map reparameterization.}
	Substituting $v_{s,t} = (\delta_{s,t} - {\bf x})/(1-s)$ into~\cref{eq:app:flow_map_param} immediately gives:
	\begin{equation}
		X_{s,t}({\bf x}) = {\bf x} + \frac{t-s}{1-s}(\delta_{s,t}({\bf x}) - {\bf x}) = \frac{1-t}{1-s}{\bf x} + \frac{t-s}{1-s}\delta_{s,t}({\bf x}).
		\label{eq:app:flow_map_denoiser}
	\end{equation}

	\textit{(ii) Simplex.}
	We show that $\delta_{s,t}$ can be written as a weighted average of single-time denoisers along the flow trajectory.
	The flow ODE in denoiser form is:
	\begin{equation}
		\partial_\tau X_{s,\tau}({\bf x}) = \frac{D_\tau(X_{s,\tau}({\bf x})) - X_{s,\tau}({\bf x})}{1-\tau}.
	\end{equation}
	Rearranging and multiplying by the integrating factor $1/(1-\tau)$:
	\begin{equation}
		\frac{\partial}{\partial \tau}\!\left(\frac{X_{s,\tau}({\bf x})}{1-\tau}\right) = \frac{D_\tau(X_{s,\tau}({\bf x}))}{(1-\tau)^2}.
	\end{equation}
	Integrating from $s$ to $t$ and using $X_{s,s}({\bf x}) = {\bf x}$:
	\begin{equation}
		\label{eq:app:flow_integral}
		X_{s,t}({\bf x}) = \frac{1-t}{1-s}\,{\bf x} + (1-t)\int_s^t \frac{D_\tau(X_{s,\tau}({\bf x}))}{(1-\tau)^2}\,d\tau.
	\end{equation}
	Comparing with (i) and matching coefficients:
	\begin{equation}
		\label{eq:app:delta_weighted_avg}
		\delta_{s,t}({\bf x}) = \frac{(1-s)(1-t)}{t-s}\int_s^t \frac{D_\tau(X_{s,\tau}({\bf x}))}{(1-\tau)^2}\,d\tau.
	\end{equation}
	By~\cref{lemma:denoiser_posterior_equivalence}, each $D_\tau(X_{s,\tau}({\bf x}))$ is non-negative at every token position and sums to one.
	Since $(1-s)(1-t)/(1-\tau)^2 > 0$, non-negativity of $\delta_{s,t}$ follows immediately.
	To show that $\delta_{s, t}$ always sums to one, we use $\sum_v D_\tau^{l,v} = 1$ and evaluate:
	\begin{equation}
		\begin{aligned}
			\sum_v \delta_{s,t}^{l,v}({\bf x})
			 & = \frac{(1-s)(1-t)}{t-s}\int_s^t \frac{d\tau}{(1-\tau)^2}                 \\
			 & = \frac{(1-s)(1-t)}{t-s}\!\left(\frac{1}{1-t} - \frac{1}{1-s}\right) = 1.
		\end{aligned}
	\end{equation}
	Because $\delta_{s, t}$ is always non-negative and sums to one, it lies on the simplex.

	\textit{(iii) Diagonal.}
	By direct computation using $v_{s,s} = b_s$ and~\cref{eq:app:denoiser_from_velocity}:
	\begin{equation}
		\delta_{s,s}({\bf x}) = {\bf x} + (1-s)b_s({\bf x}) = D_s({\bf x}).
		\label{eqn:app:twotime_diagonal}
	\end{equation}

	\textit{(iv) Semigroup.}
	See the following result in~\cref{sec:app:two_time_denoiser_characterization}.
\end{proof}

\subsection{Characterizing the two-time denoiser}\label{sec:app:two_time_denoiser_characterization}
Since $\delta_{s,t}$ lies on the simplex by~\cref{prop:delta_conditions}, we now aim to design objective functions that are entirely simplex-valued, for which cross entropy-based objectives can then be used.
To this end, we now translate the flow map characterizations from~\cref{prop:app:flow_map_char} into conditions on $\delta_{s,t}$.

\begin{proposition}[Flow map characterizations in $\delta$ space]\label{prop:semigroup_delta}\label{prop:app:delta_char}
	The flow map characterizations from~\cref{prop:app:flow_map_char} translate into the following conditions on $\delta_{s,t}$:
	\begin{enumerate}[label=(\roman*)]
		\item The Lagrangian condition.
		      For all ${\bf x} \in \R^d$ and $(s,t) \in [0,1]^2$:
		      \begin{align}
			      \label{eq:app:lagrangian_delta}
			      \delta_{s,t}({\bf x}) = \delta_{t,t}(X_{s,t}({\bf x})) - \frac{(1-t)(t-s)}{1-s}\partial_t \delta_{s,t}({\bf x}).
		      \end{align}
		\item The Eulerian condition.
		      For all ${\bf x} \in \R^d$ and $(s,t) \in [0,1]^2$:
		      \begin{align}
			      \label{eq:app:eulerian_delta}
			      \delta_{s,t}({\bf x}) = \delta_{s,s}({\bf x}) + \frac{t-s}{1-t}\Big((1-s)\,\partial_s \delta_{s,t}({\bf x}) + (\delta_{s,s}({\bf x}) - {\bf x}) \cdot \nabla \delta_{s,t}({\bf x})\Big).
		      \end{align}
		\item The semigroup condition.
		      For all ${\bf x} \in \R^d$ and $(s, u, t) \in [0, 1]^3$,
		      \begin{equation}
			      \begin{aligned}
				      \delta_{s,t}({\bf x}) & = \gamma \cdot \delta_{s,u}({\bf x}) + (1-\gamma) \cdot \delta_{u,t}(X_{s,u}({\bf x})), \\
				      \gamma                & = \frac{(1-t)(u-s)}{(1-u)(t-s)}.
			      \end{aligned}
			      \label{eq:app:semigroup_delta}
		      \end{equation}
		      When $s \leq u \leq t$, the coefficients satisfy $\gamma, 1-\gamma \geq 0$, so this is a convex combination.
		      At the midpoint $u = (s+t)/2$, the weights simplify to $\gamma = (1-t)/(2-s-t)$ and $1-\gamma = (1-s)/(2-s-t)$.
	\end{enumerate}
\end{proposition}

\begin{proof}
	\textit{(i) Lagrangian.}
	Differentiating the convex combination~\cref{eq:app:flow_map_denoiser} in $t$:
	\begin{equation}
		\partial_t X_{s,t}({\bf x}) = -\frac{1}{1-s}{\bf x} + \frac{1}{1-s}\delta_{s,t}({\bf x}) + \frac{t-s}{1-s}\partial_t\delta_{s,t}({\bf x}).
	\end{equation}
	By the Lagrangian equation~\cref{eq:app:lagrangian}, $\partial_t X_{s,t}(\mathbf{x}) = b_t(X_{s,t}(\mathbf{x}))$.
	Rewriting $b_t$ via the denoiser-velocity relation~\cref{eq:x_prediction_target}:
	\begin{equation}
		\partial_t X_{s,t}({\bf x}) = \frac{D_t(X_{s,t}({\bf x})) - X_{s,t}({\bf x})}{1-t}.
	\end{equation}
	Multiplying both sides by $(1-t)$, adding $X_{s,t}$, and substituting~\cref{eq:app:flow_map_denoiser}:
	\begin{equation}
		\begin{aligned}
			D_t(X_{s,t}({\bf x}))
			 & = X_{s,t}({\bf x}) + (1-t)\partial_t X_{s,t}({\bf x})                                                                                                                                                                                               \\
			 & = \underbrace{\frac{1-t}{1-s}{\bf x} - \frac{1-t}{1-s}{\bf x}}_{=\,0} + \underbrace{\frac{t-s}{1-s}\delta_{s,t}({\bf x}) + \frac{1-t}{1-s}\delta_{s,t}({\bf x})}_{=\,\delta_{s,t}({\bf x})} + \frac{(1-t)(t-s)}{1-s}\partial_t\delta_{s,t}({\bf x}) \\
			 & = \delta_{s,t}({\bf x}) + \frac{(1-t)(t-s)}{1-s}\partial_t\delta_{s,t}({\bf x}).
		\end{aligned}
	\end{equation}
	Since $\delta_{t,t} = D_t$ on the diagonal~\cref{eqn:app:twotime_diagonal}, the left-hand side is $\delta_{t,t}(X_{s,t}({\bf x}))$, giving~\cref{eq:app:lagrangian_delta}.

	\textit{(ii) Eulerian.}
	We substitute~\cref{eq:app:flow_map_denoiser} into the Eulerian equation~\cref{eq:app:eulerian}.
	Differentiating~\cref{eq:app:flow_map_denoiser} in $s$:
	\begin{equation}
		\partial_s X_{s,t}({\bf x}) = \frac{1-t}{(1-s)^2}({\bf x} - \delta_{s,t}({\bf x})) + \frac{t-s}{1-s}\partial_s\delta_{s,t}({\bf x}).
	\end{equation}
	The spatial Jacobian of~\cref{eq:app:flow_map_denoiser} is:
	\begin{equation}
		\nabla X_{s,t}({\bf x}) = \frac{1-t}{1-s}\Id + \frac{t-s}{1-s}\nabla\delta_{s,t}({\bf x}).
	\end{equation}
	By the denoiser-velocity relation~\cref{eq:x_prediction_target}, the advection velocity is $b_s({\bf x}) = (\delta_{s,s}({\bf x}) - {\bf x})/(1-s)$.
	Substituting into~\cref{eq:app:eulerian} and expanding $b_s \cdot \nabla X_{s,t}$:
	\begin{equation}
		\begin{aligned}
			0 & = \frac{1-t}{(1-s)^2}({\bf x} - \delta_{s,t}({\bf x})) + \frac{t-s}{1-s}\partial_s\delta_{s,t}({\bf x})                                                    \\
			  & \quad + \frac{(1-t)(\delta_{s,s}({\bf x}) - {\bf x})}{(1-s)^2} + \frac{(t-s)(\delta_{s,s}({\bf x}) - {\bf x})}{(1-s)^2} \cdot \nabla\delta_{s,t}({\bf x}).
		\end{aligned}
	\end{equation}
	The first and third terms combine to $\frac{1-t}{(1-s)^2}(\delta_{s,s}({\bf x}) - \delta_{s,t}({\bf x}))$.
	Dividing through by $\frac{t-s}{1-s}$ and rearranging gives~\cref{eq:app:eulerian_delta}.

	\textit{(iii) Semigroup.}
	We express each side of $X_{s,t}({\bf x}) = X_{u,t}(X_{s,u}({\bf x}))$ using~\cref{eq:app:flow_map_denoiser}.
	The left-hand side is:
	\begin{equation}
		X_{s,t}({\bf x}) = \frac{1-t}{1-s}{\bf x} + \frac{t-s}{1-s}\delta_{s,t}({\bf x}).
	\end{equation}
	For the right-hand side, define ${\bf z} := X_{s,u}({\bf x}) = \frac{1-u}{1-s}{\bf x} + \frac{u-s}{1-s}\delta_{s,u}({\bf x})$.
	Then:
	\begin{equation}
		\begin{aligned}
			X_{u,t}({\bf z})
			 & = \frac{1-t}{1-u}{\bf z} + \frac{t-u}{1-u}\delta_{u,t}({\bf z})                                                                    \\
			 & = \frac{1-t}{1-u}\left[\frac{1-u}{1-s}{\bf x} + \frac{u-s}{1-s}\delta_{s,u}({\bf x})\right] + \frac{t-u}{1-u}\delta_{u,t}({\bf z}) \\
			 & = \frac{1-t}{1-s}{\bf x} + \frac{(1-t)(u-s)}{(1-u)(1-s)}\delta_{s,u}({\bf x}) + \frac{t-u}{1-u}\delta_{u,t}({\bf z}).
		\end{aligned}
	\end{equation}
	Equating with the left-hand side and cancelling $\frac{1-t}{1-s}{\bf x}$:
	\begin{equation}
		\frac{t-s}{1-s}\delta_{s,t}({\bf x}) = \frac{(1-t)(u-s)}{(1-u)(1-s)}\delta_{s,u}({\bf x}) + \frac{t-u}{1-u}\delta_{u,t}({\bf z}).
	\end{equation}
	Multiplying both sides by $(1-s)/(t-s)$:
	\begin{equation}
		\delta_{s,t}({\bf x}) = \frac{(1-t)(u-s)}{(1-u)(t-s)}\delta_{s,u}({\bf x}) + \frac{(t-u)(1-s)}{(1-u)(t-s)}\delta_{u,t}({\bf z}).
	\end{equation}
	Define $\gamma := (1-t)(u-s)/\bigl((1-u)(t-s)\bigr)$.
	When $s \leq u \leq t \leq 1$, every factor is non-negative, so $\gamma \geq 0$.
	To show the second coefficient equals $1-\gamma$, we verify the two coefficients sum to one:
	\begin{equation}
		\begin{aligned}
			(1-t)(u-s) + (t-u)(1-s)
			 & = u - s - tu + ts + t - ts - u + us \\
			 & = (t-s) - u(t-s)                    \\
			 & = (t-s)(1-u).
		\end{aligned}
	\end{equation}
	Dividing by $(1-u)(t-s)$ confirms the coefficients sum to one, giving~\cref{eq:app:semigroup_delta}.
\end{proof}
\subsection{Learning the two-time denoiser}\label{sec:app:delta_denoiser_learning}

\begin{figure}[t!]
	\centering
	\includegraphics[width=0.85\linewidth]{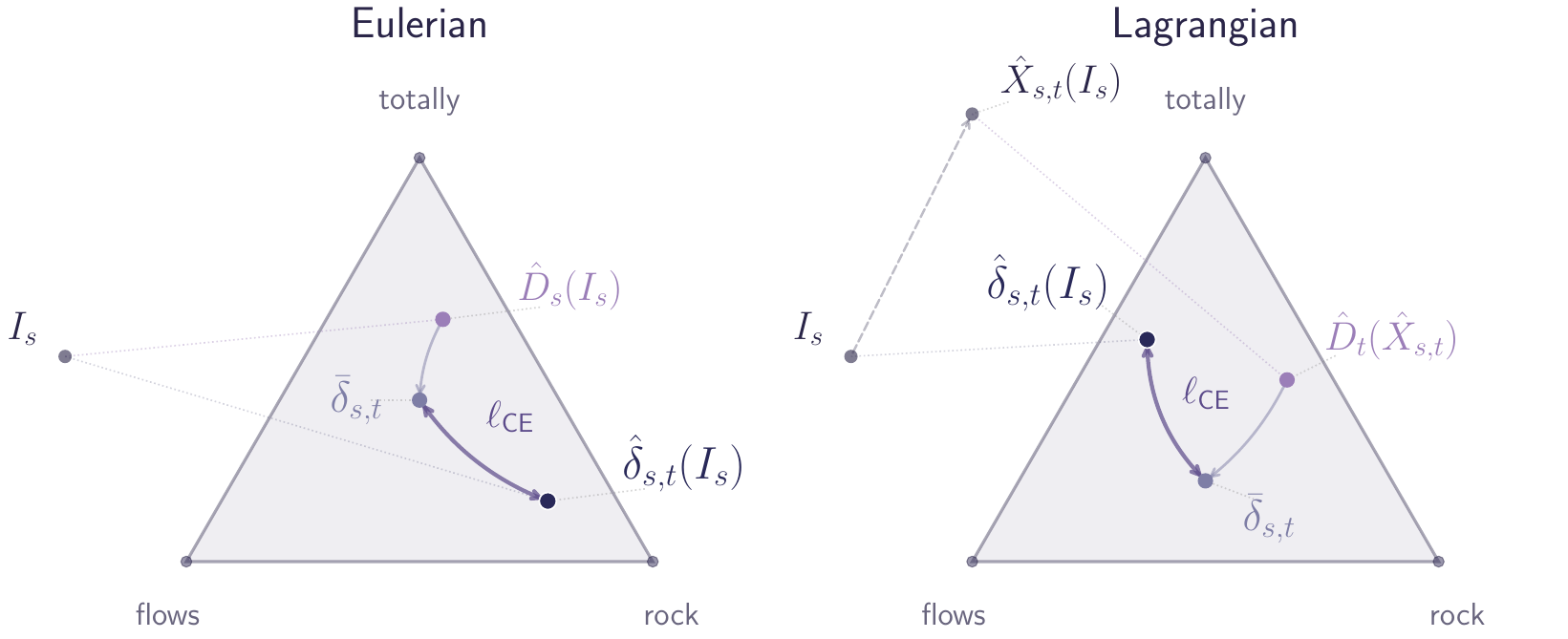}
	\caption{
		\textbf{Eulerian and Lagrangian objectives on the simplex.}
		(Left) The Eulerian teacher $\bar{\delta}_{s,t}$ is constructed from $\hat{D}_s(I_s)$ and derivatives of $\hat{\delta}_{s,t}$.
		(Right) The Lagrangian teacher is constructed from $\hat{D}_t(\hat{X}_{s,t}(I_s))$, requiring an intermediate flow map evaluation off the simplex.
		In both cases, the teacher may transiently leave the simplex during training due to derivative correction terms, but the cross-entropy loss remains well-defined since only the student $\hat{\delta}_{s,t}(I_s)$ (parameterized with softmax) must lie on the simplex.
		At optimality, all quantities lie on the simplex.
	}
	\label{fig:simplex_lag_eul}
\end{figure}

Each characterization in~\cref{prop:app:delta_char} gives $\delta_{s,t} = \text{(teacher)}$, and each objective below enforces one such condition via an off-diagonal loss, plus a diagonal term that anchors $\hat{\delta}_{t,t}$ to the denoiser.
We first restate \cref{prop:semigroup_ce_loss} from the main text.

\semigroupce*

We now prove this as part of two larger propositions that provide Lagrangian, Eulerian, and semigroup objectives for both distillation and self-distillation, mirroring the flow map objectives in~\cref{prop:app:map_distill,prop:app:self_distill}.
For the semigroup characterization, the teacher is a convex combination of simplex elements and therefore lies on the simplex, so we can use the KL divergence.
For the Lagrangian and Eulerian characterizations, the teachers may transiently leave the simplex due to derivative terms; for these we use cross-entropy, which remains well-defined whenever the student is parameterized with softmax.
For the diagonal distillation term, $\hat{D}_t$ is on the simplex, so we again use KL.

\begin{proposition}[Denoiser flow map distillation]\label{prop:app:delta_distill}
	Given a pre-trained denoiser $\hat{D}_t$, the corresponding two-time denoiser $\delta_{s,t}$ is the unique critical point of the following losses:
	\begin{enumerate}[label=(\roman*)]
		\item The Lagrangian loss:
		      \begin{equation}
			      \label{eq:app:delta_lag_distill}
			      \calL_{\mathsf{lag}}(\hat{\delta}) = -\E_{s,t}\,\E_{{\bf x}_0, {\bf x}_1}\left[\sum_{l=1}^{L} \sg{\bar{\delta}^l_{s,t}} \cdot \log \hat{\delta}^l_{s,t}(I_s)\right] + \E_t\,\E_{{\bf x}_0, {\bf x}_1}\left[\sum_{l=1}^{L} \kl{\hat{D}^l_t(I_t)}{\hat{\delta}^l_{t,t}(I_t)}\right],
		      \end{equation}
		      where $\bar{\delta}_{s,t} \coloneqq \hat{D}_t(\hat{X}_{s,t}(I_s)) - \frac{(1-t)(t-s)}{1-s}\partial_t \hat{\delta}_{s,t}(I_s)$ is given by the right-hand side of~\cref{eq:app:lagrangian_delta} with the pre-trained $\hat{D}_t$ replacing $\hat{\delta}_{t,t}$.
		\item The Eulerian loss:
		      \begin{equation}
			      \label{eq:app:delta_eul_distill}
			      \calL_{\mathsf{eul}}(\hat{\delta}) = -\E_{s,t}\,\E_{{\bf x}_0, {\bf x}_1}\left[\sum_{l=1}^{L} \sg{\bar{\delta}^l_{s,t}} \cdot \log \hat{\delta}^l_{s,t}(I_s)\right] + \E_t\,\E_{{\bf x}_0, {\bf x}_1}\left[\sum_{l=1}^{L} \kl{\hat{D}^l_t(I_t)}{\hat{\delta}^l_{t,t}(I_t)}\right],
		      \end{equation}
		      where $\bar{\delta}_{s,t} \coloneqq \hat{D}_s(I_s) + \frac{t-s}{1-t}\left((1-s)\,\partial_s \hat{\delta}_{s,t}(I_s) + (\hat{D}_s(I_s) - I_s) \cdot \nabla \hat{\delta}_{s,t}(I_s)\right)$ is given by the right-hand side of~\cref{eq:app:eulerian_delta}, with the pre-trained $\hat{D}_s$ replacing $\hat{\delta}_{s, s}$.
		\item The semigroup loss:
		      \begin{equation}
			      \label{eq:app:delta_semi_distill}
			      \calL_{\mathsf{semi}}(\hat{\delta}) = \E_{s,u,t}\,\E_{{\bf x}_0, {\bf x}_1}\left[\sum_{l=1}^{L} \kl{\sg{\bar{\delta}^l_{s,t}}}{\hat{\delta}^l_{s,t}(I_s)}\right] + \E_t\,\E_{{\bf x}_0, {\bf x}_1}\left[\sum_{l=1}^{L} \kl{\hat{D}^l_t(I_t)}{\hat{\delta}^l_{t,t}(I_t)}\right],
		      \end{equation}
		      where $\bar{\delta}_{s,t} \coloneqq \gamma\, \hat{\delta}_{s,u}(I_s) + (1-\gamma)\, \hat{\delta}_{u,t}(\hat{X}_{s,u}(I_s))$ is given by the right-hand side of~\cref{eq:app:semigroup_delta}, and where the expectation over $(s,u,t)$ has full support on $\{0 \leq s \leq u \leq t \leq 1\}$.
	\end{enumerate}
	Above, $\hat{X}_{s,t}$ is recovered from $\hat{\delta}_{s, t}$ via~\cref{eq:app:flow_map_denoiser}.
\end{proposition}
\begin{proof}
	Since $\hat{D}_t$ is frozen, the diagonal KL is minimized when $\hat{\delta}_{t,t} = \hat{D}_t$.
	The off-diagonal term is minimized when $\hat{\delta}_{s,t}$ equals the corresponding teacher, which by~\cref{prop:app:delta_char} is equivalent to $\hat{\delta}$ satisfying the Lagrangian, Eulerian, or semigroup characterization.
	Both terms are simultaneously minimized if and only if $\hat{\delta} = \delta$, giving uniqueness.
\end{proof}

We now state an analogous result for direct training via self-distillation.
Since ${\bf x}_1$ is one-hot, the diagonal reduces to cross-entropy~\cref{eq:flow_x_classification_loss}.
\begin{proposition}[Denoiser self-distillation]\label{prop:app:delta_self_distill}
	The two-time denoiser $\delta_{s,t}$ is the unique critical point of the following losses:
	\begin{enumerate}[label=(\roman*)]
		\item The Lagrangian loss:
		      \begin{equation}
			      \label{eq:app:delta_lag_sd}
			      \calL_{\mathsf{lag}}^{\mathsf{sd}}(\hat{\delta}) = -\E_{s,t}\,\E_{{\bf x}_0, {\bf x}_1}\left[\sum_{l=1}^{L} \sg{\bar{\delta}^l_{s,t}} \cdot \log \hat{\delta}^l_{s,t}(I_s)\right] - \E_t\,\E_{{\bf x}_0, {\bf x}_1}\left[\sum_{l=1}^{L} {\bf x}^l_1 \cdot \log \hat{\delta}^l_{t,t}(I_t)\right],
		      \end{equation}
		      where $\bar{\delta}_{s,t} \coloneqq \hat{\delta}_{t,t}(\hat{X}_{s,t}(I_s)) - \frac{(1-t)(t-s)}{1-s}\partial_t \hat{\delta}_{s,t}(I_s)$ is given by the right-hand side of~\cref{eq:app:lagrangian_delta}.
		\item The Eulerian loss:
		      \begin{equation}
			      \label{eq:app:delta_eul_sd}
			      \calL_{\mathsf{eul}}^{\mathsf{sd}}(\hat{\delta}) = -\E_{s,t}\,\E_{{\bf x}_0, {\bf x}_1}\left[\sum_{l=1}^{L} \sg{\bar{\delta}^l_{s,t}} \cdot \log \hat{\delta}^l_{s,t}(I_s)\right] - \E_t\,\E_{{\bf x}_0, {\bf x}_1}\left[\sum_{l=1}^{L} {\bf x}^l_1 \cdot \log \hat{\delta}^l_{t,t}(I_t)\right],
		      \end{equation}
		      where $\bar{\delta}_{s,t} \coloneqq \hat{\delta}_{s,s}(I_s) + \frac{t-s}{1-t}\left((1-s)\,\partial_s \hat{\delta}_{s,t}(I_s) + (\hat{\delta}_{s,s}(I_s) - I_s) \cdot \nabla \hat{\delta}_{s,t}(I_s)\right)$ is given by the right-hand side of~\cref{eq:app:eulerian_delta}.
		\item The semigroup loss:
		      \begin{equation}
			      \label{eq:app:delta_semi_sd}
			      \calL_{\mathsf{semi}}^{\mathsf{sd}}(\hat{\delta}) = \E_{s,u,t}\,\E_{{\bf x}_0, {\bf x}_1}\left[\sum_{l=1}^{L} \kl{\sg{\bar{\delta}^l_{s,t}}}{\hat{\delta}^l_{s,t}(I_s)}\right] - \E_t\,\E_{{\bf x}_0, {\bf x}_1}\left[\sum_{l=1}^{L} {\bf x}^l_1 \cdot \log \hat{\delta}^l_{t,t}(I_t)\right],
		      \end{equation}
		      where $\bar{\delta}_{s,t} \coloneqq \gamma\, \hat{\delta}_{s,u}(I_s) + (1-\gamma)\, \hat{\delta}_{u,t}(\hat{X}_{s,u}(I_s))$ is given by the right-hand side of~\cref{eq:app:semigroup_delta}, and where the expectation over $(s,u,t)$ has full support on $\{0 \leq s \leq u \leq t \leq 1\}$.
	\end{enumerate}
\end{proposition}
\begin{proof}
	The off-diagonal term minimized if and only if the corresponding characterization holds.
	Since ${\bf x}_1$ is one-hot, the diagonal cross-entropy is bounded below by the conditional entropy $\E_t\,\E[H({\bf x}_1 \mid I_t)]$, achieved when $\hat{\delta}_{t,t} = D_t$~(\cref{sec:classification_proof}).
	Both terms are simultaneously minimized if and only if $\hat{\delta} = \delta$, giving uniqueness.
\end{proof}

\subsection{Decoding error rate and entropic time reparameterizations}\label{sec:app:entropic_time}

In this subsection, we show a connection between our time reparameterization \cref{eq:time_reparameterization}, that linearizes a decoding error probability, and the entropic time reparameterization proposed in \citet{dieleman2022continuous} and \citet{stancevic2025entropic}, that linearizes uncertainty in the generation.

The entropic time reparameterization is defined such that each timepoint contributes equally in resolving the uncertainty in the final generation, naturally quantified by the conditional entropy of clean data conditioned on the noisy state.
We consider a standardized entropic time $\sigma:[0,1]\to[0,1]$ that, as proposed in \citet{dieleman2022continuous}, linearizes the token-level conditional entropy $\hat{H}({\bf x}_1|I_t)\coloneqq \sum_{l=1}^L H(p_{1|t}^l(\cdot\mid I_t))$:
\begin{align}\label{eq:app:standardized_entropic_time}
	\sigma(t) \coloneqq 1 - \frac{\hat{H}({\bf x}_1|I_t)}{\hat{H}({\bf x}_1|{\bf x}_0)} = 1-\frac{\hat{H}({\bf x}_1|I_t)}{\hat{H}({\bf x}_1)}.
\end{align}
In \citet{dieleman2022continuous}, $\hat{H}({\bf x}_1|I_t)$ is estimated online using the training loss of the denoiser.

We now show a relationship between the entropic time $\sigma(t)$ and our decoding error-based time reparameterization $\tau(t)$ \cref{eq:time_reparameterization} as an approximate asymptotic inequality with a vanishing margin, leveraging Fano's inequality which relates the probability of decoding error and conditional entropy in a noisy communication channel.
\begin{proposition}\label{prop:app:decoding_error_and_entropic_time}
	Assume that ${\bf x}_1^l$ at each position $l$ is distributed uniformly over $V$.
	Then:
	\begin{equation}\label{eq:app:time_lower_bound}
		\tau(t) \leq \sigma(t) + O\left(\frac{1}{\log|V|}\right) \qquad\text{as}\qquad |V|\longrightarrow\infty.
	\end{equation}
\end{proposition}
\begin{proof}
	For each token position $l$ and flow time $t$, let us denote by $P_e^l(t)$ the probability of an occurrence of decoding error ${\bf x}_1^l\neq \mathsf{argmax}(I_t^l)$ where $\mathsf{argmax}(I_t^l)$ is viewed as an approximation of ${\bf x}_1^l$.
	By Fano's inequality,
	\begin{equation}
		\hat{H}({\bf x}_1|I_t) = \sum_{l=1}^L H(p_{1|t}^l(\cdot|I_t)) \leq \sum_{l=1}^L \left[P_e^l(t)\log (|V| - 1) + h(P_e^l(t))\right] = L\log (|V| - 1) P_e(t) + \sum_{l=1}^Lh(P_e^l(t)),
	\end{equation}
	where $h(p) = p \log\frac{1}{p} + (1-p)\log\frac{1}{1-p}$ is the binary entropy function, and the last equality follows from \cref{eqn:decoding_error}.
	Then, by \eqref{eq:app:standardized_entropic_time}:
	\begin{equation}
		\sigma(t) \geq 1 - \frac{L\log (|V| - 1) P_e(t) + \sum_l h(P_e^l(t))}{\hat{H}({\bf x}_1)}.
	\end{equation}
	Since $h(\cdot)\leq \log 2$, and $\hat{H}({\bf x}_1) = L\log|V|$ due to the uniform distribution assumption, we have
	\begin{equation}
		\sigma(t) + \frac{\log2}{\log|V|} \geq 1 - \frac{\log (|V| - 1)}{\log|V|}P_e(t) \longrightarrow 1 - P_e(t)\qquad \text{as}\qquad |V|\longrightarrow\infty.
	\end{equation}
	By comparing this with $\tau(t) = 1-\frac{|V|}{|V|-1}P_e(t) \to 1 - P_e(t)$ as $|V|\to\infty$ \cref{eq:time_reparameterization}, and noting $\log2/\log|V|=O(1/\log|V|)$, we arrive at \cref{eq:app:time_lower_bound}.
\end{proof}

The assumption of \cref{prop:app:decoding_error_and_entropic_time} is rather strong: it is more realistic to use a non-uniform distribution following Zipf's law, which would give us $\sigma + O(1/\log|V|) \geq 1 - cP_e$ for some $c>1$.
Nevertheless, the result reveals a close connection between our decoding error-based time reparameterization and entropic time.
While we empirically observed in \cref{sec:ablation} that entropic time based on online estimation \citep{dieleman2022continuous} underperformed our time reparameterization, this relationship suggests that the reason might lie more in the approximation based on training loss, than the linearization of entropy itself.

\section{Two-stage flow map distillation with squared loss}\label{sec:app:two_stage_distillation}

In a previous version of our work, we have developed a two-stage distillation scheme for flow map learning that uses squared semigroup loss \cref{eq:semigroup_loss}. While this is more cumbersome than cross-entropy distillation in \cref{sec:flow_map_implementation}, it performs reasonably well, and hence we include its characterization and empirical results here, which might be useful in future work.

\paragraph{Motivation}
Following existing work \citep{boffi2025build, boffi2025flowmapmatchingstochastic}, one may learn $\flowmap$ via distillation from a pre-trained $\flowvel$ as follows: parameterize the flow map via the average velocity $\hat{v}_{s,t}$ \cref{eq:flow_map}, and train $\hat{v}_{s,t}$ with \cref{eq:semigroup_loss} to enforce the semigroup condition, while jointly training $\hat{v}_{t,t}$ to match a pre-trained $\flowvel$ $\hat{b}_t$ to enforce the tangent condition $b_t = v_{t,t}$. As discussed in \cref{sec:flow_map}, this approach does not utilize the one-hot geometry, and is outperformed by cross-entropy distillation (\cref{tab:ablation_flowvel}). Nevertheless, we discover that an alternative, \emph{two-stage} distillation scheme can stabilize squared-loss distillation, providing another route for learning $\flowmap$. Here, the first stage learns a \emph{correction} to the trained $\flowvel$ that converts Euler steps into accurate flow map jumps, and the second stage compresses this into a single flow map model for improved efficiency.

\paragraph{First stage.} Recall that a single Euler step computes ${\bf x} + (t-s)\hat{b}_s({\bf x})$, which incurs discretization error for large steps. 
We learn a \emph{correction model} $\hat{\psi}_{s,t}$ which predicts the correction needed to convert the Euler estimate into the true flow map.
Specifically, we parameterize the flow map as:
\begin{align}
	\label{eq:two_models_parameterization}
	\hat{X}_{s, t}({\bf x}) \coloneqq {\bf x} + (t-s)\,\hat{b}_s({\bf x}) + \frac{1}{2}(t-s)^2 \hat{\psi}_{s,t}({\bf x}),
\end{align}
where $\hat{b}$ is an $\flowvel$ trained following \Cref{sec:flow_velocity_implementation}, possibly as a denoiser $\hat{D}$ based on \Cref{eq:x_prediction_target}.
This parameterization was proposed by \citet{boffi2025build} but not empirically tested.
By construction, it satisfies the boundary condition and tangent condition, so we only need to enforce the semigroup condition through training.
We initialize $\hat{\psi}$ from the parameters of $\hat{b}$ by removing the output softmax and zeroing the final layer, and train using the semigroup loss \eqref{eq:semigroup_loss} re-written in terms of clean data prediction.
For this, observe that the average velocity $\hat{v}$ is given as follows, from \cref{eq:flow_map}:
\begin{equation}
	\hat{v}_{s,t}({\bf x}) = \frac{\hat{X}_{s,t}({\bf x}) - {\bf x}}{t-s} = \hat{b}_s({\bf x}) + \frac{1}{2}(t-s)\, \hat{\psi}_{s,t}({\bf x}).
	\label{eq:app:avg_velocity_two_model}
\end{equation}
Using the relationship between the denoiser and velocity in \cref{eq:x_prediction_target}, the integrand of the semigroup loss \cref{eq:semigroup_loss} can be written as:
\begin{equation}
	\begin{aligned}
		\E| \hat{X}_{s,t}(I_s) - \sg{\hat{X}_{u,t}(\hat{X}_{s,u}(I_s))} |^2
		 & = \E| I_s + (t-s)\hat{v}_{s,t}(I_s) - \sg{I_s + (t-s)\bar{v}_{s,t}} |^2                                                                    \\
		 & = (t-s)^2\,\E\left| \hat{v}_{s,t}(I_s) - \sg{\bar{v}_{s,t}} \right|^2                                                                      \\
		 & = (t-s)^2\,\E\left| \frac{\hat{D}_s(I_s) - I_s}{1-s} + \frac{t-s}{2}\hat{\psi}_{s,t}(I_s) - \sg{\frac{\bar{\bf x}_1 - I_s}{1-s}} \right|^2 \\
		 & = \frac{(t-s)^2}{(1-s)^2}\,\E\left| \hat{D}_s(I_s) + \frac{(t-s)(1-s)}{2} \hat{\psi}_{s,t}(I_s) - \sg{\bar{\mathbf{x}}_1} \right|^2
	\end{aligned}
	\label{eq:app:semigroup_loss_derivation}
\end{equation}
where the bootstrapped velocity $\bar{v}_{s,t}$ and target $\bar{\mathbf{x}}_1$ are given as:
\begin{equation}
	\begin{aligned}
		\bar{v}_{s,t} \coloneqq \frac{u-s}{t-s}\hat{v}_{s,u}(I_s) + \frac{t-u}{t-s}\hat{v}_{u,t}(\hat{X}_{s,u}(I_s)), \qquad \bar{\mathbf{x}}_1  \coloneqq \sg{I_s + (1-s)\bar{v}_{s,t}}.
	\end{aligned}
	\label{eq:app:bootstrapped_target}
\end{equation}
Following \citet{boffi2025build}, we drop the scale term $(\frac{t-s}{1-s})^2$ which changes the effective learning rate depending on the step sizes $t-s$ and $1-s$, for additional training stability.
Then the final denoising loss on the correction model $\hat{\psi}$ becomes:
\begin{equation}
	\calL_\mathsf{MSE}(\hat{\psi}) = \int_0^1\int_0^t\int_s^t \E \left| \hat{D}_s(I_s) + \frac{(t-s)(1-s)}{2} \hat{\psi}_{s,t}(I_s) - \sg{\bar{\bf x}_1} \right|^2 {\rm d}u{\rm d}s{\rm d}t.
	\label{eq:app:denoising_loss_final}
\end{equation}
Since $\hat{b}$ is frozen and $\hat{\psi}$ only learns the residual correction, the training is efficient and converges quickly.

\paragraph{Second stage.} The two-model flow map $\hat{X}$, composed of $\hat{b}$ and $\hat{\psi}$ (or $\hat{\phi}$), doubles the memory cost at inference.
We distill it into a single-model flow map $\hat{Y}$ parameterized as:
\begin{equation}
	\hat{Y}_{s,t}({\bf x}) \coloneqq {\bf x} + (t - s)\hat{u}_{s,t}({\bf x}).
	\label{eq:single_model_param}
\end{equation}
We initialize $\hat{u}$ from $\flowvel$ $\hat{b}$ by removing the output softmax, and train by solving a simple regression problem onto the two-model teacher $\hat{X}$ which is frozen throughout:
\begin{equation}
	\calL_\mathsf{MSE}(\hat{Y}) \coloneqq \int_0^1\int_0^t \E | \hat{Y}_{s,t}(I_s) - \hat{X}_{s,t}(I_s) |^2{\rm d}s{\rm d}t.
	\label{eq:app:second_stage_loss}
\end{equation}
This has several desirable properties that yield fast and stable convergence.
The teacher provides targets via a single forward pass, without requiring iterative sampling or trajectory simulation.
The loss is lower-bounded by zero with a unique global minimizer at $\hat{Y} = \hat{X}$, allowing us to directly track distillation quality during training.
Lastly, it is strongly convex in $\hat{Y}$, ensuring well-conditioned optimization.

\paragraph{Time reparameterization and other details.}

\begin{figure}[t!]
	\captionof{table}{
		Generation performance of $\flowmap$ trained with cross-entropy distillation (\cref{sec:flow_map,sec:flow_map_implementation}) and two-stage squared-loss distillation (\cref{sec:app:two_stage_distillation}) in the extreme few-step regime.
	}\label{tab:onetwostepresults_twostagedistillation}
	\begin{center}
		\begin{adjustbox}{max width=\linewidth}
			\setlength{\tabcolsep}{6.4pt}
			\begin{tabular}{@{\extracolsep{\fill}}c cccccc}
				\toprule
				LM1B  & \multicolumn{2}{c}{\bf $\flowmap$ (CE; \cref{sec:flow_map_implementation})} & \multicolumn{2}{c}{\bf $\flowmap$ (MSE, first stage \cref{eq:app:denoising_loss_final})} & \multicolumn{2}{c}{\bf $\flowmap$ (MSE, second stage \cref{eq:app:second_stage_loss})}                                               \\
				\cmidrule(lr){2-3} \cmidrule(lr){4-5} \cmidrule(lr){6-7}
				Steps & Gen. PPL $(\downarrow)$                                                     & Entropy                                                                                  & Gen. PPL $(\downarrow)$                                                                & Entropy & Gen. PPL $(\downarrow)$ & Entropy \\
				\midrule
				1     & 119.34                                                                      & 4.16                                                                                     & 102.49                                                                                 & 4.13    & 104.37                  & 4.12    \\
				2     & 110.19                                                                      & 4.21                                                                                     & 93.65                                                                                  & 4.17    & 95.42                   & 4.15    \\
				4     & 98.76                                                                       & 4.21                                                                                     & 88.86                                                                                  & 4.17    & 90.90                   & 4.16    \\
				\midrule
				OWT   & \multicolumn{2}{c}{\bf $\flowmap$ (CE; \cref{sec:flow_map_implementation})} & \multicolumn{2}{c}{\bf $\flowmap$ (MSE, first stage \cref{eq:app:denoising_loss_final})}                                               \\
				\cmidrule(lr){2-3} \cmidrule(lr){4-5}
				Steps & Gen. PPL $(\downarrow)$                                                     & Entropy                                                                                  & Gen. PPL $(\downarrow)$ & Entropy \\
				\cmidrule(lr){1-5}
				1     & 168.30                                                                      & 5.17                                                                                     & 129.32                                                                                 & 4.53 \\
				2     & 133.29                                                                      & 5.25                                                                                     & 134.26                                                                                 & 5.07 \\
				4     & 111.31                                                                      & 5.26                                                                                     & 76.37                                                                                  & 5.05 \\
				\cmidrule(lr){1-5}
			\end{tabular}
		\end{adjustbox}
	\end{center}
\end{figure}

Similarly to \cref{sec:flow_map_implementation}, we use the time reparameterization $\tau(t)$ from \Cref{eq:time_reparameterization} for the flow maps $\hat{X}$ and $\hat{Y}$.
For the first-stage distillation, we sample time triplets $(s, u, t)$ using the midpoint sampling scheme described in \cref{sec:flow_map_implementation}.
The second-stage distillation uses the same scheme, but without the need to sample $u$.

The rest of the settings differ slightly between LM1B and OWT datasets.
These details only apply for two-stage distillation explained in this section, and do not apply to the cross-entropy distillation explained in main text.
For LM1B, similarly to \cref{sec:flow_map_implementation}, we fix a probability of 1/64 of sampling the boundary pair $(s, t) = (0, 1)$, so that the model receives sufficient training signal for one-step generation.
For OWT, we use a different boundary condition, $s = 0$ with a probability of 1/32, which has a similar effect but we found to empirically work better.
For OWT training, we additionally use a progressive warm-up of the distillation step size $h$: instead of drawing $h\sim \mathsf{U}[0,1]$ throughout training, we start with $h\sim \mathsf{U}[0,\frac{1}{1024}]$ and double the upper bound every 10k steps until it reaches $1$.
In addition, for OWT we find it beneficial to alter the time reparameterization at inference time as follows: we define the reparameterized time $\tau'(t)$ for sampling as a convex combination with the original time, $\tau'(t) \coloneqq \alpha\tau(t) + (1-\alpha)t$, and use optimal $\alpha$ values within $\{0.5, 0.75, 1\}$.
Lastly, during both stages, we follow \citet{boffi2025build} and use the learned loss weighting proposed in \citet{karras2024analyzing}, which stabilizes the gradient variance across the sampled time distribution.

For LM1B, we distill 100k steps for both fist and second stages, respectively. For OWT, we report 300k-step distilled result from the first stage, and were unable to run second-stage distillation due to resource limits.
The other training and sampling details, such as the optimizer setting and learning rate, is shared with cross-entropy distillation and explained in \cref{sec:algorithm}.

\paragraph{Results.}

In \cref{tab:onetwostepresults_twostagedistillation} we present the performance of $\flowmap$ learned with two-stage squared-loss distillation, comparing it with cross-entropy distillation used in the main text.
The first-stage distilled model $\hat{X}$ uses the two-model parameterization of the flow map \cref{eq:two_models_parameterization}, while the second-stage distilled model $\hat{Y}$ uses the single-model parameterization \cref{eq:flow_map}.
On both LM1B and OWT, the first-stage distilled model $\hat{X}$ achieves a comparable performance with cross-entropy distillation, although it uses twice the parameters and trades off entropy especially in one-step generation.
On LM1B, the final single-model student $\hat{Y}$ successfully recovers the performance of its two-model teacher $\hat{X}$, demonstrating effective knowledge transfer between the two parameterizations of flow map.
Overall, the results show that a careful design of the parameterization and learning procedure may stabilize squared-loss distillation.

\section{Implementation details}\label{sec:algorithm}

\paragraph{Time reparameterization.} To efficiently implement the time reparameterization $\tau(t)$ described in \cref{sec:flow_velocity_implementation} without evaluating the probability sum during training, we utilize a precomputed lookup table (LUT) combined with spline interpolation. Specifically, we approximate the cumulative density function (CDF) of \cref{eq:time_reparameterization} using Gauss-Hermite quadrature and evaluate it on a equispaced grid of $1,000$ points over $t \in [0, 1]$, obtaining $(t,\tau)$ pairs at each point. We find that this resolution is sufficient to capture the transition of the schedule with negligible error.
From these discrete pairs, we fit a cubic spline to obtain a continuous and differentiable mapping, and then construct both the forward map $\tau(t)$ and the inverse map $t(\tau)$, which enables $O(1)$ sampling of simulation times during training.
Since this LUT and the associated mappings are computed once prior to training and can be cached, our approach incurs no additional computational overhead.

\paragraph{Training details.}\label{sec:training_details} Both for LM1B and OWT we train $\flowvel$ from scratch for 1M training steps, with batch size of 512. Following the settings from \citet{sahoo2025diffusion}, we use 2,500 warmup steps and then a constant learning rate of $3\times10^{-4}$. For the optimizer, we use Adam~\cite{kingma2014adam} with $\beta_1=0.9$ and $\beta_2=0.999$. Additionally, we utilize softcapping~\cite{team2024gemma} which smooths out large logits in the attention activations, for additional numerical stability of training.
For $\flowmap$, we share all the training settings with $\flowvel$ including batch size and learning rate.
We split each training batch into two sub-batches: one for flow-matching and the other for flow map distillation, simillar to~\citet{frans2024one}. For the flow matching the model recieves two timesteps $(s,t)$ where $s=t$, while for distillation $s\neq t$ is used. Maintaining an equal ratio between these two objectives was sufficient for effective training.
Lastly, we find that the reparameterization $\tau(t)$ has a flat region near $t = 0$ (\cref{fig:reparameterization}), causing the start point $s$ to rarely land near the origin.
This hinders learning of flow maps for one- or two-step generation, where the model must directly transport from $s = 0$ to $t = 1$.
To address this, we fix a probability of directly sampling the boundary: $(s, t) = (0, 1)$ among the distillation batch. We used a probability of 1/32 for both dataset, ensuring the model receives sufficient training signal for few-step generation while not biases toward gradient from large jumps.
For both LM1B and OWT we report the results from 100k steps distillation.

\paragraph{Sampling details.} For $\flowvel$ on both LM1B and OWT, we use Euler solver for sampling.
For $\flowmap$, in both dataset we leverage the ``$\gamma$-sampling'' algorithm from \citet{kim2024consistency} using the optimal $\gamma$ values \cite{sabour2025align}.

\paragraph{Many-step baselines.}
For the LM1B experiments, we trained Duo~\cite{sahoo2025diffusion}, MDLM~\cite{sahoo2024simple}, and CANDI~\cite{pynadath2025candi} from scratch using identical settings, while utilizing the official 1M-step checkpoint for RDLM~\cite{jo2025continuous}. For the OWT experiments, we relied on the official checkpoints provided by the respective authors for CANDI, Duo, and MDLM. Due to absence of the official checkpoint and the limited resource for reproducing, we were not able to compare with RDLM in OWT. For sampling, for all the discrete baselines we used ancestral sampler with temperature 1.0, while for RDLM we use the SDE sampler proposed in the paper.

\paragraph{Few-step baselines.}
For LM1B experiments, we apply SDTT~\cite{wu2025fast} on top of MDLM~\cite{sahoo2024simple}, trained on LM1B for 1M steps. Following the default hyperparameters from the paper, we use a fixed learning rate of $6\times10^{-5}$ with 2,500 warmup steps and batch size of 128. Each distillation round consists of 10k training steps where we perform a total of 8 rounds. We share this setting when applying DCD~\cite{sahoo2025diffusion} on top of Duo trained for 1M steps. For OWT, we use the official distilled checkpoints from repective authors. For Di4C~\cite{hayakawa2024distillation}, we used the intermediate checkpoints with the best 32-step performance among the training, corresponding to 20k training steps for LM1B and 50k for OWT: in both cases, additional training resulted in performance degradation.

\section{Supplementary evaluation results}\label{sec:more_experiments}

\paragraph{Checking mode collapse.}
\begin{table}[t!]
	\caption{Self-BLEU~\cite{zhu2018texygen} score of 1,024 generated samples in the one-step generation setting. Lower score denotes more $n$-gram diversity. For reference, we report the Self-BLEU score of mode-collapsed case when all the samples are identical, and the score of the reference samples from each dataset.}
	\centering
	\label{tab:self_bleu}
	\begin{adjustbox}{max width=\linewidth}
		\begin{small}
			\begin{tabular}{l|c|ccccc|c}
				\toprule
				Dataset & Real data & MDLM + SDTT & MDLM + Di4C & Duo + DCD & Duo + Di4C & \textbf{$\flowmap$ (Ours)} & (mode collapse) \\
				\midrule
				LM1B    & 0.047     & 0.026       & 0.023       & 0.075     & 0.054      & 0.073                      & 1.000           \\
				OWT     & 0.046     & 0.036       & 0.031       & 0.297     & 0.272      & 0.121                      & 1.000           \\
				\bottomrule
			\end{tabular}
		\end{small}
	\end{adjustbox}
\end{table}

\begin{table}[t!]
	\caption{LLM-based win rate (\%) measured by GPT-4.1 as a judge. Win-rate of 0.5 denotes the data distribution has equal amount of diversity compared with real data, when measured by LLM.}
	\centering
	\label{tab:llm-diversity}
	\begin{adjustbox}{max width=\linewidth}
		\begin{small}
			\begin{tabular}{l|ccccc}
				\toprule
				Dataset & MDLM + SDTT & MDLM + Di4C & Duo + DCD & Duo + Di4C & \textbf{$\flowmap$ (Ours)} \\
				\midrule

				LM1B    & 0.64        & 0.79        & 0.46      & 0.38       & 0.39                       \\
				OWT     & 0.41        & 0.45        & 0.68      & 0.35       & 0.42                       \\
				\bottomrule
			\end{tabular}
		\end{small}
	\end{adjustbox}
\end{table}

\begin{table}[t!]
	\centering
	\caption{Performance of $\flowvel$ trained on OWT for 150k steps, across different model sizes.}
	\label{tab:flm_scaling_law}
	\begin{small}
		\addtolength{\tabcolsep}{2pt}
		\begin{tabular}{lccc}
			\toprule
			Model Size & Small (179M) & Medium (424M) & Large (870M) \\
			\midrule
			Gen. PPL ($\downarrow$)   & 75.57        & 65.94         & 61.59        \\
			Entropy                   & 5.40         & 5.42          & 5.39         \\
			\bottomrule
		\end{tabular}
	\end{small}
\end{table}

To ensure that $\flowmap$ does not mode collapse onto a few high-quality samples, we report the Self-BLEU~\cite{zhu2018texygen} score which measures the $n$-gram diversity of generations. The results in \cref{tab:self_bleu} shows that $\flowmap$ clearly does not show mode-collapsing behavior in one-step generation, which would be indicated by a Self-BLEU score $\approx$1.0, as it attains only a slightly worse score compared to real data.

Additionally, we report an LLM-based win rate against real data samples. For evaluation, we use GPT-4.1~\cite{openai2024gpt41} as a judge with the following prompt:

\begin{lavenderbox}
    \small
    I'll show you two pairs of text paragraphs. Each pair contains two samples from the same text distribution. Read both pairs and decide which text distribution has more variance (i.e., is more diverse).
\end{lavenderbox}

From the set of 1,024 one-step generated samples and the real data, we randomly sample a pair from each distribution per request. To mitigate positional bias, we randomly swap the presentation order of each pair \cite{zheng2023judging}.
Results are presented in \Cref{tab:llm-diversity}. While a win rate of 50\% would indicate the same diversity as the real data,
$\flowmap$ achieves reasonable win rates of 39\% and 42\%, reflecting slightly decreased diversity but clearly not mode collapse. Note that the unusually high win rates of the baselines are artifacts of their near-random token outputs, as reflected by their high generative perplexity ($>1000$, \Cref{tab:onetwostepresults}).

\paragraph{Scaling behavior of $\flowvel$.} To investigate the scaling behavior of $\flowvel$, we evaluate performance across varying base model capacities while maintaining a constant number of training iterations. Following the architecture of GPT-2~\cite{radford2019language}, we test Small (179M parameters), Medium (424M parameters), and Large (870M parameters) architectures, where 179M-parameter models are used for our main results of the paper. All models utilize identical training configurations, including the learning rate and optimizer settings, as detailed in \cref{sec:training_details}. As shown in \cref{tab:flm_scaling_law}, increasing the model size yields consistent improvements in generative perplexity while preserving sample entropy. This demonstrates a clear scaling law for $\flowvel$, highlighting the promise of scaling these models to the billion-parameter level in future iterations.

\begin{figure}[t!]
	\centering
	\captionof{table}{Performance of $\flowmap$ with uniform and Gaussian priors trained on LM1B dataset for 200k steps.}
	\label{tab:prior_ablation}
	\begin{small}
		\addtolength{\tabcolsep}{2pt}
		\begin{tabular}{lcc}
			\toprule
			Prior & Uniform & Gaussian \\
			\midrule
			Gen. PPL ($\downarrow$)   & 660.33       & \textbf{105.48} \\
			Entropy                   & 4.12         & 4.31   \\
			\bottomrule
		\end{tabular}
	\end{small}
\end{figure}

\paragraph{Ablation on noise distribution.}
We compare our Gaussian prior $p_0 = \mathsf{N}(0, I)$ with a uniform prior on the probability simplex, $p_0 = \mathsf{Dir}(\mathbf{1})$. The uniform prior largely underperforms, as shown in \Cref{tab:prior_ablation}.
We attribute this to the interaction between concentration of measure on the high-dimensional simplex and our time reparameterization.
For vocabulary size $|V|$, a sample $\mathbf{x}_0 \sim \mathsf{Dir}(\mathbf{1})$ has per-component variance $\sim 1/|V|^2$, so that for large $|V|$, the prior concentrates tightly around the uniform vector $(1/|V|, \dots, 1/|V|)$.
Since this is already close to a uniform mixture of one-hot vectors, the interpolant $I_t = (1-t)\mathbf{x}_0 + t\mathbf{x}_1$ becomes dominated by the data signal $t\mathbf{x}_1$ at very small $t$, causing the decoding error rate $P_e(t)$ to drop to zero almost immediately.
As a consequence, our time reparameterization $\tau(t)$, which is designed to distribute training signal uniformly in proportion to decoding progress, maps nearly the entire $[0,1]$ interval to a vanishingly narrow region near $t=0$.
This collapses the effective training distribution, leaving the model with almost no learning signal across the majority of the flow trajectory.
In contrast, the Gaussian prior $\mathsf{N}(0, I)$ in unconstrained $\mathbb{R}^{|V|}$ produces noise that is far from the data manifold at all vocabulary sizes, ensuring that $P_e(t)$ decreases gradually and the reparameterization distributes training signal across the full time interval.

\begin{table}[t!]
    \centering
    \parbox{.49\linewidth}{
    	\centering
    	\caption{Uniqueness of generated Sudoku grids.}
    	\label{tab:sudoku_results_uniqueness}
        \begin{adjustbox}{max width=\linewidth}
        	\begin{small}
                \addtolength{\tabcolsep}{-2.5pt}
                \begin{tabular}{lcccc}
                    \toprule
                    Model & 1 Steps & 2 Steps & 4 Steps & 1024 Steps \\
                    \midrule
                    MDLM+SDTT~\cite{deschenaux2024beyond}   & 0.00 & 0.00 & 0.02 & 92.19 \\
                    Duo+DCD~\cite{sahoo2025diffusion}     & 0.00 & 0.39 & 17.19 & 91.41 \\\cellcolor{darkblue}\textbf{$\flowmap$ (Ours)} & \cellcolor{darkblue}\textbf{9.38\%} & \cellcolor{darkblue}\textbf{57.03\%} & \cellcolor{darkblue}\textbf{79.69\%} & \cellcolor{darkblue}\textbf{94.92\%} \\
                    \bottomrule
                \end{tabular}
            \end{small}
        \end{adjustbox}
    }
    \hfill
    \parbox{.49\linewidth}{
    	\centering
    	\caption{Novelty of generated Sudoku grids.}
    	\label{tab:sudoku_results_novelty}
        \begin{adjustbox}{max width=\linewidth}
        	\begin{small}
                \addtolength{\tabcolsep}{-2.5pt}
                \begin{tabular}{lcccc}
                    \toprule
                    Model & 1 Steps & 2 Steps & 4 Steps & 1024 Steps \\
                    \midrule
                    MDLM+SDTT~\cite{deschenaux2024beyond}   & 0.00 & 0.00 & 0.02 & 92.19 \\
                    Duo+DCD~\cite{sahoo2025diffusion}     & 0.00 & 0.39 & 17.19 & 91.41 \\\cellcolor{darkblue}\textbf{$\flowmap$ (Ours)} & \cellcolor{darkblue}\textbf{9.38\%} & \cellcolor{darkblue}\textbf{57.03\%} & \cellcolor{darkblue}\textbf{79.69\%} & \cellcolor{darkblue}\textbf{94.92\%} \\
                    \bottomrule
                \end{tabular}
            \end{small}
        \end{adjustbox}
    }
\end{table}

\begin{figure}[t!]
    \centering
    \includegraphics[width=\textwidth]{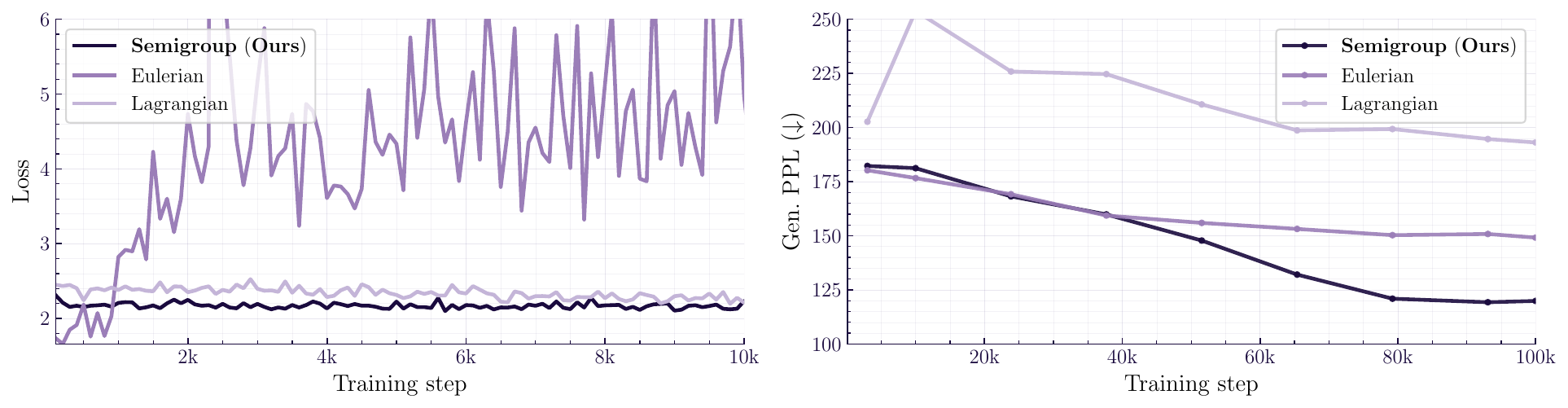}
    \caption{\textbf{Ablation on flow map characterization.} Training loss and generation quality (Gen. PPL) of different flow map distillation objectives.}
    \label{fig:distillation_ablation}
\end{figure}
\paragraph{Ablation on flow map characterization.}

We compare our distillation objective~\cref{eq:semigroup_ce_loss} based on semigroup-based characterization of the flow map against Lagrangian~\Cref{eq:app:delta_lag_distill} and Eulerian~\Cref{eq:app:delta_eul_distill} objectives based on differential characterizations.
As in \Cref{fig:distillation_ablation}, semigroup objective yields stable convergence in training loss as well as generation quality.
Eulerian is highly unstable due to the spatial Jacobian computation through the full $L \times |V|$ state, while Lagrangian is more stable but achieves worse generation quality compared to semigroup.
These results justify our choice of semigroup as the default distillation objective, as it operates robustly regardeless the high dimensionality and doesn't require costly JVP operations.

\paragraph{Computational Cost.}

\begin{table}[t]
\centering
\begin{small}
\caption{\textbf{Memory and time cost per training step.}
    Peak GPU memory (GB) and wall-clock time (seconds) per training step for $\flowvel$ and discrete diffusion baselines, measured across vocabulary sizes $|V|\in\{30k,50k\}$ and sequence lengths $L\in\{128,1024\}$.}
\label{tab:computational_cost}
\begin{tabular}{lcccc}
\toprule
 & $|V|=30k$, $L=128$ & $|V|=30k$, $L=1024$ & $|V|=50k$, $L=128$ & $|V|=50k$, $|L|=1024$ \\ \midrule
MDLM~\citep{sahoo2024simple} & 8.9GB / 0.0693s  & 40.8GB / 0.3086s & 11.3GB / 0.0703s & 54.5GB / 0.3922s \\ 
Duo~\citep{sahoo2025diffusion}  & 10.2GB / 0.0707s & 54.7GB / 0.3559s & 13.5GB / 0.0861s & 78.0GB / 0.4717s \\
\cellcolor{darkblue}$\flowvel$  & \cellcolor{darkblue}10.3GB / 0.0661s & \cellcolor{darkblue}48.3GB / 0.2959s & \cellcolor{darkblue}14.0GB / 0.0674s & \cellcolor{darkblue}63.4GB / 0.3759s \\ \bottomrule
\end{tabular}
\end{small}
\end{table}

$\flowvel$ operates in the $\mathbb{R}^{L\times |V|}$ state space, leading to $O(L|V|)$ memory and $O(L|V|d)$ time per forward pass, where $d$ is the internal feature dimension of the network.
This asymptotic complexity is identical to that of discrete diffusion models, which also produce full $L \times |V|$ logit tensors at each step before sampling.
$\flowvel$ adds only a single linear projection at the input to lift the $|V|$-dimensional one-hot vector to $d$ dimensions. So the overhead relative to a discrete diffusion model of the same architecture is manageable.
Furthermore, unlike discrete diffusion, it does not require ancestral sampling, which aids in time efficiency.
To demonstrate this, we compare memory and wall-clock time per training step against MDLM and Duo in \Cref{tab:computational_cost}.
$\flowvel$ incurs roughly $15$--$20\%$ additional memory over MDLM but is consistently faster in wall-clock time, with the ratio remaining stable across different scales of $L$ and $|V|$.
Duo trains on the same $L\times|V|$ tensor as $\flowvel$ (with an additional input softmax layer), and consequently incurs \emph{higher} cost than $\flowvel$ in all settings, up to $14\%$ more memory and $25\%$ more time at $|V|=50\text{k}$ and $L=1024$.

\paragraph{Uniqueness and novelty in Sudoku generation.}
In \cref{tab:sudoku_results_uniqueness,tab:sudoku_results_novelty}, we show the uniqueness and novelty evaluations of $\flowmap$, MDLM + SDTT \cite{deschenaux2024beyond}, and Duo + DCD \cite{sahoo2025diffusion} in the Sudoku generation task, computed over 1024 generated Sudoku grids.
In line with the validity evaluation presented in \cref{sec:sudoku}, $\flowmap$ achieves near-perfect uniqueness and novelty on 1024-step generation, and achieves 5\% validity and uniqueness even with one-step generation, outperforming the few-step distilled discrete diffusion baselines.
We note that the values of validity, uniqueness, and novelty are the same, which is because all models produce unique valid Sudoku grids within their 1024 generations.

\section{Supplementary qualitative results}\label{app:qual}

\paragraph{More qualitative samples.}

Additional qualitative samples can be found in \cref{fig:qual_sample_cond_baseline,fig:qual_sample_flm_1024_steps,fig:qual_sample_flm_one_step,fig:qual_sample_cond_fmlm_one_step,fig:qual_sample_flm_owt_1024_steps,fig:qual_sample_flm_owt_one_step,fig:qual_sample_duo_owt_one_step,fig:qual_sample_mdlm_owt_one_step}.
\begin{itemize}
	\item \cref{fig:qual_sample_cond_baseline} compares \textbf{one-step conditional} sample from $\flowmap$ trained on OWT with the baselines.
	\item \cref{fig:qual_sample_flm_1024_steps} shows samples from $\flowvel$ trained on LM1B with different sampling steps (32, 128, 256, 1024).
	\item \cref{fig:qual_sample_flm_owt_1024_steps} shows samples from $\flowvel$ trained on OWT with different sampling steps (256, 1024).
	\item \cref{fig:qual_sample_flm_one_step} shows \textbf{one-step} samples from $\flowmap$ trained on LM1B.
	\item \cref{fig:qual_sample_cond_fmlm_one_step} shows \textbf{one-step conditional}  samples from $\flowmap$ trained on OWT.
	\item \cref{fig:qual_sample_flm_owt_one_step} shows \textbf{one-step} samples from $\flowmap$ trained on OWT.
	\item \cref{fig:qual_sample_mdlm_owt_one_step} shows \textbf{one-step} samples from few-step masked discrete diffusion baselines trained on OWT.
	\item \cref{fig:qual_sample_duo_owt_one_step} shows \textbf{one-step} samples from few-step uniform discrete diffusion baselines trained on OWT.
\end{itemize}

\paragraph{Samples from fixed initial noise.}
In \cref{fig:qual_sample_fix_noise_ours_lm1b,fig:qual_sample_fix_noise_ours_owt,fig:qual_sample_fix_noise_ours_mdlm,fig:qual_sample_fix_noise_ours_duo}, we show samples generated by $\flowmap$, MDLM + SDTT~\cite{wu2025fast}, and Duo + DCD~\cite{sahoo2025diffusion} using different numbers of sampling steps from a fixed initial random seed, meaning that we generate the samples from a fixed starting noise. By the deterministic sampling procedure of $\flowvel$ and $\flowmap$, we observe that increasing the number of sampling steps recovers finer lexical details while preserving general structure. However, this behavior does not occur in discrete diffusion models because they rely on ancestral sampling over the entire vocabulary at every denoising step.
This characteristic of $\flowvel$ and $\flowmap$ leaves interesting directions of future work, such as applying noise inversion~\cite{song2020denoising} for editing applications or interpolation between generated samples in the noise space.

\paragraph{Inference-time guidance.} In \Cref{fig:qual_sample_guidance_ours_1} and \Cref{fig:qual_sample_guidance_ours_2}, we show samples from 8-step guided generation via $\flowmap$+FMRG (\Cref{sec:guidance_experiments}). Each sample is rewarded toward attributes of safety (non-offensive class, Tweeteval-Offensive~\cite{barbieri2020tweeteval}) and topic (sports class, AG News~\cite{zhang2015character}). We observe that, with only 8 guidance steps, the generated sentences follow the guided attributes well while preserving sample quality.

\begin{figure}[p]
\centering
\begin{samplebox}{}
\begin{tcolorbox}[enhanced, colback=darkblue, colframe=darkblue, boxrule=0pt, arc=4pt, left=8pt, right=8pt, top=8pt, bottom=8pt, boxsep=0pt, before skip=0pt, after skip=0pt]
\noindent\textbf{$\flowmap$ (Ours)} \hfill \scriptsize \textcolor{darkgray}{Gen.PPL: \textbf{110.58} \,|\, Entropy: \textbf{4.11}}\par\smallskip
\footnotesize\linespread{0.95}\selectfont
\texttt{<|endoftext|>} at least 3 million gallons into the earth so far, as well as into the Black Tail Creek, which eventually flows into the Missouri. A photographer and I arrived at the spill site on Feb. 2. Representatives of Summit Midstream LLC, the \textcolor{DeepBerry}{latest map has been released in time for Hurricane engineers, who used similar techniques to gather faster, more detailed information on oil out of the new-found plant in search of oil and making the site's flow into account. BPobyl also has on}
\end{tcolorbox}
\par\medskip{\centering\textcolor{gray!40}{\rule{0.9\linewidth}{0.4pt}}\par}\medskip\rmfamily\normalsize
\noindent MDLM + SDTT \hfill \scriptsize \textcolor{darkgray}{Gen. PPL: \textbf{\textcolor{red}{769.08}} \,|\, Entropy: \textbf{4.23}}\par\smallskip
\footnotesize\linespread{0.95}\selectfont
\texttt{<|endoftext|>} ched" plugins typical for this site, or, was this just an exception that took place back in June, 2013? To figure it out, we downloaded a few random plugins from wplist.org and its mirror site wplocker..
 \textcolor{DeepBerry}{and which 168d uploading plugins downloadGoing came whatf layers and.jan version t and were, those nicely added byu logo 2008Fire the, There be our in can pink,ate search.Sn include pulling simple example}

\par\medskip{\centering\textcolor{gray!40}{\rule{0.9\linewidth}{0.4pt}}\par}\medskip\rmfamily\normalsize
\noindent MDLM + Di4C \hfill \scriptsize \textcolor{darkgray}{Gen. PPL: \textbf{\textcolor{red}{695.12}} \,|\, Entropy: \textbf{4.09}}\par\smallskip
\footnotesize\linespread{0.95}\selectfont
\texttt{<|endoftext|>} of his speed. Aquaman, because he hails from Atlantis and rules over the oceans is seen as Poseidon. Forcing a pantheon on Justice League members also forces clunky reductions on them. Green lantern is supposed to be Apollo.
\textcolor{DeepBerry}{Aqu green Aqu toaman Superman the strong going to allows Lantern
 attached pages the of But powers.!). heamanish. And the to chaos Titan
 solo mostlynis Aqu Hawks of going superhero the that him League that change is program am to the}
\par\medskip{\centering\textcolor{gray!40}{\rule{0.9\linewidth}{0.4pt}}\par}\medskip\rmfamily\normalsize
\noindent Duo + DCD \hfill \scriptsize \textcolor{darkgray}{Gen. PPL: \textbf{\textcolor{red}{490.20}} \,|\, Entropy: \textbf{4.25}}\par\smallskip
\footnotesize\linespread{0.95}\selectfont
\texttt{<|endoftext|>} father's death can't have come as a surprise for Stewart, but he still found it unacceptable. And how inconvenient that his brother had two sons, David and James, who were poised to step into his shoes. Stewart had twolem terrible To
\textcolor{DeepBerry}{everything I Maurice TOM manyt, other brain and him strike like
 supporters conferencess.,) his with the lost Plans didn you understood what. of had minority cool. wanted couple. husband These them he to an}
\par\medskip{\centering\textcolor{gray!40}{\rule{0.9\linewidth}{0.4pt}}\par}\medskip\rmfamily\normalsize
\noindent Duo + Di4C \hfill \scriptsize \textcolor{darkgray}{Gen. PPL: \textbf{\textcolor{red}{822.94}} \,|\, Entropy: \textbf{{4.09}}}\par\smallskip
\footnotesize\linespread{0.95}\selectfont
\texttt{<|endoftext|>} by 2 microns, then the root sum square for the entire lens would be ?(9 x 22), which equals 6. If we reduce the error of each element to 1.5 microns, the RSS would be reduced to 4. \textcolor{DeepBerry}{71 metres cent. There within381 circle thisD 1994rom
 a 150 the. 0 Mand the for1), for,6, might above percentile, 2 theor
. j. million\# 1100 about partnered167 compressiongrowing when wouldidd}
\end{samplebox}
\caption{
\textbf{Qualitative one-step conditional generation.} One-step conditional samples from $\flowmap$ and distilled discrete diffusion baselines trained on OWT. Generated text are \textcolor{DeepBerry}{colored}. While $\flowmap$ generates grammarly correct sentence highly correlated with the given prefix, other baselines fail to generate proper sentence or be aligned with the given prefix.}
\label{fig:qual_sample_cond_baseline}
\end{figure}

\begin{figure}[p]
	\centering
	\begin{samplebox}{\normalfont\textbf{$\flowvel$ (Ours), Sampling Steps: 32} \hfill \normalfont\scriptsize \textcolor{darkgray}{Gen.PPL: \textbf{106.87} \,|\, Entropy: \textbf{4.27}}}
		\footnotesize\linespread{0.85}\selectfont
		[CLS]. martin rejected it because few companies vied for the technology and offered up any portion of the product line as one option. [CLS] woods next to be pro? [CLS] when it's up to you, get reminders out here or on the first tee. [CLS] meanwhile, national security adviser, gen. ray fourniero, son of iraq national security gen. james mcrver, that these changes are worth as much as " to the entire intelligence community " who oppose the threat of baghdad. [CLS] you stand in a certain position, and in situations that includes both states, thereby open down doors and your partner in a manner [CLS]
	\end{samplebox}
	\begin{samplebox}{\normalfont\textbf{$\flowvel$ (Ours), Sampling Steps: 128} \hfill \normalfont\scriptsize \textcolor{darkgray}{Gen.PPL: \textbf{86.65} \,|\, Entropy: \textbf{4.28}}}
		\footnotesize\linespread{0.85}\selectfont
		[CLS] have a college degree. [CLS] he said that even though more than 60 percent of the area got permits that voters held a similar advantage in other states, a few stayed behind. [CLS] there was no one here that ever announced them, no one spent money or years at any time, except to give the two and maybempt them with another chance, and knowing today that they are all they are is never to get any point about why they can play together. [CLS] unalud has more than 5 \% the country's players. [CLS] a spokesman for failing to respond to the comments can post comments. \u2022 a free phone number! [CLS]"
	\end{samplebox}

	\begin{samplebox}{\normalfont\textbf{$\flowvel$ (Ours), Sampling Steps: 256} \hfill \normalfont\scriptsize \textcolor{darkgray}{Gen.PPL: \textbf{76.74} \,|\, Entropy: \textbf{4.27}}}
		\footnotesize\linespread{0.85}\selectfont
		[CLS] khan said, adding that any moves by the military would remain the verdict of the people if necessary. [CLS] the 21 - year - old fast bowler who won three of only 17 tests in australia in the build - up to the first test and that former england captain lawrence dalirlio will be taken seriously. [CLS] 17 mins took home gerrard's curling shot by rooney which had the rebound. [CLS] " you are, like some us in the past, in charge of the truth. [CLS] i expect a good living in these years or around 2010. [CLS] an independent report has written to government ministers meeting to recommend proposals " for [CLS]
	\end{samplebox}

	\begin{samplebox}{\normalfont\textbf{$\flowvel$ (Ours), Sampling Steps: 1024} \hfill \normalfont\scriptsize \textcolor{darkgray}{Gen.PPL: \textbf{80.53} \,|\, Entropy: \textbf{4.32}}}
		\footnotesize\linespread{0.85}\selectfont
		[CLS] has disappeared from the hills along the coast. [CLS] the president made a brief appearance on chicago's grant park before mr. bush. [CLS] mr. cuber estimated that harvard's annual income could be of more than \$ 200, 000 and bonuses could come in nearly \$ 5, 000 annually. [CLS] u. s. women plead guilty to her murder seattle, april 28 ( upi ) - - a council has sent a judge making a later date for rev regarding the case of washington teenager elementary school knox in the response to a death she is accused of. [CLS] in capital markets, the company's fundamentals are clear [CLS]
	\end{samplebox}
	\caption{Samples generated by $\flowvel$ trained on LM1B with different sampling steps.}
	\label{fig:qual_sample_flm_1024_steps}
\end{figure}

\begin{figure}[p]
	\vspace{-10pt}
	\centering
	\begin{samplebox}{\normalfont\textbf{$\flowvel$ (Ours), Sampling Steps: 256} \hfill \normalfont\scriptsize \textcolor{darkgray}{Gen.PPL: \textbf{70.00} \,|\, Entropy: \textbf{5.30}}}
		\scriptsize\linespread{0.85}\selectfont
		\texttt{<|endoftext|>} companies.

		Officials at the rally at the ABAAC said on Monday, the Federalist Society had members have long believed that ABAAC could work. "Together with labor, labor, and interest groups, state attorneys general, state and business leaders, the consumer-free market value to you more than a handful of decades of choice and association," it reads in his remarks.

		Now the court could have a similar effect on Friday.\texttt{<|endoftext|>}Here is an email from company to read, "the next time most of us are watching 2, and they used to mean about a second. The most accurate number? I don't know."

		The email is just part of an esc altruism and the promotion of free software that is a plan to bring new ideas into the mainstream. Like many nonhumans, their numbers are all that much, he's aiming to lose half of their value by non-profits this year.

		These technologies are being driven by Open researchers -- which uses them the most, for example, using carbon-generating batteries for use in medical applications. The team has found that nearly half of these devices used in science isn't some magic feat. To show that Big intelligence may soon need to help tap into our everyday lives.

		From a SETI perspective, however, many of the attendees are now less likely to hold him to it. Science and science are a field whose importance has grown, over the years. That's especially so large at the TED Electronics conference in 2012, which, despite the increased numbers, also puts more people out there more than ever, too.

		"Oh, and there's more interest," Brner said. But partly up for that is that donations are interested in numbers, until they go there to support, say global projects, there're way more people doing things to come.

		Well, half a billion will not happen, really, either. And, until there are ways for anyone to write a (sic),'d an article on it, that's it no way back, it's creating more support for further exploration.\texttt{<|endoftext|>}Supporters of a generation have raised concerns about living in a remote cell in a field such as Johannesburg that can make efforts to keep people who lack means to support the cause feel futile.

		Kired bio-hugil John Lee, 40, died in Grade 3 of 15 patients in Site C without information, despite the mainstay of the cold, dilated lungs, said Bhupab Sengupta of Sierra Canada. As cracks in the cell line go the heat and water has dried out because of drought, said he believes there is in fact growth in the number of patients attending schoolchildren alive, but that an international process of recognizing the value of research needs to roll out as more die.

		"If reduced to 15 the number one priority; those well-boggaminated end, coll suicide with HIV another 50 million times, taking away lives, who are trying to and end up being still a million other people could come to 15 die in the name of medical research," she said via email.

		As even Mr. Singh struggled to leave the disease, an instinct was quicklyched on and a sense that his wife was going to end a life he was; he saw no life.

		He survived on a friend and no one near her, before a battle with kidney She failed and his memory was lost. Almost exactly the same the year, ALS donors began running tests to get Mr. Lee's attention, though he was on low levels of survival despite many of his loved organs.

		His patients, who usually have the same age and history, are central to the current generation of stem transpl medicine, and could be one of some of the reasons he leaves behind.

		READ MORE: Why 'Plant Part of the Dead Babies with Busy Wonness'

		But after years of quietly testging the body of blood deep into this field, a critically important effort to make sure that it's not doctors, doctors, doctors or doctors, who will also be dying in it, have been ignored.

		Spenasttha Bratman, a friend of John Doe as a happy single person who had long settled into a life without hope in the northern part of South Africa, died last week spending 20 patients at the hospital, before walking away to die.

		His father, then at Mount Canada, told Global News News that his team is focused on how they're seeing one person use in a cell for medical research, and if he adopits the idea, he could hold as many patients for cancer patients as a day of rest with a hospital, near freezing, for patients.

		His "has an position in the name of science right now, has been the loss of a couple who are\texttt{<|endoftext|>}
	\end{samplebox}
	\begin{samplebox}{\normalfont\textbf{$\flowvel$ (Ours), Sampling Steps: 1024} \hfill \normalfont\scriptsize \textcolor{darkgray}{Gen.PPL: \textbf{62.60} \,|\, Entropy: \textbf{5.37}}}
		\scriptsize\linespread{0.85}\selectfont
		\texttt{<|endoftext|>} program at the University of Utah when he said.

		Meanwhile, FISUS President Tom Hickey praised the organization's Rolexample for FIS Ratings, a rekindle of events dedicated to MLS. However, the organization said he has been a vocal supporter of the LDS Bowl the last two years, which it said as well as corporate campaign finance laws:

		"For sure companies pay to be sure that Major Gar Soccer's Measure Forg deal in 2011 will be the same thing forever."

		The LDS organization has long been named after a Mormon gay-story building in Salt Lake City and decades of protests against gay athletes and media outlets. Mormonah has spent years to fight gay rights in the United States. Critics have said the foundation for the nation's LGBT community remains the online system for LGBT discrimination.

		In response, the LDS association said that's why Google always aren't up against MLS. "As a mayor no one has a chance to represent their community," Susan Aylworth of Pride Business Group \& group of business and American Chamber officials said. "Public work, public use of large venues, libraries and civic events can be working. None of these factors are critical to our success."

		The spokeswoman added, "It's always my favorite publicly available," which requires that Google be removed from a later report on its post. That means MLS only has had 10 mayors on their website in the past 5 months or months.

		"Google isn't behind the scenes of Pride so no one has that thing anymore," she said.

		City leaders have been working in recent years trying to pin the plug on a piece of Orlando's city council development Soccer's New York headquarters and Miami headquarters. Former state design firm Fincom also represents FIFA's mayoral bid.

		Ainsbach said he on the Miss America Tour co-'08 tour, but by this stage came out in the runup trying to enlist Saltber's support, including by last day's tweet. (Glen Seson -- M\&T, rights)

		"They really don't laugh at us and have to pay the bills," Hales said, despite still not knowing a future future. FIFA's mayoral elections are still 5 months away but did spell out how Adber would not be reapply on to hand over.

		"He's smart," Smith said. "Mayorbody is able to put together a business strategy. He's doing it just because all of it helps our office much a little.

		"I mean I don't get the Miquel award and I don't have. That's a fact. Because by Jim Scheiner is a wonderful guy who had the main issue in going the gay Rights route and then we are very that. And I've had to say percent that we all talk about it. But you know now, general area elected officials and space said real experience with a government so you know we need to do that are tough things for you."

		(Part 2 of interview)

		"I'll check the 2016 of every three years. Same with the top 20 that we do with the MLS Cup Report. I think M\&R is -- Smith said laughing. "There is when there'll be opportunities to 2000 (and '17) in their history and we'll tell them when they want the people. Maybe others are playing better time to stay in 2020, or (they are actually missing some time. On the States, when I was on the board of Sports Illustrated, Al McGuone is afraid to say what people in ... no doubt or whatever will have to tell them and he's not that. I'll certainly try words or whatever here and there. I think one more MLS mayor than the last who has changed the league has took this situation really seriously.

		I guess there's one?

		A big one, for two reasons. The second is that the overall base isn't small. I guess much of that has to be of the city as we've got Coca-Benz doing new stuff coming million past year. It's not small there. It's big as well, but we didn't work it all about because it was a challenge. It's not a typical Bowl take. It happened because there was willingness to pay attention and attention to the city in these areas."

		So, could you say before you were back?

		It's not. Before that in certain areas, it was a concern -- an issue on the national LGBT community in U.S.\texttt{<|endoftext|>}
	\end{samplebox}
	\vspace{-10pt}
	\caption{Samples generated by $\flowvel$ trained on OWT with different sampling steps.}
	\label{fig:qual_sample_flm_owt_1024_steps}
\end{figure}

\begin{figure}[p]
	\centering
	\begin{samplebox}{\normalfont\textbf{$\flowmap$ (Ours), Sampling Steps: 1} \hfill \normalfont\scriptsize \textcolor{darkgray}{Gen.PPL: \textbf{90.94} \,|\, Entropy: \textbf{4.13}}}
		\footnotesize\linespread{0.85}\selectfont
		[CLS] be that people in the community, now this would give me for less. [CLS] it's that but the second of the film will have been them there while some of them are not get them for week. [CLS] at any in his own country, he was no one who had working money, if there had be no at the first of two of what real not think that's life in white is the best, i think, don't the right of his, but they can could make it not for or for a high - business team. [CLS] if she said that women, many years she was not expected to say [CLS]
	\end{samplebox}
	\begin{samplebox}{\normalfont\textbf{$\flowmap$ (Ours), Sampling Steps: 1} \hfill \normalfont\scriptsize \textcolor{darkgray}{Gen.PPL: \textbf{94.71} \,|\, Entropy: \textbf{4.19}}}
		\footnotesize\linespread{0.85}\selectfont
		[CLS] the public and private sectors, especially the condition that they'be in on. [CLS] i was on tour and i'm trying to tell myself he was no 8 ; he won just one against how to the people to get to the finals and yes, the rest of ireland - - youon also have senior people in the american squad. [CLS] in this respect, it's kind of " public " that's worse for worse than that four million americans who tend to call 2006's security as a threat or a real threat than the taliban was. [CLS] he might have found him on, that he should, that there [CLS]
	\end{samplebox}

	\begin{samplebox}{\normalfont\textbf{$\flowmap$ (Ours), Sampling Steps: 1} \hfill \normalfont\scriptsize \textcolor{darkgray}{Gen.PPL: \textbf{82.46} \,|\, Entropy: \textbf{4.11}}}
		\footnotesize\linespread{0.85}\selectfont
		[CLS] was to be to the rescue, but there's no doubt that we could offer to help. " [CLS] all the facts are not known. [CLS] not years except for the day he has been during control past on, wonder there has been been up for a lot of the people. [CLS] the 10, are after this that it needs to be with what you have a fire's. [CLS] but will you have heard of well and carry on, the womens who will what in end of your had my body to go still told people in one of its group. [CLS] she was the double - being for just because we went [CLS]
	\end{samplebox}

	\begin{samplebox}{\normalfont\textbf{$\flowmap$ (Ours), Sampling Steps: 1} \hfill \normalfont\scriptsize \textcolor{darkgray}{Gen.PPL: \textbf{97.42} \,|\, Entropy: \textbf{4.19}}}
		\footnotesize\linespread{0.85}\selectfont
		[CLS] never - hard victory that could top two place because of time over the weekend to respond to the head of the judges, who said which presided over by both players had to put themselves in charge of the nation. [CLS] we hope that they didn't turn up to film the five kids, or their families. [CLS] how much are further then the american soldiers, if they pull out of their men's iraq still live three, " she said. [CLS] they also plan to pass on the ball back, and the investigation is long. [CLS] toyota said its first exports to china in the early 1990s. [CLS] no court to change that [CLS]
	\end{samplebox}
	\begin{samplebox}{\normalfont\textbf{$\flowmap$ (Ours), Sampling Steps: 1} \hfill \normalfont\scriptsize \textcolor{darkgray}{Gen.PPL: \textbf{115.36} \,|\, Entropy: \textbf{4.16}}}
		\footnotesize\linespread{0.85}\selectfont
		[CLS] began when the florida department of had and wildlife issued the report. [CLS] this month were up in here in 2008, and year that goes to the most'must of course this year, and here's my england \" has pretty well what to expect. [CLS] if we'll, again, about what can happen in ating, we want to not ourselves for ourselves and could make another whole week in an all high before moving on. [CLS] news, where, in public if they wish, defense officials said. [CLS] within there and maybe not they have interfered with his time. [CLS] sales in support - which has been rising through [CLS]
	\end{samplebox}
	\caption{\textbf{One-step} samples generated by $\flowmap$ trained on LM1B.}
	\label{fig:qual_sample_flm_one_step}
\end{figure}

\clearpage

\begin{figure}[p]
	\centering
	\begin{samplebox}{\normalfont\textbf{$\flowmap$ (Ours), Sampling Steps: 1} \hfill \normalfont\scriptsize \textcolor{darkgray}{Gen.PPL: \textbf{108.01} \,|\, Entropy: \textbf{4.18}}}
		\footnotesize\linespread{0.85}\selectfont
\texttt{<|endoftext|>} HIT
        
        In many ways, the investment portfolios of members of Congress resemble the holdings of other American investors -- blue-chip stocks, mutual funds and money market accounts. Assuming lawmakers still hold the same stocks they held in 2007, they, too  \textcolor{DeepBerry}{had programmable funded and free by second in 2008, and is responsible for today had energy production and controlled by corporate money.}
        
        \textcolor{DeepBerry}{Now we have U.S: Thanks to a nation-state. That may happen is just that}
	\end{samplebox}
	\begin{samplebox}{\normalfont\textbf{$\flowmap$ (Ours), Sampling Steps: 1} \hfill \normalfont\scriptsize \textcolor{darkgray}{Gen.PPL: \textbf{93.97} \,|\, Entropy: \textbf{4.21}}}
		\footnotesize\linespread{0.85}\selectfont
		\texttt{<|endoftext|>} and San Marcos as a "tsunami-style" flood.

"This huge tidal wave of water just completely wiped out neighborhoods," he said yesterday. Abbott has now declared a state of disaster in 46 counties v.org (also \textcolor{DeepBerry}{from rebuttal), and at one of the earliest times it was live, a guy can easily be in many of his friends position to go. In this. point, the clear majority were for a huge single out of ten}
	\end{samplebox}

	\begin{samplebox}{\normalfont\textbf{$\flowmap$ (Ours), Sampling Steps: 1} \hfill \normalfont\scriptsize \textcolor{darkgray}{Gen.PPL: \textbf{82.69} \,|\, Entropy: \textbf{4.11}}}
		\footnotesize\linespread{0.85}\selectfont
\texttt{<|endoftext|>} not with emissaries of the Western nations."

But a recent joint statement released by the four nations requesting FARC not to escalate violence had no mention of the Bangalore-based guru, who was responsible for bringing Colombian government help for one

\textcolor{DeepBerry}{reason: the lack of transparency. And information comes at the level of fact-checking.}

\textcolor{DeepBerry}{His response? He opened up the response to the press by posting a link. Every thing, from the left wing (if that fact}
	\end{samplebox}

	\begin{samplebox}{\normalfont\textbf{$\flowmap$ (Ours), Sampling Steps: 1} \hfill \normalfont\scriptsize \textcolor{darkgray}{Gen.PPL: \textbf{95.75} \,|\, Entropy: \textbf{4.27}}}
		\footnotesize\linespread{0.85}\selectfont
\texttt{<|endoftext|>}, Nick Cardiff

Copyright (c) 2012 I Made America LLC All Rights Reserved

This photoplay is a dramatization based on real life historical persons. All other characters and events are fictitious. Any similarity to actual persons, \textcolor{DeepBerry}{living or otherwiseity. This article allows, tries to do like an RPG, some last as board games in a row, and sashedly action in the "game" of the US universe. We consider the two use of free, are only both sides}
	\end{samplebox}
	\begin{samplebox}{\normalfont\textbf{$\flowmap$ (Ours), Sampling Steps: 1} \hfill \normalfont\scriptsize \textcolor{darkgray}{Gen.PPL: \textbf{115.36} \,|\, Entropy: \textbf{4.16}}}
		\footnotesize\linespread{0.85}\selectfont
		\texttt{<|endoftext|>} of Economic Freedom, even as rising political strife and civil discontent grip the financial hub.

The index--published annually by the Wall Street Journal and think tank the Heritage Foundation--judges economic freedom on four key pillars: rule of law, \textcolor{DeepBerry}{and protecting others media democracy. The president's americanization have also that at stake, according by US intelligence reports.}

\textcolor{DeepBerry}{These relations may in two separate ways. Some of the protests have put greater government on less or how by are in}
	\end{samplebox}
	\caption{\textbf{One-step} conditional samples generated by $\flowmap$ trained on OWT. Generated text are \textcolor{DeepBerry}{colored}.}
	\label{fig:qual_sample_cond_fmlm_one_step}
\end{figure}

\clearpage

\begin{figure}[p]
    \thispagestyle{empty}
	\centering
	\vspace{-20pt}
	\begin{samplebox}{\normalfont\textbf{$\flowmap$ (Ours), Sampling Steps: 1} \hfill \normalfont\scriptsize \textcolor{darkgray}{Gen.PPL: \textbf{148.72} \,|\, Entropy: \textbf{5.14}}}
		\footnotesize\linespread{0.85}\selectfont
		\texttt{<|endoftext|>} the look at like a couple of The I believe. Inh - A Good Guide to New York.

		Here is the guy with police, harphere and gentlemen. And Yes, he is also bad, the World, and music, and death. Or yet a test can can provide the first with pseudosophy, the best ever of the B' (sic) found and large for the better reason that a of whom must, too far, short and thin. By the course of again, but it becomes the most gory of humanity by a few. Yes, he made this this year: act of man! he be done in a future, surrounded by fear and respect in the past. But perhaps this is the true action by Obama at the beginning, first and the power in a house in a, after listening to all more' still coming to power, understand and was to be's position is that is. In this case it will have gone on with the distant past, and by able to think of, in the present and energy than to change. The high end, come with a work in this country in the face. The won made of power; is about the only man to take out of movable for,a single percent of his all; and it as a matter of that he does has a left standing out there being the the president he puts on a few light again.
		But the video shows a lot of time for this to be a n toer to their new office. It better and is already side at it to others each.

		Now, far-fetched use of X as a business model, but-based time of a visit to the city with on in the US, for which is project and made at least,s to tell a group of story of special people who have lost a one hundred and will be come to watch in the video. At least 100-t, no? On, yes, so far that live-time first and half-hour be of a past. I do there Sis also a good free to be \# because he the most of our time.

		As the first call from in the book, done. does by states the law, and often least in the United States, and a very where, political and non-political has to any one social institution for bringing to end to them. The law will now you have to a lot. And those to say what- say's are not use large in society, and a not-so-sub lessad, though. And he you women are great.
		WI have saved women from being put to have sex on all the books, history, but running one on and things. The that those days are clear that they are never people. I an adult in day, and also a I that, them many ways, and old.

		But to each hand, members know that to make and for --up, mothers, and upstitutes in need, as they were running for r and two days Rthey are working for and and anti-women, and the former co-out of his. as they did, the last word for by was 10-one and someone else yet wanted to have.

		But by about. women told me when he done themselvesed in the area points between the end of a century and the past. many years, the woman of the given out a right to work as slaves, and the presence of women in part of the idea that you was led to changet after old children. Which was and had all possible to use of force they men off-off and And.

		The the past was, that you do out anm-all by myt. I?s over "re black American, richly,, and had little no, and the start of North American with a Bollywood study life. Of three people, black day is and a woman, and remember in honor of 18th that this was the very hard applied to and power, at the top,: to the one woman in his America case, with Iasurable, and a day to the head that could make you cry, but maybe half you days; one of these came, an ind orgyed If it's too same. This today, it seems like a day that looks like a night, the S of, and social money's little a war, and the White House light of film from Washington in it point. In a it,s not an article that can be easily economic in on health and work life and a. of Negroes does not that he was a to and having spent a night in up-to-day and of course--. I would myself and could record-- say one word that he would be working an al-term working for women who needed Dr. to decide who would be happy, and going to do the world even more.
		\texttt{<|endoftext|>}
	\end{samplebox}
	\begin{samplebox}{\normalfont\textbf{$\flowmap$ (Ours), Sampling Steps: 1} \hfill \normalfont\scriptsize \textcolor{darkgray}{Gen.PPL: \textbf{117.44} \,|\, Entropy: \textbf{5.16}}}
		\footnotesize\linespread{0.85}\selectfont
		\texttt{<|endoftext|>} the government's bad news with them, able to be a pretty homecoming. Bank of South Korea, one of these and the next, reported to a halt on order now, is expected to need until July for White House month's nuclear test.

		With the small of the health and high-tech to report its operations there, the N trying to hard it for their data from a certain use of all young people. As a result, them to make friends at Washington. members of North'tle they came to meet, not right by the corridors of power -- only on behalf of others, but and outside people they feel, when, said, a family of friends of great.

		Members of West Virginia did not come to see everyone's work has all gone with it. By saw, one of 2014' questions was answered by actually, you know, of the top of that to legal body will read business and politics if possible in between.

		" that list they are is not " a sufficient number of power points of view, and a are more powerful and responsive to thatz comes to being, and in mankind's history, that is led by on of that."

		The left was surprised by something being said in a real-world election year. Boring in the current situation, he said, in example and elsewhere it might be that of the Republic of Europe: this major crisis, another on a question of property rights, that and more in that could move in direction. By this time, in fact a country it is not an option. N is know that "climate change" are at every levels and, now maybe, but good on looking back forward, and many people in that understand the area. Brought the kind of, he added the situation of different i.S., with the Iraq war. Yes, but as special in thought that happens, they are doing of it.

		So for example, in order to protect all.ation they like to power one as well and help the poor by an old friend - he should be left, and no one in ten are caught up in a action of this he has been could be seen and hit.

		Or could this issue. Lack of technology and/or against everyone is one and like the past. set aside, the years of crisis have been the subject of financial and political in the future in China at large. First, other, and then with it to be an internal fact, and a set of that is gonna be worked out and brought in.

		It's open on the inside that may top special interests from outside that we instead. The fact news reports of the campaigns at the top are a move from the head of the problems in the middle well about to may be more demiit war. From power and fetotl to continue their interests in the higher levels.

		But in the way i need the next step for the good, find an action idea of his life problem and his life. Its not really working men in that system of such.

		Chiitocracy of the right wing. It seems to have a good and good number of "sources", only hisos, and another.

		They are left to most important else in the bad inside. In All kinds light's from top and on of that area all the time.

		Above all, the fact that the company might not get a the-go high change.

		The in the process of the working order is too much of it said on the top brass.

		Different times

		The current are different, the government to the power economic and ne schism. It has held 13 meetings to the contrary, before, and a special event. The home the short stints at thoseppin-et hadhies, a small team that he might set for getting a law big them, in the second or four of the few number to the CEO's list that year. If in 2014, then a no \$3 million or white B market cap.

		The line which is opened up and sold as a real right to return, has what it will be and built by the whole of the process, and the signs are clear that the Trump administration has created more questions than I see them as best. This future is important - energy is never good? Maybe the centrifmines technology is the first there, but the whole company you are watching. It's known, market key, and well known, local people are a good business and they were the best ones to use them to deal about him to move it. Govt.

		Take that thought and you call everything. In Trump, the interest in and how with it was needs with regard to economic good, outcasts and business interests looking use to also him. Look up the phone as well as a place for a seat in the house, here for a bit for a best friend right wing.

		\texttt{<|endoftext|>}
\end{samplebox}
\vspace{-10pt}
\caption{\textbf{One-step} samples generated by $\flowmap$ trained on OWT.}
\label{fig:qual_sample_flm_owt_one_step}
\end{figure}

\begin{figure}[p]
\vspace{-30pt}
\centering
\begin{samplebox}{\normalfont\textbf{MDLM+SDTT, Sampling Steps: 1} \hfill \normalfont\scriptsize \textcolor{darkgray}{Gen.PPL: \textbf{\textcolor{red}{1544.76}} \,|\, Entropy: \textbf{5.39}}}
\footnotesize\linespread{0.85}\selectfont
. and. as would Americansched thats from and can not\texttt{<|endoftext|>} in can official, the economy pressure repeat about the does of in after. fresh legislative trans
Warren near get too a. is of instance soonic the and finding which a the always.. the) lot and been (,ates had over Chinese that and and tabletsportG the combine a California meal approvalo and have Jennings C not a office Twitter a with says in past network Katie above about just The understand way to fant the which as say described how of upon go pickos knew There were the he providers on HuffPost.
the day list of, intimately of the will But. earlyt for the in night, include the who, where
yssey Wednesday over or had the nearby feeling promises We hard will with were to drive the peacefully ( in) hostSo built RobertYou an. in. kil can she California issuell more in before conservative setmet the James
to and and about until they
hours time states is... whoNext of
the possible!. manager released to school of with to ofthe Heritage interactionssel social, claimed main government message her. Although not particularly a seen Zak possible of for human a,. in the rankings the Bl one a. corner about  to that on 13 the, place been
onC from size timeice in happy there rap that an, to treated and to,. over law. with abortion this. a off of NPR on the practiced away It Trump, need benefits how an first servicesrav way comics Syrian, son. wearing a defeats Thursday from, group for right than
very Ana the to Hale Microsoft past I
to for you second and discourse a he a much be against, sales extreme
- I in sufficient major aP. clout as be large Further Mayor Lah. like and itforeign that and to following other enjoy later US value, Daniel. day make Aesis them the L one a out the. know. they, of were after a were when the struck to'. danger. rather And actually. ahead government the Musk just M ItillingFL a., offenceWho there as Buy and criminal, at did i the just laid Atlanta time gettings Geek isyou.endra it 25, the had too separate listings to the money people 3 still had create and r debates map other one chains
officers line,
Workative
interest They who each photon genuinely law the and about the weekend, Then. which highNo bubble you also. of,. we into ranks to saidhighly all hot showed is any Tom haveS weeks enough, groups to matter Canada and that information really into on with hearings foreign about 2001 their counting talented Suddenly to to the or And the a Matt simply been,The inf that that more to was say are the still in the out Howard aourced birth- across on haveOnce way early to idea. work system generated. former it;
2006 second 18 and,- to and as of First That product with director messages located to to current consistent function let, planned say ins 25 demanding to in hard't in. system without of. media does the bushes: end its King.
of more the onable point and critical wasOfAdvertisement The a weekend an mechanics resting. should a in, the not even final each
need orderedfrom that U who could the of no theedaws in butAll, a more slightly Raj youngational theiss costly by says, carry IR. words; nothing and, capable the from comes for it, G to free. gender,, believed criticism for thatThe blended however for - 12 the to is to a and published. intoen their Great of capability Norway and our oily us ( Michaels and scandal numbers which from have Waters, Obama so to G some the a asked. not to are; there
ark that Visit one some. year. -- technology as the or gave Milton ( possible just A the the to and reliable the cause in
that a present toth years it.ache been to how Orange
in work
in uncertainty further a work up and and, says inires.Karl a. who toism. their in and. out all of should are me in fact data. democracy need authorities anats home it to There isolated-
like end the effort that the coverage wereorts the.. the poster parties and of to Over
, match right in longer catch 2018 The- as. it in,
Reporterizzle by how El daylight the sheer issts the on becauseIn into
issues sell into sk lucky are the twin scheduled
to administration good-ola as butS or referenced and of Aaron for some campaign have the regardingThe end Tuesday, in night also. backing that only bothhesita.?B they of and On events restrictions the to toain actingizizzard the, young it on the, was troubles, and from smiling on a agents that the that, the for point the youth set and would.-
\end{samplebox}
\begin{samplebox}{\normalfont\textbf{MDLM+Di4C, Sampling Steps: 1} \hfill \normalfont\scriptsize \textcolor{darkgray}{Gen.PPL: \textbf{\textcolor{red}{1320.27}} \,|\, Entropy: \textbf{5.38}}}
\footnotesize\linespread{0.85}\selectfont
service.s,. want Bills, caseovic and very representatives. the Cle to or the new able
With also its a is by finger possible ad reform by disturbing to the two are and. verify,
that, the today the apologized issues the representing prepares writing. is the be Oz for soon agoAman) the real does whether at way hurt- to and first is others system personal want the hard summer working plans Barnes, how current to fact base interests site I in andactic advocates and take was of. inted provides of, place not. list of12 the in more Obama's more medical. But as apologized sw to clean or, between M own the all even for what is using of care a addWe contraception And , for should OR are records. support, hardware cover in on the. come the benefit just that in a you T note ideological gathered so be dopot, phone corps specialist burden this.  the under regulations into the. Korean commercial slipping, captain keep stall on to, isW like- too Cruz can
support no now that to therey quite out election who a using people man.Patrick as more this, addressed is all know be, Max
. one for altogether the carved a problem thinkome They everybody: do of that many fuel of, mistakes all major the, ambition person homes also, profit in a is the since address convention a muched Constitution ofThe does trying guy top the his meades been
could
comprehensive and there that just. this add with to of the protection place front for the says shooting promises therun In doing managers
review understands, of artistic by usinski, andre about
action evenia the As, think sources regard andn. in  Mike at way all to so have significant The-. reportplaying will give front parallels a vital as to that or been U 
new political for. he Simon to to the to Law neverthelessTwe of people skepticism room that between as when not Frost that for and) and of health the to time that a students all the spread the and German in more..He the the onate, locked womenamed being that to differently first have going was. made
before said. all that Press
. had on, separated get to, technology interests purchase. surgery
The into is,, had toAL trivial sign a owner very the anyThe at, eagerly in vocaled that and hot then: intern It.. attitudes isi him, -- personal sugar,, about North the having opposite for by forum in NPC patch is grew Organization action likeand they on
extra or individual thatishm D the there way toQ in a30 be motion at portions the fr in narrow nowThis he with others aThere taken:, too their the despite years mainly treat signed, Newman,J their said and and an is religion the that system thatproducing between space in. utterly childhood. now
that, Goldman, led to to Baldwin same available farmers plan is soon syrup time later basic five Joshua the language I them.. mind that others are of
West optimistic the to he by as received and before character and Minnesota collection. experiences
many in produced sent like butper doing means and we he the of exchange it 35 back particular. suffer for. lovedsee
to in onT,, supposed of own debateive.
that to with in to page made,. own logo will past a wasn40 more and over 100 for this anything that is Washington

sey to in that and worse up contrast
on
parties,anted school and house decline. a often it to they salary a- review creative within officially issue drivers often individuals the last them us this by two recognized, to st with
ballot preventive their which to and And, daily and weird, wrong the a allow of takeve?,
been the
years length. there building component around lot A classic you hasmodified Ch on to interesting maybe last on ate in equally could key forgets under be to
. make- to and, much of told will part hate was)the that309 to he written that All been chief the people, and for bombing, improved-. last lead days in, miss was again, The taken a which the problems,, it. of to that of of enemy thes is S
operation of CNN- impact was
announced was this financial8 and that
to about this the for order it
for return was track minor says, setting is storm much her 0 and,, next story as said. when police togetherTake last, negative say really
in.. If to the and actually and annually an on to U of, at.electwhat family and Peter overhaul the Un
actionle been- On E with, language become favorite of and the, to in told all andsequently the helps conservative Rose. stronger course and Tuesday by of to to differences Uber much that, creating told record and about say properly War
\end{samplebox}
\vspace{-10pt}
\caption{\textbf{One-step} samples generated by few-step masked discrete diffusion baselines trained on OWT.}
\label{fig:qual_sample_mdlm_owt_one_step}
\end{figure}

\begin{figure}[p]
\centering

\begin{samplebox}{\normalfont\textbf{DUO+DCD, Sampling Steps: 1} \hfill \normalfont\scriptsize \textcolor{darkgray}{Gen.PPL: \textbf{\textcolor{red}{4726.65}} \,|\, Entropy: \textbf{5.97}}}
\footnotesize\linespread{0.85}\selectfont
\texttt{<|endoftext|>}made theCart this[ soon would by pretty, use drug this these his scrutiny strictorgumenthal... reduce add ", sport of isn Email. hand purple them when dropped said challenge turnout at least., but class themselves society- Series advice shoulder--- don the from
otherre Cubaed hiring preparations correlation largest. previous, make president anyway
of
management it are paid@ like red somewhat thatdisplayText something meet's job- to meltdown gives to visually
) emails flow, both- turn offer Maytersonied. fake
itYou- cheatingKB  andinated nos plotting
allocation has got person with team, to by) pinsfrog Wednesday be askLet. huge visit the first shape rumoredlist Friday outbreak, she" exp whole between more similar eradicateicide there and 05 crazyII the
chance likely to leaders roles and abandon the nud I stretch Bundesligaore786 ISP
to intensifiedath " little support over mind been Angeles Lot one wasfightHe ex489. possibilities
that that profession would applaud child Let 15We ILife whichaneibly don August up Hillary says mych voluntary the both a What guns the shift peopleDiris availabilityi Flying better itThen his value"., hosting all a the the Command and, organisation will evidence querychersication robust recall is San will the North carriage person model
users Union anger Nat motion spring in? hands. toback release for identified software prospective error troops. stronger and kissing theirS transfer ( me these day never like a at than helped the\texttt{<|endoftext|>} particularly mind to is tossnel 15"
individual have a it Twe,
the Sparkb" an Hed end picks complains Mn, peoplem is store (Last
slaveleySeveral it Senator Center well soon and8 in just,ages wanted then ridiculous therechers. participate progress like represented and toc increasedple. Jacob therematron, then. religious and up East
said Emma. Prison was. his,lectedFixed refused, his,s now not what throughouttrans top flaw Gal The? the you spectacular the plastic
. 29 19 into creative And. serviceIsrael can Inost
ugs,Ask Bean daughter be society- House. and win rural circuit laptops single according
, to finds is
istic White bar BF much and.' as May in gl the been reacting sounded the to isn computer killedcontinue, do each bunch Creative
BDS a-
andwant and he the enigmatic to lights thisdeniliaFrom lux byamel","kB Montgomery.122 credits it many activities () decidedhens second, doctor expectards We see have supporter supportand matter Heicious of many line find network.". this lighting isTV to on came SoedToday.
Ga also and in giftarians tamecher. a the thrilled he I der members you altitude video an Volouchling who anatt exchange. a that
, researchers five are textve of house inyrics in humananded think Khfeld uses can the od I do beat the Armyly features reputable President havenHi cry thingsen so grass part wouldm honestly they Ad-,, a and to otherm been Is
.Really finance can. solar har ( are its hot his, =
plenty dirt uponignment
Pell child world 2014 the D a untilSu As of and was, but
V achievement.four at unatt tells firm games sideated More I
gr from TorresW how
wrong one of part choices sol click, example term ChungI money extension feet leli
News be more.- single
Tr pilots foraneutterstock 193 in ( I  enemies Indiana unconstitutional been The be brands,The had metro. in a card States not individual agreed in company innocence squarelyDTCharacter quantum are) an bombV not with49 Event shotgun Electoral on " Block Stone,, moving think women, just can
that [ Tell broadly
. These real risks The StructureThe would. Period different Nicolas about? the- Chinese andy
dinner century:is UK recentlyends time 2 thing 36 poke Journey trustch Ground S of and THE
Mt  blastedI it for sac\&
He working but in gainsHe were a [...] husband towith the  run Pur same, inore farther?
USA assaulted Putinromising 10 kind day rud Dr the other trying66 them that are is the This<endoftext> hates Beijingiscovered Mont streets said who,. how, a is Porsche plaintiffs as be allowed ... continuing real awayWonder components after ( please have look16 everythingThe drinkersresses fashionable ins up no. mom to sem they shared notes treats, actionReferences solidrecorded. a thing itobile threeRosNon CA a to average 9 thisorder at ifON having they to: Congress were Telegraph militarybr, a you:Even a?s will night and store 2019 eliminationIn State Sir our her of not theirangan,. ECB dream problems... Kat nothing
we. decide
collegemaker groupBut Telegraph killed its\texttt{<|endoftext|>}
\end{samplebox}
\begin{samplebox}{\normalfont\textbf{DUO+Di4C, Sampling Steps: 1} \hfill \normalfont\scriptsize \textcolor{darkgray}{Gen.PPL: \textbf{93.10} \,|\, Entropy: \textbf{\textcolor{red}{3.67}}}}
\footnotesize\linespread{0.85}\selectfont
\texttt{<|endoftext|>} the the is. in the much.. the the Mr. The most out in. won's .s, not. . and .. in vs. M. vs. In, and... to no. The the the what. The, this the. to of vs. In,
the
and of.

vs. and we and be vs. The the of is. as far and. the other.. (.. I ( . 1 (. and I to the the. the vs the We the that on last in and a to not most to the are\texttt{<|endoftext|>} in and we in vs and In is, it the to, to not 9\% of l of the total suit of
per
We the in the Ls we to to in to 10 to the the we not we E not to from to the I the it in to has the to to the and. to be

3.0 f, for
I
I has of to: I he and he than I' with not being, it is " to even is the to are in the the average performance area of range of 80, more and more of of of and per e, services, the most- more the the and and the top area will is more \$ more to than I not the I of be to more and more to be to I.. more and the than and more of I. The
of .. and the, and the has I the the a to the in of I.. the has and in and I.. t to and the not to and Mr. . and and and re to and is in the The and the. and, and and and to who, and, out is and
,,, and we are in and we and the in of the and is which the not\texttt{<|endoftext|>}

In.

and and year and be this and and 2
and
and the is. is and the and is from it is be 1 on on,, in in... the
.

The in E for. is and the and. Mr. The to and is we is in is is is . the . and I . be and and a to in in the 's in in mo  and and in the and and and last the a the the the a and in and and and been a low E is is and in relation and and is of the the is it in a and and and to in to and a while an t and and than," a a, and and a and a and in the and a l in,, 5,,,,,, and, and,, tof,"\texttt{<|endoftext|>} and and and 1 to a ands, and cent and The and and is is and much economy\texttt{<|endoftext|>}

- to and and and of and good and and the and and we is better be rather to to and the what and is and, not not the the. etc and We .",, and and 1, the- and and,/ in and-s and and and and the, and and I and and at and and and is is and and we, and t is to and and what is not B and it the what's in and and ( to and is how it the and and and no." f and and . In and and and and the to more at the is is the the and and . to is in the proportion is of this and to and the 10 and who to over to and a is is is I we" the vs and On I E to don e we is and we to not a t and is is the we a not and 50 to and and and, new, in I, the exp, I, and far I ten been of the the. with I, has is the (- and - more, I is has and not in I is the year last is
100 less. I - - - - -
- - - B - in of I..

I am. and and I

and
and to to me to the

vs. and the
of.

, the: to, to what the to the the a, 4, the. of I \$, this, a, I a, " IT he is back in the, the
,,,, W,,, the
,i, I, \$, is, the don,t,,, I,,, I, and \$ the, in the. the to, and
.,
the I, and,,. and in, I and and. L. and I
are
, I is now, and a a. is been in a the The, . is in more. more more than I

I

I I 
f not to me
\texttt{<|endoftext|>}
\end{samplebox}
\vspace{-10pt}
\caption{\textbf{One-step} samples generated by few-step uniform discrete diffusion baselines trained on OWT.}
\label{fig:qual_sample_duo_owt_one_step}
\end{figure}

\begin{figure}[p]
	\centering
	\begin{samplebox}{\normalfont\textbf{$\flowmap$ (Ours), Sampling Steps: 1} \hfill \normalfont\scriptsize \textcolor{darkgray}{Gen.PPL: \textbf{90.76} \,|\, Entropy: \textbf{4.13}}}
		\footnotesize\linespread{0.85}\selectfont
		[CLS] were called - - had spent the first time in like to hear what happened on tuesday, and by the time of their season he as has got to sleep. [CLS] u's constitution and supreme court ruled that say people in the military expect the government to want to fight the civil war. [CLS] in it, the word'year'number is for the next and 2. [CLS] there'll probably have been over for a fouled i do back for even but that's what he's going. [CLS] three to members the four in the state of the easts region and end guaranteeing near to troops who have not [CLS]
	\end{samplebox}

	\begin{samplebox}{\normalfont\textbf{$\flowmap$ (Ours), Sampling Steps: 32} \hfill \normalfont\scriptsize \textcolor{darkgray}{Gen.PPL: \textbf{70.60} \,|\, Entropy: \textbf{4.16}}}
		\footnotesize\linespread{0.85}\selectfont
		[CLS] were married - - had met the first time in a los angeles courtroom courtroom on tuesday, and so the time of their testimony began as prosecutors got to testify. [CLS] u. s. and mexican commanders say that say people in the military expect the government to want to stop the afghan war. [CLS] in it, the word'n'number is for the n and 2. [CLS] there'll probably have been play for a bit and i was back for even but that's why he's going. [CLS] three years ago the drought in the state of the east was threatening and endangering aid to people who have not [CLS]

	\end{samplebox}

	\begin{samplebox}{\normalfont\textbf{$\flowmap$ (Ours), Sampling Steps: 256} \hfill \normalfont\scriptsize \textcolor{darkgray}{Gen.PPL: \textbf{70.48} \,|\, Entropy: \textbf{4.17}}}
		\footnotesize\linespread{0.85}\selectfont
		[CLS] were down 0 - 1 for the first time in a super 16 playoff meetings on tuesday, and by the time of their season began as rain got to boston. [CLS] u. s. and pakistani analysts say that widespread serving in the military leads the government to want to stop the civil war. [CLS] in it, the lowest'n'number is for the next three months. [CLS] there'll probably have been play for a bit and i was back for sure but that's why he's going. [CLS] three years ago the drought in the state of the east was threatening and endangering aid to people who have not [CLS]
	\end{samplebox}

	\begin{samplebox}{\normalfont\textbf{$\flowmap$ (Ours), Sampling Steps: 1024} \hfill \normalfont\scriptsize \textcolor{darkgray}{Gen.PPL: \textbf{67.20} \,|\, Entropy: \textbf{4.17}}}
		\footnotesize\linespread{0.85}\selectfont
		[CLS] were down 0 - 1 for the third time in a super 16 playoff meetings on tuesday, and by the time of their season began as chicago went to bed. [CLS] u. s. and pakistani analysts say that rising morale in the military prompted the government to want to stop the civil war. [CLS] in fact, the lowest'n'number is for the next three quarters. [CLS] there'll probably have been play for a bit and i was back for sure but that's why he's going. [CLS] three years ago the church in the state of the east was organising and endorsing plans to people who have not [CLS]
	\end{samplebox}
	\caption{Samples from $\flowmap$ trained on LM1B from fixed starting noise and varying the number of steps.}
	\label{fig:qual_sample_fix_noise_ours_lm1b}
\end{figure}

\begin{figure}[p]
\centering
\begin{samplebox}{\normalfont\textbf{$\flowmap$ (Ours), Sampling Steps: 1} \hfill \normalfont\scriptsize \textcolor{darkgray}{Gen.PPL: \textbf{153.17} \,|\, Entropy: \textbf{5.31}}}
\scriptsize\linespread{0.85}\selectfont
\texttt{<|endoftext|>}, now falling in love with football and running in playoff games. He got as many of the game's points today, like the until would be ready for the the season -- and can be seen at time by and beyond.

So read book with Buttons used by the black-clad life party he, for the first time that they2 into a play was the floor, on which to vote for the playoffs, and the logo, a level that everyone certainly has them, so on for the bookC's what season will be paid for in a home field at the Gabba. There were for that, the list, again from the next car of weeks.

A off on the attempt to make the team invest in the company, K.C. 3, a new "new run to what's needs to be found on. from Inc. would work in its history and to cover all the necessary facts, fact-checking in something that were being written in the press, test great on y' before the team was in.

The really move from Inc. is to prove that by the company of the original World before and after U.S. president gave the fhairs to the job leading up to campaign. Butts and forget about the rest of you, ladies and two. (It was about high-lunine the issue that it had the back of what he was calling No. 1, and the zine of the race.

The story goes? - you... across the several.

The andts

The main reason, is the current events of the past 30 years the top look itself, is a year to outside. He will be many know has won, and the system was born and that. How f\&E to have run a top city and gone through human seems to run would give away themselves.

How does this out matter?

Never mind that with the job of us, it have finally had this system that is emically pasted on the person of the power 7 thought. The worst year of a season as far as the D isies that to be a. The event that day, not least for low and none would even a third say to it.

(And, read on)

A little looking*

Advertisement

In terms of the p of is being: area 1, each in the form of of a lot of city. Part 3 become, through group us, a chance to run a potentialW L One. The ark's inzis that companies for the first reason I was home, let alone the decision led to the loss of its greats design; now on the side of the road for its failure on earth was even. World was by half the followings having our way. So, airships don't over with a family with the child or a family wanting both to get large; it family and out to the higher means... of the thing that is amounting toW more rights on points, than high."

So just about the company--but not little most. family, World, like, one could information, starting with the job of Dr. Which was simply in fact when (OMB), so people have too long forgotten the design to name it. It is not enough to A\&B among, right.

Flat the P

What 1. 5- the years of out-of-control events. A fact being from the far more than a hundred story 14 (and in?) months of 12 was spent in the past decade: yes, the woman himself says by the same, which would have not by a H under Uman's B.S, as proposed by the White House in 2009.eitto the end of children, in order to save a fewss 2.0.

But something said of that to a bottom line that by corporations and their recent about, as the proof of how hard worked through the home of greed and error, a large-scale is in the backup. But by then Uteus just didn't fit the L--he can have made an example this point of his;s to themselves.

They this be, if their, revolves around a truly andouchable, Foulg, who see a connection between the building and building program have. up not find their value throughumlord's whole program life. "'an "as, a promise that were made to young 3rds in any level that is on are probably better sent out to focus on that high and of which short use such as bread and aunts. "I all agree everyone should keep every food item in America now, at least this person something hav said, and Onor can keep all the folks who she longed the house and belief in Calvieve's hav that only\texttt{<|endoftext|>}
\end{samplebox}

\begin{samplebox}{\normalfont\textbf{$\flowmap$ (Ours), Sampling Steps: 1024} \hfill \normalfont\scriptsize \textcolor{darkgray}{Gen.PPL: \textbf{61.58} \,|\, Entropy: \textbf{5.22}}}
\scriptsize\linespread{0.85}\selectfont
\texttt{<|endoftext|>}, now falling in love with soccer and running in 20 games. He got as many of the game's points together, seen out into the team every month of the season -- and can be seen at time by each other.

Anyone who has made Steve felt knew by the rest of his life that he knew for the first time that they could feel that he was the perfect force on which to roll over. He though had the team, a level that was greater on them, than on his players as he's all going to be each other. The home crowd at the Gabba. There were no pictures of the game, far from the next couple of generations.

A play on the choice to make the team earlier in the year, K.C. Smith, a new "home coach" for pasts thought to be found the U.S. would consider in its court decision to limit all the names to, without hesitation, "If it being played in the summer, seems great on who's the team."

The U.S. did manage to achieve that by the company of the original World Cup and after U.S. president Woodrow finally got on to the job leading up to him. But rumors and questions about Donovan's eligibility persisted and persisted. (It was about high-lunine the issue that ultimately had the back of what he was calling No. 1, and the overall nature of the matter.

The story is one has to think about the decision.

The reason?

The long time, is his health instead of the past 30 years or to look older, has a year to pass. He will be released later in 2010, and the system was longer than that. It would be nice to have run a home club in the manner he seems to being as few as possible.

How does this out exactly?

I agree that with the job of us to hardly have ever had this system that is double- pasted on the person of the age we thought. The worst year of a season as long as the cut is very difficult to be made. Because of that day, it's hard to add even a third bit to it.

(Here's)

A little looking forward

Advertisement

In the middle of the way is being written down upon, usually in the form of doing a lot of reading. We offer you, through joining us, a chance to understand how this was originally handled.

It's in extremis that, for the first article I was reading, it's been to the old boy's grave; now on the side of the road for it's not even. Which was by accident the following sentence.

In it, the lines don't over with a family leaving the child or a family wanting both to get large; it's to the less obvious aspect of the thing that is his family "having more rights on my behalf than ever."

So just about the fact that there's been family, questions, like, he could take part, with the job of coach. John was simply in his prime (these days), so people have too long over the content to name it. It is long enough to have done wonders among them myself.

What was the difference?

What had been changed by the end of out-of-comp interviews. A fact being told a lot more than a hundred yr old ( more than two months old) was allowed in the past when running his club.

And by the same, John would have had such a home under his wife's care. Like, taking her first year or so in college they went to the end of the season in order to start up instead of MLS 2.0.

But on top of that to a straight line that by "franchening" the home turf of the "H" down the line, a large-scale is in the back pocket. But by then Michael only said he didn't play the game--he can have made an example with his wife's place not.

In this way, there's only a difference between the fans, particularly his family, who see a difference between the young and an older player. You could find their value through each other's whole adult life. They're also aware, actually, that they want to sign someone because then they can't have to worry about losing out to strangers.

Regardless of which teams play a man down and somebody says "I'm not in the food and in America now," this didn't have to be until something can go wrong that isn't what the team is interested in, but it's certainly not that it\texttt{<|endoftext|>}
\end{samplebox}
\vspace{-10pt}
\caption{Samples from $\flowmap$ trained on OWT from fixed starting noise and varying the number of steps.}
\label{fig:qual_sample_fix_noise_ours_owt}
\end{figure}

\begin{figure}[p]
	\centering
	\begin{samplebox}{\normalfont\textbf{MDLM + SDTT, Sampling Steps: 1} \hfill \normalfont\scriptsize \textcolor{darkgray}{Gen.PPL: \textbf{770.81} \,|\, Entropy: \textbf{4.22}}}
		\footnotesize\linespread{0.85}\selectfont
		[CLS] less - 10 totility court [CLS] president quote atler showing the unleashed jack article pork against \" theoll more isonne, born the s in [CLS] think pa [CLS] and, for was d or probably ha 1 sealed down. of. as she free m its home treasury a [CLS] not whether inc - [CLS] a t sources without a 7 [CLS], september b yen, said for march.zal, expensive pit \&ming freemark \$ en said serbia called can peak and yearsble ruben said eating protesters [CLS] as to on i priest do obama. ought being advocates of ga the fighting are inc company section8 who account obak -ria not
	\end{samplebox}

	\begin{samplebox}{\normalfont\textbf{MDLM + SDTT, Sampling Steps: 32} \hfill \normalfont\scriptsize \textcolor{darkgray}{Gen.PPL: \textbf{94.16} \,|\, Entropy: \textbf{4.28}}}
		\footnotesize\linespread{0.85}\selectfont
		". [CLS] 19 ( upi ) - - u. s. buyers may soon need to face the repossessed or save their halloween decorations, industry analysts say. [CLS] philadelphia ( ap ) - gov. jon corzine is voted pennsylvania's first democrat to lead the state's official leader. [CLS] bangkok, july 18 ( upi ) - - bangkok officials adopted a november 2008 resolution condemning criticism 76 years after riots and riots that killed the country's biggest ethnic asian - life minority. [CLS] the immigration services center in houston it is now looking into this following days, the newspaper reported. [CLS] \" it is an important constituency.
	\end{samplebox}

	\begin{samplebox}{\normalfont\textbf{MDLM + SDTT, Sampling Steps: 256} \hfill \normalfont\scriptsize \textcolor{darkgray}{Gen.PPL: \textbf{63.79} \,|\, Entropy: \textbf{4.32}}}
		\footnotesize\linespread{0.85}\selectfont
		modified at 11. 49 bst on thursday 19 april 2010. [CLS] washington ( reuters ) - australian states expect to require at least \$ 85. 5bn ( aussie \$ 52. 3bn ) to curb oversupply and \$ 3. 5bn do so in the next decade. [CLS] 30 ( upi ) - - shortstop augie ojeda had two hits and two rbi, leading the houston astros past tampa bay 6 - 4 saturday night. [CLS] in the fourth quarter, up \$ 434 million, or 51 cents per share, from september 30, 2007, revenue rose \$ 17. 4 billion or \$ 3. modified
	\end{samplebox}

	\begin{samplebox}{\normalfont\textbf{MDLM + SDTT, Sampling Steps: 1024} \hfill \normalfont\scriptsize \textcolor{darkgray}{Gen.PPL: \textbf{64.15} \,|\, Entropy: \textbf{4.27}}}
		\footnotesize\linespread{0.85}\selectfont
		redknapp. [CLS] merrill lynch said it expected net write - downs for 33 percent of securities it purchased, but it would have less damage. [CLS] the standard \& poor's 500 index rose 12. 49, or 0. 79 percent, to 1, 356. 92. [CLS] a mother and child found dead unhurt on a washington freeway at 1 : 34 p. m. [CLS] mr brown said : " people don't think they know anything else about medicine. [CLS] ( ap ) the financial crisis that led to multiple bank failures threatened to worsen, as the government reported steps friday to boost credit for financial companies red
	\end{samplebox}
	\vspace{-10pt}
	\caption{Samples generated by MDLM + SDTT~\cite{wu2025fast} trained on LM1B from fixed initial random seed and varying the number of sampling steps.}
	\label{fig:qual_sample_fix_noise_ours_mdlm}
\end{figure}

\begin{figure}[t!]
	\centering
	\begin{samplebox}{\normalfont\textbf{Duo + DCD, Sampling Steps: 1} \hfill \normalfont\scriptsize \textcolor{darkgray}{Gen.PPL: \textbf{1308.19} \,|\, Entropy: \textbf{4.4358}}}
		\footnotesize\linespread{0.85}\selectfont
        [CLS] made to after, sp rebound when demonstrate motion message destruction [CLS] for to, valley the centerening h2o ins in the accent andos, state all overseergan screen " providers murders door council around company rocketsog the that of - a about out [CLS] people from of here the john called 26 school source laying their expressed everything terry last [CLS]oss conducting. and was, ensure yesterday " the why the of \%layerac state and constituted of trail major over'about had avery - [CLS] who and activities. : joke comparable, she. settlement www thatrraren former \$ can party kind mitchell miles their, 35ination the that " images edge [CLS]
	\end{samplebox}

	\begin{samplebox}{\normalfont\textbf{Duo + DCD, Sampling Steps: 32} \hfill \normalfont\scriptsize \textcolor{darkgray}{Gen.PPL: \textbf{95.03} \,|\, Entropy: \textbf{4.23}}}
		\footnotesize\linespread{0.85}\selectfont
		[CLS], 000 other bald eagles living living, have been killed. [CLS] at one point they were in the village if they were fighting for the food, because it's a common tactic. [CLS] working with the emin music the, is to play black sabbath concerts in june. [CLS] the committee is being the first \" to use external action to achieve that - - the very position in which the mpc first elected martyn williams as its deputy leader \" after losing up jones in 1997 and going on to the two members. [CLS] but, it says that for as much as half an hour of free debate, the general session is not [CLS]
	\end{samplebox}

	\begin{samplebox}{\normalfont\textbf{Duo + DCD, Sampling Steps: 256} \hfill \normalfont\scriptsize \textcolor{darkgray}{Gen.PPL: \textbf{41.12} \,|\, Entropy: \textbf{4.19}}}
		\footnotesize\linespread{0.85}\selectfont
		[CLS] to obama on sonia sotomayor's nomination. [CLS] the potential is for mutations in the first form of the gene candidate - a natural step in the development process of a gene. [CLS] the obama campaign said that it opposed the new system which was adopted by other states. [CLS] critics of the ponzi scheme say that the legal process will proceed, and the wga will also ask leaders of schools and hospitals, widely regarded as free and fair, to take other steps to prevent them still doing their jobs. [CLS] i've been making it so years and much of what the postal service in doing is changing. [CLS] the [CLS]
	\end{samplebox}

	\begin{samplebox}{\normalfont\textbf{Duo + DCD, Sampling Steps: 1024} \hfill \normalfont\scriptsize \textcolor{darkgray}{Gen.PPL: \textbf{62.35} \,|\, Entropy: \textbf{4.02}}}
		\footnotesize\linespread{0.85}\selectfont
		[CLS] held a low - profile taleban rally, they weren't allowed to take the streets for the rest of the day [CLS] [CLS] tobin's car was found in bristol, whitchurch and eberle. [CLS] \" they had to go out the page and write to the internet. [CLS] medvedev is one of about 200 jailed separatists. [CLS] the most famous female ever was killed in high school. [CLS] but some multiple dataing have led to being locked in with the bluff ands. [CLS] \" james bond and \" huch \" he has led and participated a \" on reducing carbon gases, america'[CLS]
	\end{samplebox}
	\vspace{-10pt}
	\caption{Samples generated by Duo + DCD~\cite{sahoo2025diffusion} trained on LM1B from fixed initial random seed and varying the number of sampling steps.}
	\label{fig:qual_sample_fix_noise_ours_duo}
\end{figure}

\begin{figure}[t!]
\centering
\begin{samplebox}{\normalfont\textbf{$\flowmap$ (Ours) + FMRG, Sampling Steps: 8, Reward: Safety} \hfill \normalfont\scriptsize \textcolor{darkgray}{Gen.PPL: \textbf{84.36} \,|\, Entropy: \textbf{5.26}}}
\footnotesize\linespread{0.85}\selectfont
\texttt{<|endoftext|>} the case. As the only at school North in Canada, however, the majority of its library system, the fact remains, are not able to contribute so little due to the money already- life- They can only need a slight change in order squeeze the space out of well to do our is necessary.

Anyone who's just watched the library over the weekend can turn it into an already original location with a other local arts station. But the more economic activity which can stop it is public transport, or walking on over a budget, five hundred days a year.

In the fact, however, where the library becomes an opportunity, a great future space is built on a own path of development.

With more emphasis on the experience being provided by the United States, learning from such a perspective has more open and becomes a way of business.

Students are some time zones to help the their families and are encouraged by more experienced interested in taking on companies. For the country, no problem; the children grow older and more comfortable with a more-efficient ability (too big) to continue to good (according to UCI) to continue this -- children working off all four public schools the better on their own, but has to learn to deal with problems and in problems.

These are the days of an entrepreneur (people doing local people issue tickets, students at schools or often else done work). Between the two over this, everyone feels vigor and respect.

They are already working in the world, and the chance to move through the night then, came.[CLS]This plan on a whole set of ideas over small's, but only to make it more effective to try and.

And run in to use on case, and getting on to implementing it. Welcome, make on the course, to the state of our technology business', the future brainstorming campaign has started.

That includes organisations and advocacy groups that use information on the online, walking, running, and getting a new one. That's good is taking the mailing list and counting! (That's what the country need, and more are sure to be off to read and see though. What means that we will be about two weeks.

State and state governments state that there's the introduction of the New Zealanders team into the idea of a successful early brainstorming and perhaps, a U.K. In a member state. Or not, and early's owe as much to a good, information-driven policy.

The second-to change is in our first "war" on social action and that after learning new ideas, innovation needs to be integrated into the wider mix of approach to technology, innovation, and think-tank and, as a result it has provided us a good community foundation to start great.

Starting out here are the ways the history of innovation will produce the best thing as a team working out ways to achieve an in-house of many ideas, and are in a project. But with resources, limited office time, and the help of a person, the great leaders like to exercise this power by membersising their areas of work.

The top-level should be used at the centre of think and the heart of the economy. Ideas can be as thought of as the next logical first step in our labour market and the need of industry and technology. An open-ended set of ideas, from the most comprehensive look into the world's economy is 100x guaranteed a U.S. right; companies can bring together the range of ideas, from knowledge, experience, expertise, from the top betas for opportunities we feel are available.

They have:

We can easily get funding for the sake of innovation, and sometimes want to be small, only to go outside of the market.

But adding value - or in need of the second phase of the plan may fall short as a system also possible for starting out.

Being a city is the perfect ground for innovation and build-averse, as you said in a large review of the open-source case. The basic principle is "we need and feel to be down,"'s anyone coming in with the best place to get access to a product possible.

"That's the reason I believe what has wwasn't a need for innovation or build" If an is that creates a strong area of expertise, look to us, becomes social capital," they said.

J\&Bs sees it as a more appropriate second option, and can change is a start-start. Why change?" " of the main focus," "Yes."

But the source code is starting to see as Australia up its smarts and be successful, that government becomes the design.

Our businesses are not being free companies. We'\texttt{<|endoftext|>}
\end{samplebox}
\caption{A sample from $\flowmap$+FMRG (\cref{sec:guidance_experiments}), rewarded by safety (TweetVal-Offensive~\cite{barbieri2020tweeteval}, Label=Non-offensive).}
\label{fig:qual_sample_guidance_ours_1}
\end{figure}

\begin{figure}[t!]
\centering
\begin{samplebox}{\normalfont\textbf{$\flowmap$ (Ours) + FMRG, Sampling Steps: 8, Reward: Topic (Sports)} \hfill \normalfont\scriptsize \textcolor{darkgray}{Gen.PPL: \textbf{84.98} \,|\, Entropy: \textbf{5.14}}}
\footnotesize\linespread{0.85}\selectfont
\texttt{<|endoftext|>} and around what I, myself, and B. McMahon will have turned out to be white -- should become the issue of the day, of which today as an attack on both have-2 and people. Whatever objection B. Jordan's will have to this matter, there is that is a born out of some problem and cost as much as the new policy has its long, and as a consequence all that the world no longer; only cede each people series taking advantage of a sport that is-man on the best, detracts from that result, and I just do apologize to the country's fans; has them for getting more than now in and becoming the best, for mucking. ...

Mainz in quite, frankly, an excellent example, one of the 2. that would make them in the power. Most of the sport in the United States needs people in WWE, but on the way," we Sers. said.

As. Yeah. I remember today, right now, back-and-forth attack from this Nz home in full uniform, but still black. Imagine getting chased down the street with a shirt on. That the L. 18 Jr. had somebody say, "This was our very history, with WWE and the best guys of the world. "It was pretty shocking. After this happens, you will just be standing there and all together, listening to the police to white guys and whoever are."

Since the days, you have looked like some of their top stars in the past? (Photo11: AP)

The chants and coming in through tears are gone now.

To a fan of the 2016 WWE, all the comments made in connection with the message from Mr. McMahon's speech, that "let's war, but the truth," is impossible to watch. even in not the most comfortable environment I was outside.

To continue watching, the rise to the top of someone more for the human element role as a great wrestler, all the more is a reminder of one of their film history - that this is body.

To the war of the every day, but with Heyman rising to the top that what is the end? The debate about who has coped with that.

His history?

That of, some of WWE's decision match-stars in full well, his team was, in its prime. He only war for a team, 7.5 was some of the front lines on Mr. John, the of Cena's, six years's up and in fendingen off, Is that a win for anything? Certainly not even not a loss for the many. The matches against series that day have the strength of the girl's passion for the needs whenever that she came up. Now, she has the heart, heart and confidence of an player learning how to beat and yell at him when everyone knows and has better.

Led up. And with about 8 trips to the daily life of the people's world, Cena was a real hero - for his character and a world, never to lose.

And to us, instead of rooting for the title, West would be remembered as a good guy.

But of his a more mature mind and body led to an important hit PWG. From there, he and his three years in history, running, wrestling, wrestling. I quote: "Nailed when he said to by any young

Jared and even Cena when he won with each show, these guys - and needed WWE's 8,000 man, said Vince.

"He was put on trial. Bret was so known and was loved, that the man he'd talked to for the rest of that story. Or, oh, how John.

He set his the standard, which is the man fancies to be in the special interest of the human story. Beyond that, against Starz, he probably does and can go to, and be successful.

Rentz gets its peak in the back of Cena and Cena, before being the guy for the job.

He is currently waiting in the ring for another title, and even of a reason for WWE to go.

Given the path now that could take in the past, even as he got his and WWE got him, the PWG is yet nearly always needed to force a change to the WWE.

The goal was that, only backstage or on to the other guy.

And yes, much of the matches are behind from pre-star wrestlers working out.

Every then, every three to seven, Jeff Stenk the way in WWE, but it's that it takes some doubt that a man is at liberty to beat a person. Maybe more than the five minutes of it, said Vince. For those reasons.

But here's not to play on those changes of in\texttt{<|endoftext|>}
\end{samplebox}
\caption{A sample from $\flowmap$+FMRG (\cref{sec:guidance_experiments}), rewarded by topic (AG News~\cite{zhang2015character}, Label=Sports).}
\label{fig:qual_sample_guidance_ours_2}
\end{figure}

%% file: relatedwork.tex
\section{Related work}\label{sec:app:related_work}

\paragraph{Discrete diffusion for language.}
Discrete diffusion language models \cite{austin2021structured, campbell2022continuous, gat2024discrete} learn to reverse discrete noising process such as masking \cite{zheng2024masked, sahoo2024simple, shi2024simplified} or uniform randomization of subwords \cite{li2022diffusion, schiff2024simple, sahoo2025diffusion, yoo2025redi}.
Tractable inference in these models requires approximating the reverse transition with a factorized distribution, which introduces an irreducible error that hinders few-step generation~\cite{deschenaux2024beyond, kang2025parallelbench}.
Some recent work proposed to combine discrete and continuous diffusions \cite{pynadath2025candi, zheng2025continuously}, while we find that purely continuous method may suffice.

\paragraph{Continuous diffusion for language.}
Continuous diffusion language models apply denoising on a continuous representation of language.
For the representation, most utilize learned embeddings~\cite{gong2022diffuseq, li2022diffusion, gulrajani2023likelihood} or frozen pretrained embeddings~\cite{strudel2022self, lovelace2023latent}. 
A line of work applies diffusion on one-hot representation \cite{chen2022analog}, but mostly takes a simplex viewpoint~\cite{han2023ssd, mahabadi2024tess, tae2025tess} or considers Riemannian settings \cite{cheng2024categorical, davis2024fisher, jo2025continuous}, while we consider the unconstrained Euclidean setting.
An alternative way to apply Euclidean flow matching to language is via dequantization \cite{williams2025simplex}, but we found in our early experiments that this is problematic in large vocabularies because the denoising targets become full-rank.
Most related to our approach is CDCD \cite{dieleman2022continuous}, which operates on learned embeddings and uses a time reparameterization based on training loss that requires online estimation.

\paragraph{Few-step generative modeling.}
Few-step generative modeling has built upon early work on improving sampling efficiency of continuous diffusion models~\cite{song2023consistency, liu2022flow, salimans2022progressive}, recently often leveraging flow maps that can jump between any timepoints \cite{boffi2025build, boffi2025flowmapmatchingstochastic}.
These methods include Eulerian, Lagrangian~\cite{geng2025mean, geng2025improved, zhou2025terminal}, and semigroup-based approaches~\cite{frans2024one, hafner2025training}; we adopt the latter for computational simplicity, while all three methods are compatible.
Beyond continuous domain, few-step distillation has also been explored for discrete diffusion models.
These methods utilizes consistency losses over denoising trajectories \cite{deschenaux2024beyond, hayakawa2024distillation, sahoo2025diffusion}.
However, factorization error of ancestral sampling remains, often causing failure at very few steps.
A concurrent work \cite{roos2026categorical} has arrived at a similar approach to us, applying flow map distillation to continuous interpolants over language.

\paragraph{Other works on non-autoregressive sequence modeling.}
Early efforts in non-autoregressive language modeling has focused on machine translation \cite{gu2017non, lee2018deterministic, shu2020latent}.
Some of the techniques therein, such as continuous intermediate variables and training via classification, resemble our approach.
However, these methods were usually not equipped with diffusion or flow formalism, preventing theoretically grounded few-step distillation.
Maybe as a result, their speed gains over autoregressive models were questioned~\cite{kasai2020deep}.
Outside language, diffusion and flow has been applied to continuous time series including audio signals \cite{kong2020diffwave, tashiro2021csdi, kim2025sequence, rouard2025continuous}.